%% file: main.tex
\documentclass{article}
\usepackage{arxiv}

\usepackage[utf8]{inputenc} 
\usepackage[T1]{fontenc}    
\usepackage{hyperref}       
\usepackage{url}            
\usepackage{booktabs}       
\usepackage{amsfonts}       
\usepackage{nicefrac}       
\usepackage{microtype}
\usepackage{subfigure,wrapfig}
\usepackage{booktabs} 
\usepackage{graphicx}				
\usepackage{amsthm,epsfig}								
\usepackage{amssymb,amsmath}
\usepackage[shortlabels]{enumitem}
\usepackage{float}
\usepackage{bbm}
\usepackage{stfloats}
\usepackage{color, array, stfloats}

\newtheorem{assumption}{Assumption}
\newtheorem{corollary}{Corollary}[section]

\newtheorem{mylemma}{Lemma}[section]
\newtheorem{mypro}{Proposition}[section]

\newtheorem{definition}{Definition}
\usepackage{thmtools,thm-restate}
\usepackage{indentfirst}
\usepackage[short]{optidef}
 \input{math_def.tex}

\title{System Identification via Meta-Learning in Linear Time-Varying Environments}

 \author{%
   Sen Lin\\
   \And
   Hang Wang\\
   \And
   Junshan Zhang\thanks{The authors are affiliated with the School of EECE, Arizona State University, Tempe, AZ 85287, USA; e-mail:\{slin70, hwang442, junshan.zhang\}@asu.edu} \\
 }
\begin{document}

\maketitle

\begin{abstract}
System identification is a fundamental problem in reinforcement learning, control theory and signal processing,  and  the non-asymptotic analysis of the corresponding sample complexity  is challenging and elusive, even for linear time-varying (LTV) systems. To tackle this challenge, we  develop an episodic block model for the LTV system where  the model parameters remain constant within each block but change from block to block. Based on the observation that the model parameters across different blocks are related, we treat each episodic block as a learning task and then run meta-learning over many blocks for  system identification, using two steps, namely offline meta-learning and online adaptation.
We carry out a comprehensive non-asymptotic analysis of the performance of meta-learning based system identification. To deal with the technical challenges rooted in the sample correlation and small sample sizes in each block, we devise a new two-scale martingale small-ball approach for offline meta-learning, for arbitrary model correlation structure across blocks. We then quantify the finite time error of online adaptation by leveraging recent advances in linear stochastic approximation with correlated samples.
\end{abstract}

\section{Introduction}
With the recent success stories in video games and Go, there is a general consensus that reinforcement learning (RL) techniques have great potential for intelligent decision making in dynamical systems, thanks to its ability to learn  from the  environment on the fly and carry out adaptive control.  
It is therefore of great interest to understand the system identification in dynamic systems, especially  for model-based   RL. Thus motivated, this study focuses on characterizing the sample complexity of system identification, i.e., how many data samples are required  to estimate the unknown parameters of a time-varying dynamic system. Notably,  sharp non-asymptotic analysis, even for the system identification of   linear time-invariant (LTI) systems, is rare. Recent work \cite{dean2017sample} has built a finite sample theoretical guarantee of the least squares estimator for LTI systems in the context of Linear Quadratic Regulator (LQR) using multiple independent trajectories, with all but the last state-transition discarded for each trajectory. A sharp non-asymptotic analysis of the least squares estimator for the identification of LTI systems with a single trajectory  is provided in \cite{simchowitz2018learning}.
Nevertheless, the LTI model would not be applicable to the time-varying dynamic systems in many applications.


Considering that the environment is often time-varying and  evolves over time, we move one step forward and study the performance of system identification in unknown linear time-varying (LTV) systems. We aim to obtain a clear understanding of the impact of the sample size and model dynamics on the parameter estimation in LTV systems.
Inspired by the block fading channel model in wireless communications systems where the random channel gains are assumed to be constant within a block, we treat the LTV system as episodic blocks in which the model parameters remain constant within each block but change from block to block in a stochastic manner.

Clearly, system identification is a challenging task for LTV systems, for a number of reasons, including (1) \emph{continuous learning}: one-time learning for a global model would not suffice because different episodes have distinct model parameters; (2) \emph{fast learning}: since the environment may change quickly, straggled learning  could result in outdated estimators. A key observation here is that model parameters across adjacent episodes  are often `related' (in some sense) and in many applications they may follow some common distribution. Based on this observation, we propose meta-learning (Meta-L)   \cite{finn2017model} for system identification. The underlying rationale behind Meta-L is to learn a good model initialization by training over many similar tasks \cite{finn2017model}, and use it for fast adaptation to learn the new model using only a small amount of data from the new learning task. Thus inspired, we advocate meta-learning to continuously and quickly learn the  model parameters in LTV systems.  

The main contributions of this paper can be summarized as follows.

    (1) We propose an episodic block model for the LTV system, where the model parameters are assumed to be constant  within each episodic block of length $L$ but change from block to block.  The block length $L$ hinges upon the system dynamics; and the faster the variation is, the smaller $L$ is. Building on this proposed episodic block model, we leverage meta-learning to learn a  model initialization by making use of the model similarity across episodic blocks, thereby addressing the challenges in system identification of LTV dynamics.  The proposed Meta-L based system identification consists of 1) offline Meta-L and 2) online adaptation, and  it is akin to a recursive least square (RLS) estimator, with interleaved usage of training data and testing data within each block and iterations across blocks.

    
    (2) To the best of our knowledge, this work provides the first  non-asymptotic analysis for the system identification performance of Meta-L with general correlation structure in LTV dynamic systems. In particular, based on \cite{simchowitz2018learning}, we devise a new two-scale martingale small-ball method to address the difficulties rooted in  sample correlation and small block sizes. The derived upper bound on the distance between Meta-L based model initialization and the underlying  parameters is sharp and  encapsulates the impact of the model similarity and the sample size on system identification.
    
    (3) Further, we characterize the model estimation error corresponding to the online adaptation using the  model initialization learnt from offline meta-learning.
     We devise a multi-step gradient descent algorithm and recast it as a  linear stochastic approximation algorithm with correlated samples. The upper bounds on the finite time error  reveal that  the error between the model estimator and the underlying model decays  exponentially.

It is worth noting that the selection of block length $L$ in  the episodic block model used to approximate the LTV system can be nontrivial. As will be shown in Theorem \ref{thm:thm.1}, one can choose a smaller block length $L$ to improve the approximation accuracy of the episodic block model for the LTV system, and the   meta-learning algorithm for offline learning can  yield a good model initialization as long as there is a large number of episodic blocks available.  However, the selection of $L$ is more challenging for real-time learning which may necessitate  
an adaptive episodic block model to approximate the LTV system; and this deserves further investigation.
We note that  related work \cite{ouyang2017learning} proposed a   time-varying model where the model parameters are assumed to follow a jump process and change independently subject to the constraints on the average number of jumps in a given time window; in contrast,  the proposed episodic block model encompasses general correlation structure and  makes it possible to leverage meta-learning as a promising approach for system identification of LTV systems.

\subsection{Related Work}
\emph{Meta-Learning:} Meta-learning has recently emerged as a promising solution for learning to learn. Both meta-learning and multi-task learning aim to improve the performance by leveraging other related tasks. However, meta-learning focuses on learning a good model initializer first and uses it for fast learning in a new task \cite{santoro2016meta,munkhdalai2017meta,snell2017prototypical}, whereas conventional multi-task learning aims to learn all tasks simultaneously.  
One gradient-based meta-learning algorithm, called MAML \cite{finn2017model}, directly optimizes the learning performance with respect to an initialization of the model such that fast adaptation from the initialization can produce good performance on a new task. A first-order method  named Reptile is proposed in \cite{nichol2018first} to circumvent the need of  second derivatives in MAML.
These approaches have been extended to devise new reinforcement learning algorithms, which can perform significantly better than standard reinforcement learning algorithms that learn from scratch \cite{gupta2018unsupervised,nagabandi2018learning,rakelly2019efficient}.
However, there is a lack of fundamental understanding about the performance of meta-learning with correlated samples in terms of the sample complexity required to achieve certain performance. This work  makes a first attempt to  characterize the non-asymptotic estimation error  for  meta-learning based system identification. And  our study on Meta-L based adaptive control is underway.

\emph{System Identification:} System identification is a fundamental problem in control theory, reinforcement learning and signal processing. Most existing studies in this area have used
mixing-time arguments, which rely on fast convergence to a stationary distribution
so that correlated samples can be treated roughly as if they were independent (see, e.g., \cite{yu1994rates,mohri2008stability,kuznetsov2017generalization,mcdonald2017nonparametric}).
Recently, there has been increasing interest in non-asymptotic analysis of system identification. Polynomial time guarantee in terms of predication errors for identifying stable linear systems is 
provided in \cite{shah2012linear,hardt2018gradient,hazan2018spectral}. The series of recent work \cite{faradonbeh2017finite,faradonbeh2018finite} characterize a non-asymptotic convergence rate of the least-square estimator. Simchowitz et al.\cite{simchowitz2018learning} address the coupling between the covariate process and the noise process, and devise the innovative block Martingale small ball (BMSB)  method, based on which sharp non-asymptotic analyses is carried out for the sample complexity of system identification in LTI systems.
Taking one step further, \cite{sarkar2019near} derives the finite time error bounds for general LTI systems where eigenvalues of the model parameters are arbitrarily distributed in three different regimes, i.e., stable, marginally stable and explosive.

Nevertheless, system identification for LTV systems from a learning perspective remains not well understood. Existing work on LTV system identification relies on time-domain recursion algorithms \cite{bosse1998real,goethals2004recursive} and the frequency domain analysis \cite{xu2012identification}, subject to restrictive assumptions. 
There are also some prior work on transforming LTV systems into LTI systems to simplify the analysis. Tsatsanis et al. \cite{tsatsanis1993time} transform the LTV identification into a LTI identification problem  via expanding coefficients onto a finite set of wavelet basis sequences. By assuming that the model parameters follow a jump process model, \cite{ouyang2017learning} proposes a Thompson sampling-based learning algorithm for the Linear Quadratic control but with no performance guarantee for system identification.

\section{Problem Formulation}
In this section, we introduce an episodic block model for  LTV systems, building on which we explore  meta-learning for system identification. For ease of exposition, we study the system identification problem in the context of LQR with LTV dynamics approximated by the episodic block model. It is clear that our analysis techniques for meta-learning based system identification can be carried over to  general LTV systems.

\subsection{Episodic Block Model for LTV Systems}
Consider  a LTV system with the system dynamics satisfying $x_{t+1}=A_tx_t+B_tu_t+w_t$,
where $x_t\in\mathbb{R}^n$ is the system state with $x_0=0$, $u_t\in\mathbb{R}^m$ is the  action based on the state history $\{x_t, ..., x_0 \}$, and $w_t\sim \mathcal{N}(0, \sigma_w^2I_n)$ is the stochastic disturbance. $A_t$ and $B_t$ are the unknown model parameters at time $t$ with proper dimensions. 

 
\begin{figure}
\centering
\includegraphics[scale=0.43]{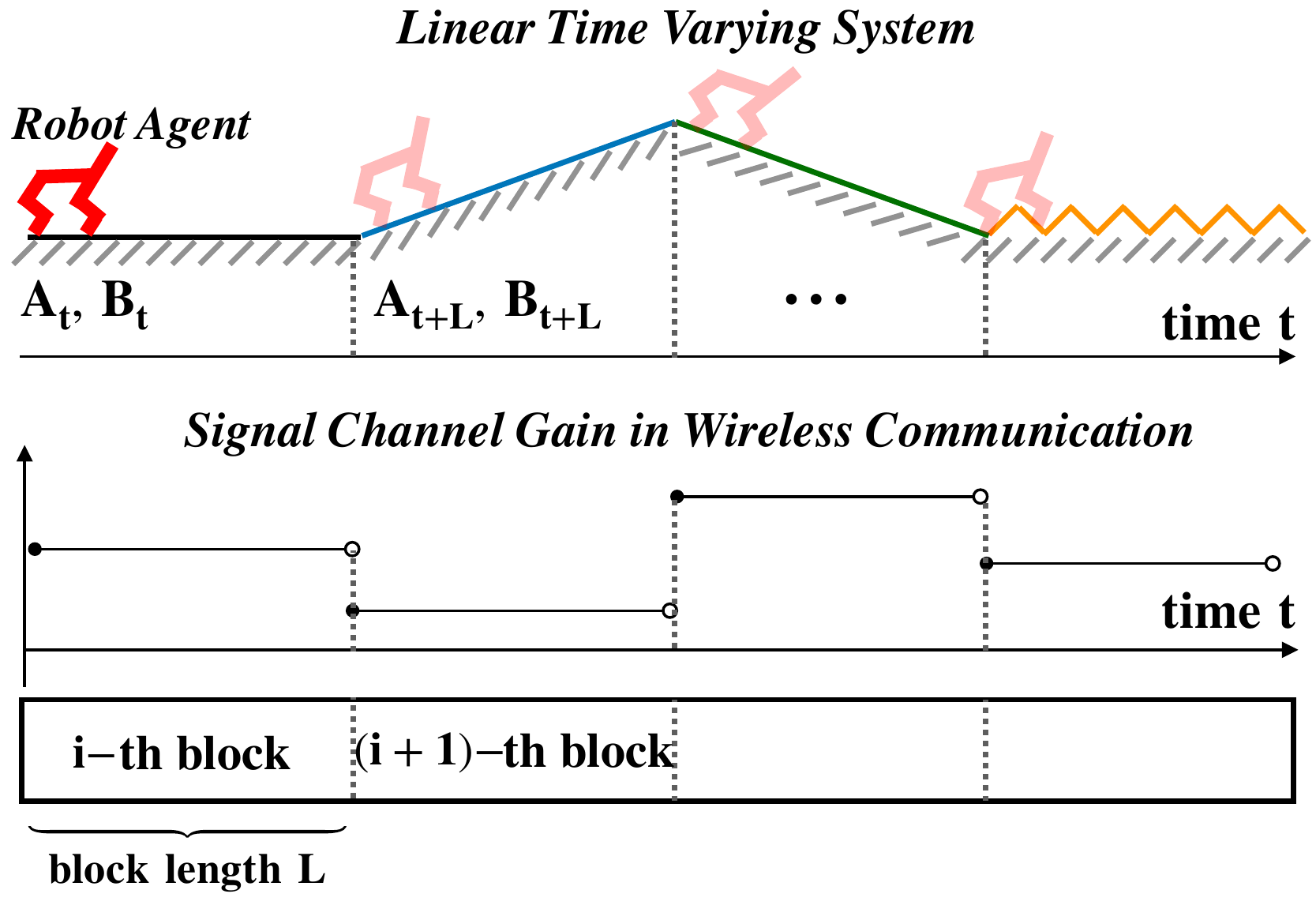}
\caption{An episodic block model for the LTV system.}
\label{Fig:block}
\end{figure}
Motivated by the widely used block fading channel model in time-varying wireless communication systems, we advocate an episodic block model for the LTV system, where the model parameters are assumed to be constant  within each episodic block of length $L$ but change from block to block. Note that the episodic block model with block length $L$ is used to approximate the underlying general LTV systems, in the same spirit as using piecewise linear functions to approximate \emph{any} nonlinear functions.
$L$ is a design parameter of choice, and the smaller $L$ is, the more accurate the approximation is.
For a fast changing LTV system, $L$ should be set small to guarantee  approximation accuracy.
Suppose there are $N$ episodic blocks.
As shown in Figure \ref{Fig:block}, $(A_t, B_t)= (A_i, B_i)$ for $t\in [(i-1)L+1, iL]$ and $i\in [1, N]$, and $(A_t, B_t)$ varies from one  block to another. In this work, we consider the general case where the model parameters can be correlated across blocks.
Further, departing from the standard assumption in meta-learning that the model parameters follow a known distribution \cite{finn2017model}, in this study we only assume that the parameter $(A_i, B_i)$ lies in a compact set $\Theta \subset \mathbb{R}^{(m+n)\times n}$ where $\{A_i\}$ are within unit disk, as in very recent work \cite{khodak2019provable}. And this assumption is made only to facilitate the analysis of meta-L based system identification, given that meta-learning does not need the knowledge of the compact set in implementation.

For convenience, define for each block $d$,
    $G_{t,d}=\sum_{i=0}^{t-1} A_d^iB_dB_d^T(A_d^i)^T$ and $F_{t,d}=\sum_{i=0}^{t-1} A_d^i(A_d^i)^T$,
as the finite time controllability Gramians to capture the magnitudes of the system excitations induced by the control inputs and the noise process. 
Let $\|\Bar{A}\|\triangleq\max{\|A_d\|}$, $\|\Bar{B}\|\triangleq\max{\|B_d\|}$, $\|\underline{A}\|\triangleq\min{\|A_d\|}$ and $\|\underline{B}\|\triangleq\min{\|B_d\|}$. 
 We further assume that there exist positive semi-definite matrices $G_t$, $F_t$, $\underline{G}_{t}$ and $\underline{F}_{t}$ such that
\begin{align*}
    \underline{G}_{t}\preceq G_{t,d} \preceq G_t, \; \underline{F}_{t}\preceq F_{t,d} \preceq F_t.
\end{align*}
We note that the above assumption would hold when $\{A_i\}$ are within unit disk (recall that the parameters $(A_i, B_i)$ lie in a compact set $\Theta$).  
As expected, the performance of meta-learning based system identification hinges  on the finite-time controllability Gramians \cite{simchowitz2018learning}.

\subsection{Meta-Learning based System Identification}

 Building on the episodic block model above, we next apply Meta-L to train a global model initialization $(A_{\theta}, B_{\theta})$ in an offline manner, with data from  historical episodes, such that $(A_{\theta}, B_{\theta})$ could be
   quickly adapted to learn the model parameters  $(\hat{A}_i, \hat{B}_i)$ for a new episodic block. As illustrated in Figure \ref{Fig:meta}, there are two main steps for Meta-L based system identification, namely 1) offline Meta-L and 2) online adaptation.
   
\begin{figure}
\centering
\includegraphics[scale=0.38]{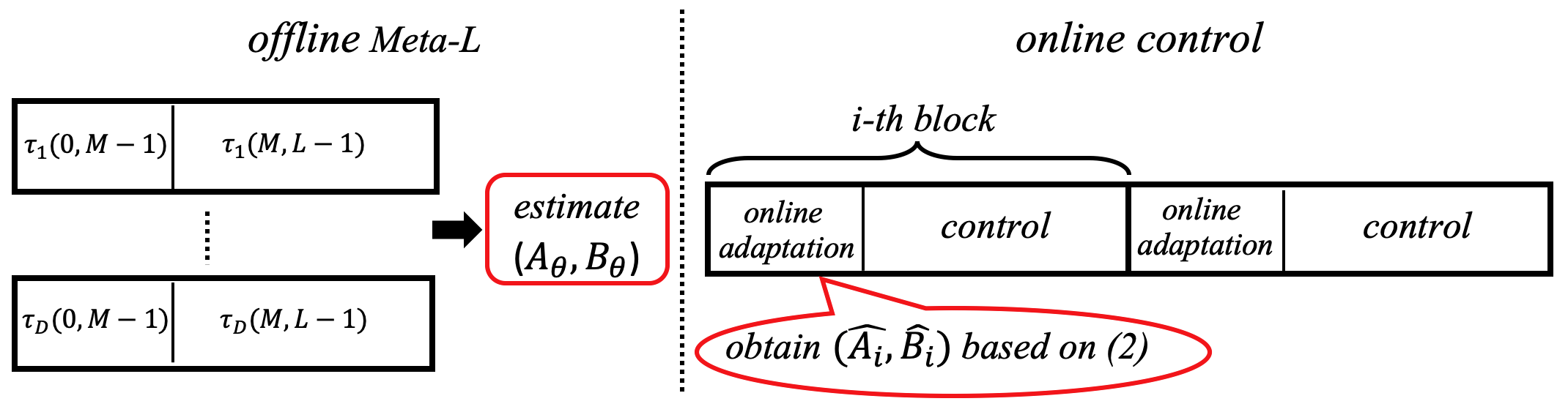}
\caption{Meta-L based system identification under the episodic block model.}
\label{Fig:meta}
\end{figure}

\emph{Offline Meta-L}: Without loss of generality, suppose there is  a sequence of arbitrary  \emph{realizations} over $D$ episodic trajectories $\{\tau_1, \tau_2, ..., \tau_D\}$, each with length $L$. The LTI dynamics for trajectory $\tau_d$ is given by 
    $x_{t+1,d}=A_dx_{t,d}+B_du_{t,d}+w_{t,d}$, for $t\in[0, L-1]$.
Following \cite{dean2017sample}, we assume that the system could be reset to zero initial state after each trajectory, \emph{for offline learning only}, to collect $D$ trajectories, i.e., $x_{0,d}=0$, and input $u_{t,d}\sim \mathcal{N}(0, \sigma_a^2I_m)$,  which not only simplifies the analysis but is also  important to deal with potentially unstable systems.

For convenience, we use  $\tau_d(i,j):=\{x_{i,d},u_{i,d},...,x_{j,d},u_{j,d},x_{j+1,d}\}$ to denote the sample trajectory from $i$ to $j$ within episodic block $d$. 
Let $z_{t,d}:=[(x_{t,d})^T, (u_{t,d})^T]^T$ and $\phi_{d}^T:=[A_d, B_d]$. The dynamics for block $d$ can be rewritten as 
    $x_{t+1,d}=\phi_d^T z_{t,d}+w_{t,d}$.
As is standard, we define the loss function as
    $\mathcal{L}(\tau(i, j), (A, B)):= \frac{1}{2}\sum_{k=i}^j \|x_{k+1}-Ax_k-Bu_k\|_2^2$.
Further, for each  block $d$, the samples collected in the first $M$ time steps, i.e., $\tau_d(0, M-1)$, are taken as the training set for that block, whereas the rest of the block, i.e., $\tau_d(M,L-1)$, serves as the testing set.
Denote $\hat{\phi}^T_{d}:=[\hat{A}_d, \hat{B}_d]$ as the estimated  model parameter of block $d$ and $\phi_{\theta}^T:=[A_{\theta}, B_{\theta}]$ as the Meta-L model initialization to be learned.
The {\it offline Meta-L problem} is given as follows:
\begin{small}
\begin{align}\label{problemformulation-1}
    &\underset{\phi_{\theta}}{\min}\ \ \ \ \ \ \ \ \ ~\sum\nolimits_{d=1}^D \mathcal{L}(\tau_d(M,L-1),\hat{\phi}_{d}),\\
    &\mbox{subject \  to}~ \ \ \hat{\phi}_{d}= \phi_{\theta}-\alpha\nabla \mathcal{L}(\tau_d(0, M-1),\phi_{\theta}),\nonumber
\end{align}\par
\end{small}
where $\alpha$ is the learning rate.
In general, the optimal Meta-L model initialization $\phi^*_{\theta}$ can be found as the solution to the above  optimization problem (see, e.g., \cite{finn2017model,lin2020collaborative}).

\emph{Online Adaptation}: As the system continuously evolves online (with no resetting), based on model initialization $\phi^*_{\theta}$, the model parameter $\hat{\phi}_i$ of a new episodic block $i$ can be obtained via {\it online adaptation} using  $M$ samples in the block, i.e.,
\begin{small}
\begin{align}\label{problemformulation-2}
    \hat{\phi}_i=\phi^*_{\theta}-\alpha \nabla\mathcal{L}(\tau_i(0,M-1),\phi^*_{\theta}).
\end{align}\par
\end{small}

In a nutshell, the Meta-L based system identification boils down to characterizing the solution to \eqref{problemformulation-1} and \eqref{problemformulation-2}. In particular, we seek to answer the following key questions:
What is the distance between the Meta-L initialization $\phi^*_{\theta}$ and the underlying true model parameter of a given block?
How do the training dataset size $M$, the testing dataset size $L-M$ and the number of trajectories $D$, impact this distance? What is the impact of the model similarity across different blocks on this distance? What is the estimation error after fast adaptation with a few samples only?

\section{Performance Analysis of Offline Meta-Learning}

Next, we  quantify the distance between the Meta-L initialization and the  model parameter of a given episodic block $j$, i.e., $\|\phi^*_{\theta}-\phi_j\|$. 
For convenience,  define,  for the training dataset of block $d$, 
\begin{align*}
x_{tr,d}:=[x_{1,d}, ..., x_{M,d}], \;
    z_{tr,d}:=[z_{0,d}, ..., z_{M-1,d}], \;  w_{tr,d}:=[w_{0,d}, ..., w_{M-1,d}]; 
\end{align*}\par
 and for the testing dataset of block $d$,
\[  
 x_{te,d}
 :=[x_{M+1,d}, ..., x_{L,d}], \;
 z_{te,d}:=[z_{M,d}, ..., z_{L-1,d}], \; 
 w_{te,d}:=[w_{M,d}, ..., w_{L-1,d}];
 \]
\begin{small}
\[ \Gamma_{t,d}=\diag(\sigma_a^2 G_{t,d}+\sigma_w^2 F_{t,d},\sigma_a^2 I_m), \;
 \Gamma_{t}=\diag(\sigma_a^2 G_{t}+\sigma_w^2 F_{t},\sigma_a^2 I_m), \;  
\underline{\Gamma}_{t}=\diag(\sigma_a^2 \underline{G}_{t}+\sigma_w^2 \underline{F}_{t},\sigma_a^2 I_m).
\]\par
\end{small}
Further, define the following matrices:
\begin{small}
\begin{align*}
&  Z_d= ( I- \alpha z_{tr,d}z_{tr,d}^T)z_{te,d}, Z=[Z_1,...,Z_D];\
\Pi_d=(z_{te,d}^T - \alpha z_{te,d}^T z_{tr,d} z_{tr,d}^T)\phi_d, \Pi=[\Pi_1^T,...,\Pi_D^T]^T  ;\\
&W_d=w_{te,d}^T - \alpha z_{te,d}^Tz_{tr,d}w_{tr,d}^T,  W=[W_1^T,..., W_D^T]^T;\
\Tilde{W}_d=x_{te,d}^T - \alpha z_{te,d}^T z_{tr,d} x_{tr,d}^T, \Tilde{W}=[\Tilde{W}_1^T,...,\Tilde{W}_D^T]^T.
\end{align*}\par
\end{small}
The following result characterizes the  Meta-L model initialization in terms of $Z$ and $\Tilde{W}$. 
\begin{restatable}{lem}{lemmasolution}
\label{lem:lem.solu}
For a matrix $Z$ we denote by $Z^{\dagger}$ its pseudo-inverse. The solution to the problem \eqref{problemformulation-1} is
\begin{align*}
    \phi^*_{\theta}=(Z^T)^{\dagger}\Tilde{W}=(Z^T)^{\dagger}(\Pi+W).
\end{align*}\par
\end{restatable}
 
It can be seen from  Lemma~\ref{lem:lem.solu} that the Meta-L based model initialization $\phi^*_{\theta}$ can be viewed as a weighted sum of the true model parameters for all $D$ trajectories, with perturbation incurred by the noise process.
Based on Lemma~\ref{lem:lem.solu}, we next investigate the distance between meta-initialization $\phi^*_{\theta}$ and the  true model $\phi_j$, $\|\phi^*_{\theta}-\phi_j\|$, aiming to quantify the impact of the sample size and the model similarity on the estimation error.  

To this end, we first apply  SVD to $Z^T$, i.e., $Z^T=U\Sigma V^T$ where $\Sigma, V\in\mathbb{R}^{(m+n)\times (m+n)}$ and $U\in\mathbb{R}^{D(L-M)\times (m+n)}$. Through careful manipulation,
we  have that
\begin{small}
\begin{align*}
\|\phi^*_{\theta}-\phi_j\|
=\|(Z^T)^{\dagger}(P+Q_w-\alpha Q_0)\|
\leq& \frac{1}{\sqrt{\lambda_{min}(ZZ^T)}}(\|U^TP\|+\|U^TQ_w\|+\alpha \|U^TQ_0\|),
\end{align*}\par
\end{small}
where {\small $P\triangleq \begin{bmatrix}
(z_{te,1}^T - \alpha z_{te,1}^T z_{tr,1} z_{tr,1}^T)(\phi_{1}-\phi_j) \\
(z_{te,2}^T - \alpha z_{te,2}^T z_{tr,2} z_{tr,2}^T)(\phi_{2}-\phi_j) \\
\vdots\\
(z_{te,D}^T - \alpha z_{te,D}^T z_{tr,D} z_{tr,D}^T)(\phi_{D} -\phi_j)
\end{bmatrix}$}, {\small $Q_w\triangleq \begin{bmatrix}
w_{te,1}^T \\
w_{te,2}^T \\
\vdots\\
w_{te,D}^T 
\end{bmatrix}$}
and {\small $Q_0\triangleq \begin{bmatrix}
 z_{te,1}^Tz_{tr,1}w_{tr,1}^T\\
 z_{te,2}^Tz_{tr,2}w_{tr,2}^T\\
\vdots\\
z_{te,D}^Tz_{tr,D}w_{tr,D}^T
\end{bmatrix}$}. 

A few key observations are in order. 
Intuitively, $U^TP$ encapsulates the impact of the model similarity across different blocks, whereas $U^TQ_w$ and $U^TQ_0$ capture the impact of the noise process from the testing dataset and the training dataset, respectively. 
 To find tight upper bounds on $\|\phi^*_{\theta}-\phi_j\|$, there are a few  challenges originating from the following facts:
1) (\emph{sample correlation}) Not only the elements of $Z$ are dependent, $Z$ is also correlated with $P$, $Q_w$ and $Q_0$;
2) (\emph{small episodic block size}) for a LTV system, the block size $L$ and the training dataset size $M$ could be small, which impedes the usage of the standard tools for analyzing in-block properties  based on large sample sizes.

To tackle these challenges, we  devise a new two-scale martingale small-ball approach to deal with the correlation among elements of $Z$ and small block sizes, and then use a  martingale-Chernoff bound method  \cite{simchowitz2018learning} to analyze the correlation between the system state and the random disturbance. 

\subsection{Lower Bound on $\lambda_{min}(ZZ^T)$: A Two-Scale Small-Ball Approach}
Observe that
\begin{small}
\begin{align}\label{lbzz1}
    ZZ^T=\sum\nolimits_{d=1}^D ( I- \alpha z_{tr,d}z_{tr,d}^T)z_{te,d}z^T_{te,d}(I- \alpha z_{tr,d}z_{tr,d}^T).
\end{align}\par
\end{small}
\begin{figure}
    \centering
    \includegraphics[width=0.46\textwidth]{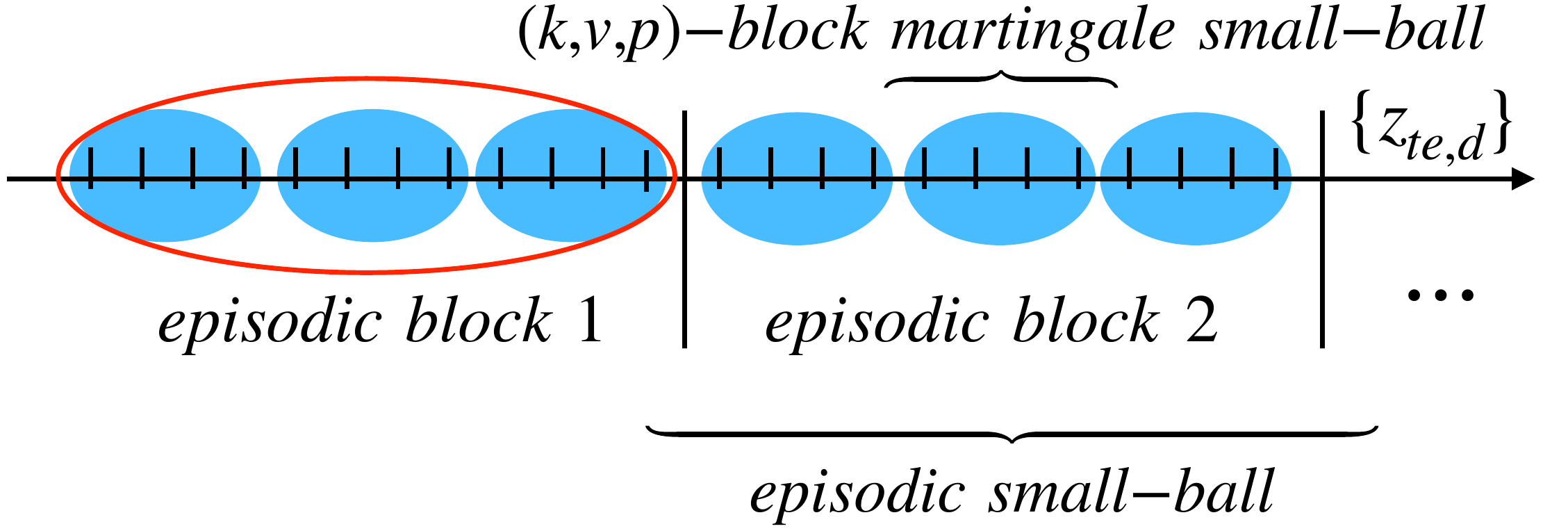}
    \caption{A two-scale small-ball view of episodic blocks.}
    \label{fig:twoscale}
\end{figure}
It follows  that it would suffice to find a lower bound for $\sum_{d=1}^D z_{te,d}z^T_{te,d}$,  provided that we can also find a  lower bound on $I- \alpha z_{tr,d}z_{tr,d}^T$ uniformly for all $D$ trajectories. Observe that $\sum_{d=1}^D z_{te,d}z^T_{te,d}$ is the sum of  sample covariances over $D$ episodic blocks, each being an independent martingale process \emph{conditioned on the realized model parameters in $D$ historical blocks}. With this insight,  we devise a two-scale small-ball method, in which a block martingale small ball method is used to handle the correlation structure within each block and then Mendelson's small ball method is applied across episodic blocks jointly to find a lower bound accordingly. Such a two-scale approach enables us to quantify the correlation structure with each block and then exploit the conditional independence of observations across blocks, given realized model parameters,  thereby yielding a sharper lower bound.

Specifically,  as illustrated in Figure \ref{fig:twoscale}, we treat the testing sequence $\{z_{te,d} \}$ for all $D$ blocks together as a `super-sequence'. Then  this `super-sequence' can be treated as a combination of $D$ independent martingale processes where within each block the system state sequence is a martingale process with filtration given by $\mathcal{F}_t:=\sigma(z_{M,d},...,z_{t,d}, w_{M,d},..., w_{t,d})$. By dividing the testing sequence $z_{te,d}$ of each block $d$ as a set of mini-blocks with block size $k$, a martingale small-ball method is applied to evaluate the sample covariances within the entire testing sequence. For the sake of completeness, we restate a generalized martingale small-ball condition \cite{simchowitz2018learning}  (cf.~\cite{mendelson2014learning} which is developed to deal with the correlation within  each block).

\begin{definition}[Martingale Small-Ball]
Let $(Z_t)_{t\geq 1}$ be an $\{\mathcal{F}_t\}_{t\geq 1}$-adapted random process taking values in $\mathbb{R}$. We say $(Z_t)_{t\geq 1}$ satisfies the $(k,\nu,p)$-block martingale small-ball (BMSB) condition if, for any $j\geq 0$, one has $\frac{1}{k}\sum_{i=1}^k \mathbb{P}(|Z_{j+i}|\geq\nu)\geq p$ almost surely. Given a process $(X_t)_{t\geq 1}$ taking values in $\mathbb{R}^d$, we say that it satisfies the $(k, \Gamma_{sb}, p)$-BMSB condition for $\Gamma_{sb}\succ 0$ if, for any fixed $u\in\mathcal{S}^{d-1}$, the process $Z_t:=\langle u, X_t \rangle$ satisfies $(k, \sqrt{u^T\Gamma_{sb} u}, p)$-BMSB.
\end{definition}

 It can be shown that for each block, $\{z_{t,d}\}_{t=M}^{L-1}$ satisfies the $(k,\Gamma_{\lfloor k/2\rfloor,d},p)$-BMSB condition where $k\in[1, \lfloor L-M\rfloor/2]$.
 Next, based on the observation that in fact each block as a whole  also satisfies the small-ball condition, we  apply Mendelson's small-ball method to evaluate the sample covariances of the entire super-sequence. 
define
 \[\Bar{\lambda}
 \triangleq M^3\|\Bar{B}\|^2\Big(1+3\sqrt{\log\frac{10DM}{\delta}}\Big)^2\max\{m\sigma_a^2,n\sigma_w^2\}.\]
It is clear that $\Bar{\lambda}$ is of order $\Tilde{O}\left(M^3\max\{m\sigma_a^2,n\sigma_w^2\}\right)$, where the term $\Tilde{O}(\cdot)$ encompasses some constants and ploylog factors.
With this two-scale small-ball approach,
we can obtain the following result about the lower bound on $\lambda_{min}(ZZ^T)$ as follows.
\begin{restatable}{lem}{lemmaa}
\label{lem:lem.2}
Suppose that the learning rate satisfies that
$0<\alpha<1/\Bar{\lambda}$,
and for any $k\in [1,\lfloor L-M\rfloor/2]$, the number of blocks $D$, training dataset size per block $M$ and the block size $L$ satisfy that
\begin{small}
\begin{align*}
    D\left(1-e^{-\frac{p^2\lfloor (L-M)/k\rfloor}{8}}\right)^2
    \geq D_{\lambda}
    =\Tilde{O}\left(\frac{(m+n)\lambda_{max}^2(\Gamma_{\lfloor k/2\rfloor})}{\lambda^2_{min}(\underline{\Gamma}_{\lfloor k/2\rfloor})}\right).
\end{align*}\par
\end{small}
Then, for any $\delta\in(0,1)$, it follows that with probability  $1-\delta$, 
\begin{small}
\begin{align*}
    \lambda_{min}(ZZ^T)\geq& \frac{D(L-M)p^2(1-e^{-\frac{p^2\lfloor (L-M)/k\rfloor}{8}})(1-\alpha\Bar{\lambda})^2\lambda_{min}(\underline{\Gamma}_{\lfloor k/2\rfloor})}{48}.
\end{align*}\par
\end{small}
\end{restatable}
 
The dependence on $\lambda_{min}(\underline{\Gamma}_{\lfloor k/2\rfloor})$ directly manifests the impact of the system ``excitability" on the lower bound of the minimum eigenvalue of the sample covariance matrix. In particular, the more the system is excited by the noise process, the larger $\lambda_{min}(ZZ^T)$ is. As expected, note that the impact of $\lambda_{max}(\Gamma_{M-1})$, i.e., the Gramians of the training process, can be controlled via tuning the learning rate $\alpha$. 

\subsection{Upper Bounds on $\|U^TP\|$, $\|U^TQ_w\|$ and $\|U^TQ_0\|$}

\emph{Upper bound on $\|U^TP\|$}: The term $\|U^TP\|$ captures the impact of the model similarity across different blocks on the estimation gap. To obtain a fundamental understanding of this impact, we impose the following assumption.
\begin{assumption}\label{assum1}
There exists a positive number $D_0$ such that when $D\geq D_0$, for any sequence $\{\phi_d, d=1, \ldots, D\}$ and $\phi_i$ in the compact set $\Theta $ the following inequalities hold:
\begin{small}
\begin{align}
    \frac{1}{D}\sum\nolimits_{d=1}^D \|\phi_d-\phi_i\|\leq \eta,
    \ 
    \frac{1}{D}\sum\nolimits_{d=1}^D\|\phi_d-\phi_i\|^2\leq V_{\phi}.
    \label{assumption1-condition}
    \end{align}\par
\end{small}
\end{assumption}

Assumption \ref{assum1} encapsulates the model similarity between offline episodic blocks  and any  block in terms of the average distance  (and the corresponding variance) between model parameters. This condition is mild in the sense that there  exist  $\eta$ and $V_{\phi}$ for (\ref{assumption1-condition}) to hold with high probability when $\{\phi_i\}$ follows some distribution within the support in $\Theta$. For instance, when $\phi_i$  follows a uniform   distribution in  $\Theta$, an upper bound on $\mathbb{E}[\|\phi_d-\phi_i\|]$ can be found in \cite{burgstaller2009average}. 
Based on the Law of Large Numbers, the sample average $\frac{1}{D}\sum_{d=1}^D \|\phi_d-\phi_i\|$ is close to the expectation $\mathbb{E}[\|\phi_d-\phi_i\|]$ when $D$ is large enough.
Under Assumption 1,  we have the following upper bound  on $\|U^TP\|$. 

\begin{restatable}{lem}{lemmab}
\label{lem:lem.3}
Suppose Assumption 1 holds. With probability $1-\delta$, the following inequality holds:
\begin{small}
\begin{align*}
   \|U^TP\|\leq \frac{D\eta\lambda_{max}\left(L\Gamma_{L-1}-\frac{\alpha}{2}\left(\sum_{t=0}^{M-1} \underline{\Gamma}_{t}\right)^2\right)+(L-M)\sqrt{\frac{1}{\delta}}\sqrt{C_vDV_{\phi}}}{\sqrt{\lambda_{min}(ZZ^T)}},
\end{align*}\par
\end{small}
where $C_v$ is some constant.
\end{restatable}


We relegate the bounds on $\|U^TQ_w\|$ and $\|U^TQ_0\|$ to the appendix, and state the main ideas here. 

\emph{Upper bound on $\|U^TQ_w\|$}: Along the lines in \cite{simchowitz2018learning}, we  study the quantities with $U$ in terms of $Z$, since $\|U^TQ_w\|\leq\sup_{v\in\mathcal{S}^{n-1},u\in\mathcal{S}^{m+n-1}\setminus \{0\}}\frac{u^TZQ_wv}{\|Z^Tu\|}$, with $\mathcal{S}^{n-1}$ being the unit sphere in $\mathbb{R}^n$.
With a closer look in $Z$, it can be seen that the weighted system state sequence remains a martingale process for each block.
The key idea here is to control the deviation of sum of independent sub-Gaussian martingale sequences by using the martingale-Chernoff bound approach (cf.~ \cite{simchowitz2018learning}). 

\emph{Upper bound on $\|U^TQ_0\|$}: Each row block $Q_{0,d}=z_{te,d}^Tz_{tr,d}w_{tr,d}^T$ in the matrix $Q_0$ is intimately related to how the training noise $w_{tr,d}^T$ is amplified during the system evolution, i.e.,  how the system is excited by  $w_{tr,d}^T$. Once again, 
we study this error term with $Z$. 
However, the martingale-Chernoff bound approach is not applicable here due to the complicated correlation structure between $z_{te,d}$ and $w_{tr,d}$ within each block. Therefore, we  develop bounds for $u^TZQ_0v$ and $\|Z^Tu\|$ separately.

\subsection{Summary of Main Results}
Summarizing, we have the following theorem on the distance between the Meta-L based model initialization $\phi^*_{\theta}$ and the true  model parameter $\phi_j$ for a given episodic block $j$.
\begin{restatable}{theorem}{theorema}
\label{thm:thm.1}
Suppose Assumption 1 holds, 
and the learning rate $\alpha$ and block sizes  satisfy:
\begin{small}
\begin{align*}
0<\alpha<\frac{1}{\Bar{\lambda}},\;
\mbox{and } \ 
    D\left(1-e^{-\frac{p^2\lfloor (L-M)/k\rfloor}{8}}\right)^2\geq \max\{D_{\lambda}, D_0\}.
\end{align*}\par
\end{small}
Then, for a given  $\phi_j$ and some $C_0$,  the following inequality holds with probability  $1-5\delta$:
\begin{small}
\begin{align*}
    \|\phi^*_{\theta}-\phi_j\|\leq C_0\eta+\Tilde{O}\left(D^{-1/2}\right)\sqrt{V_{\phi}}+\Tilde{O}\left([D(L-M)]^{-1/2}\right),
\end{align*}\par
\end{small}
where
\begin{small}
\begin{align*}
C_0=&\frac{48\lambda_{max}\left(L\Gamma_{L-1}-\frac{\alpha}{2}\left(\sum_{t=0}^{M-1} \underline{\Gamma}_{t}\right)^2\right)}{ (L-M)p^2(1-e^{-\frac{p^2\lfloor (L-M)/k\rfloor}{8}})(1-\alpha\Bar{\lambda})^2\lambda_{min}(\underline{\Gamma}_{\lfloor k/2\rfloor})}.
\end{align*}\par
\end{small}
\end{restatable}

Based on Theorem \ref{thm:thm.1}, we have the following important remarks.

\noindent{\bf Benefits of Meta-Learning.}
To better understand the benefits of meta-learning in system identification of the linear time-varying system, we first take a closer look at the optimal model initialization $\phi_{\theta}^*$, and compare it with the solution by using least square estimation  over all $D$ blocks (i.e., setting $\alpha=0$ and using all $L$ samples for each block) in a noiseless setup. Let $z_d=[z_{0,d},...,z_{L-1,d}]$. As indicated by Lemma \ref{lem:lem.solu}, $\phi_{\theta}^*$ is a weighted sum of the true model parameters $\{\phi_d\}$ for all $D$ trajectories, where the weight for block $d$ is captured by $( I- \alpha z_{tr,d}z_{tr,d}^T)z_{te,d}z^T_{te,d}(I- \alpha z_{tr,d}z_{tr,d}^T)$. Observe that the solution obtained by least square estimation is also a weighted sum of $\{\phi_d\}$ where the weight for block $d$ is captured by $z_dz_d^T$ instead. Intuitively, the least square estimator would always assign higher weights to the blocks with `larger' model parameters, which would in turn  generate a model initialization closer to the `larger' model parameters. 
In stark contrast, meta-learning  assigns  more `balanced' weights across blocks, in the sense that a larger  $I- \alpha z_{tr,d}z_{tr,d}^T$ is used to scale  a smaller  $z_{te,d}z^T_{te,d}$ in an intrinsic manner (see Figure \ref{fig:example} for examples). Therefore, meta-learning paves a better way to utilize the information in each episodic block, thus leading to a smaller estimation error for a suitably selected learning rate $\alpha$. 

\begin{figure*}
    \centering
    \subfigure[]{\label{fig:ma}\includegraphics[width=40mm]{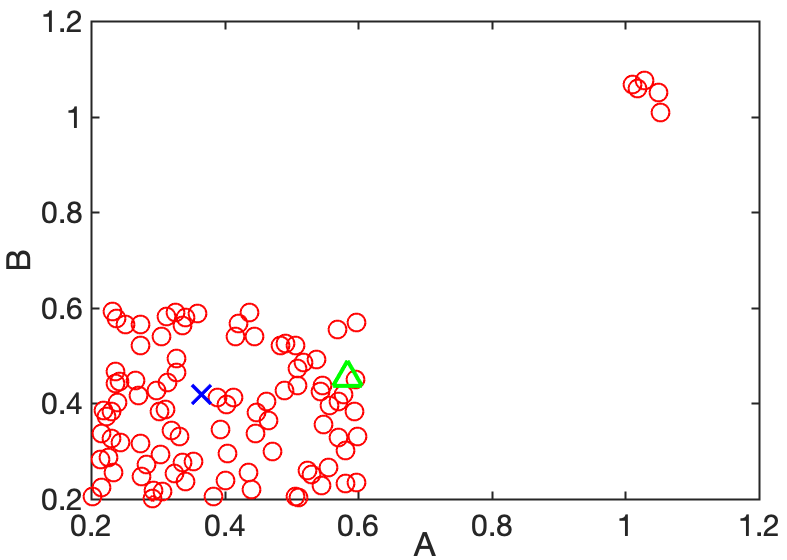}}
    \subfigure[ ]{\label{fig:mb}\includegraphics[width=40mm]{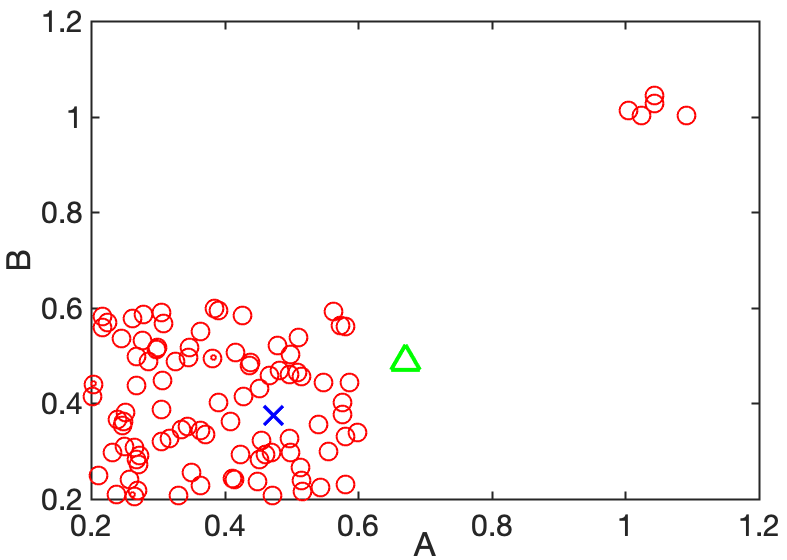}}
    \subfigure[]{\label{fig:mc}\includegraphics[width=40mm]{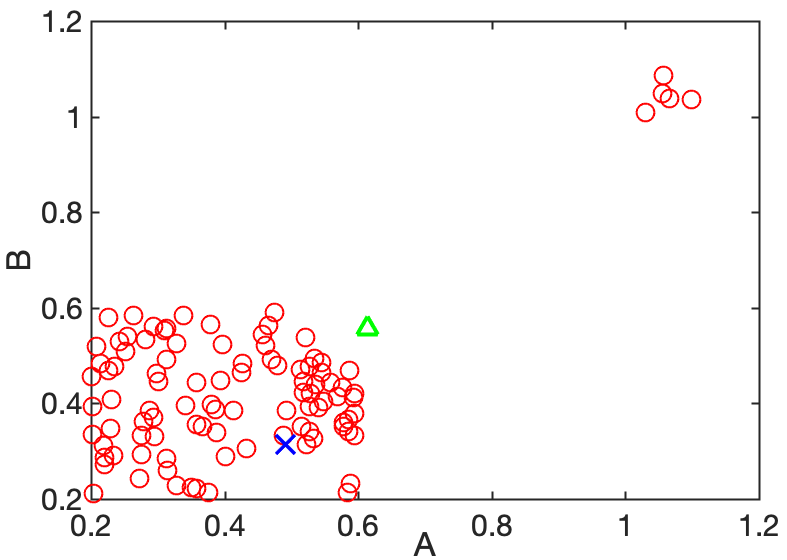}}
    \subfigure[ ]{\label{fig:md}\includegraphics[width=40mm]{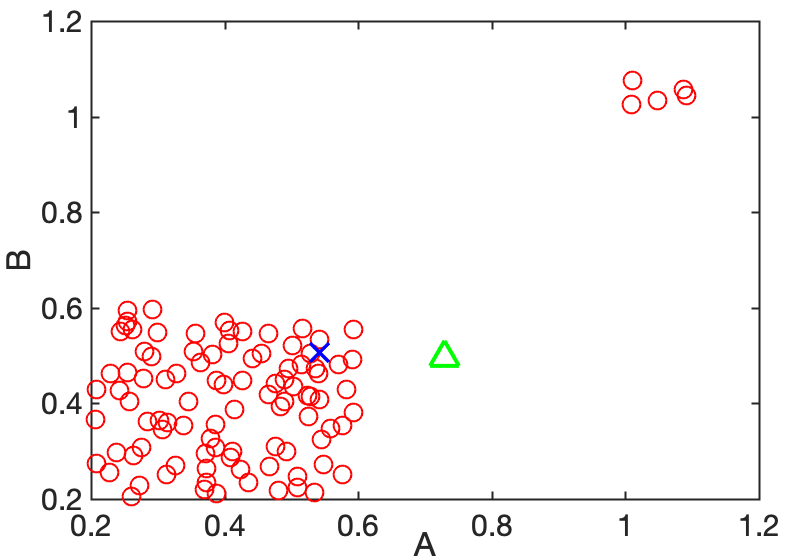}}
    \caption{Illustration of the solution by using meta-learning vs. that based on least square estimation. The red circles represent the model parameters for each block, the blue cross is the solution using meta-learning, and the green triangle is the solution using least square estimation.}
    \label{fig:example}
\end{figure*}

\noindent{\bf Impact of $M$, $L$, $D$.}
It can be seen that the offline Meta-L  distance $\|\phi^*_{\theta}-\phi_j\|$ can be decomposed into two parts: 1) the error incurred by model dissimilarity captured by
    $C_0 \eta+\Tilde{O}\left(D^{-1/2}\right)\sqrt{V_{\phi}}$, and 2) the error caused by finite sample sizes and the random noise process  captured through $\Tilde{O}\left([D(L-M)]^{-1/2}\right)$.
   In particular,  this  distance  is linear in model average distance $\eta$, while the impact of model variation decays at a rate of $\Tilde{O}\left(D^{-1/2}\right)$.
   Besides, one can choose a block length $L$ as desired to control the approximation accuracy of the episodic block model for the LTV system, and the  offline Meta-L algorithm can yield a good model initialization as long as there is a large number of blocks (large $D$) available (also corroborated by the experiments in the appendix).
   
   The impact of the training dataset size $M$ on the distance $\|\phi_{\theta}^*-\phi_j\|$ is not easy to tell, considering that both $C_0$ and the terms in $\Tilde{O}(\cdot)$ depends on $M$ directly or implicitly. To get some insights on the impact of $M$, we consider a noiseless scalar system where both $A_d$ and $B_d$ are scalars. In this case, the distance $\|\phi_{\theta}^*-\phi_j\|$ is bounded from above by $C_0 \eta+\Tilde{O}\left(D^{-1/2}\right)\sqrt{V_{\phi}}$, and it suffices to understand how $C_0$ changes with $M$. When $\{A_d\}$ are in an open unit disk, by taking a detailed look into $C_0$, $C_0$ can be regarded as a function of $M\in[1,L-1]$, i.e., $h(M)=\frac{bL-f(M)}{L-M}$, where $f(M)$ is a function of $M$ with $\nabla f(M)\geq 0$ and $\nabla^2 f(M)\leq 0$. Let $g(M)=f(M)+(L-M)\nabla f(M)$. It can be seen that $\nabla h(M)=\frac{bL-g(M)}{(L-M)^2}$. Since $\nabla g(M)=(L-M)\nabla^2 f(M)\leq 0$, $\nabla h(M)$ could be (1) always non-negative, or (2) first negative and then non-negative as $M$ increases, which indicates that the distance $\|\phi_{\theta}^*-\phi_j\|$ may (1) increase or (2) first decrease and then increase with $M$. In a nutshell, there exists an optimal $M$ for which the estimation gap is minimized.

\noindent{\bf Impact of Model Similarity.}
In a noiseless system, given a set of $D$ episodic blocks for offline Meta-L, the best upper bound on $\|\phi^*_{\theta}-\phi_j\|$ that one can achieve is intimately related to $\eta$, namely the average distance between any $\{\phi_d, d=1, \ldots, D \}$ and $\phi_j$ in $\Theta$. In the presence of stochastic disturbance in the LTV system, Theorem \ref{thm:thm.1} reveals that our upper bound only  degrades to $C_0\eta$ for a small constant $C_0$, and our result holds for any realization sequence of $\{\phi_d\}$ regardless of the correlation structure among them. 

It is clear that when $D$ goes to infinity, the last two error terms in Theorem \ref{thm:thm.1} diminish, whereas there still exists a gap between $\phi^*_{\theta}$ and $\phi_j$ due to the model dissimilarity. This makes sense since  meta-learning is used to find a good model initialization point `close to' all offline task models. Note that when the system is LTI (all blocks have same model parameters), the term $C_0\eta+\Tilde{O}\left(D^{-1/2}\right)\sqrt{V_{\phi}}$ diminishes simply because the distance between block parameters is zero. With $D(L-M)$ samples for meta-learning, Theorem \ref{thm:thm.1} recovers the nearly minimax optimal bound in \cite{simchowitz2018learning} for the identification of LTI systems with least square estimation.

\section{Error Bound for Online Adaptation}
Given the Meta-L initializer $\phi^*_{\theta}$, the model estimator $\hat{\phi}_i$ for a new block $i$ can be obtained via online adaptation by using its samples only. Since there is no resetting in online adaptation, the states would `persist' across different blocks, which however does not affect the performance of online adaptation using samples in the new block. So we focus on the adaptation of a single block. 
In what follows,  
we quantify the estimation error after  online adaptation. 
Once again, the challenge herein is originating from the small sample size and the sample correlation, such that existing techniques based on large sample sizes  and assuming that the system state is near the steady state would not work well here. 

In light of this, instead of using the one-step gradient descent algorithm where all $M$ samples are used in one shot as in \eqref{problemformulation-2}, we consider a $M$-step gradient descent algorithm with a trajectory of correlated samples $\{(x_{t,i},u_{t,i},x_{t+1,i})\}_{t=0}^{M-1}$ following the linear dynamics $x_{t+1,i}=A_i x_{t,i}+B_i u_{t,i}+w_{t,i}=\phi_i^T z_{t,i}+w_{t,i}$, and only one sample is used for every step, i.e.,
\begin{align}\label{lsa}
    \hat{\phi}_i(t+1)=\hat{\phi}_i(t)-\alpha \hat{g}_t(\hat{\phi}_i(t)),
\end{align}\par
where $\hat{g}_t(\hat{\phi}_i(t))\triangleq \nabla\mathcal{L}
    =z_{t,i}z_{t,i}^T\hat{\phi}_i(t)-z_{t,i}x_{t+1,i}^T$ for $0\leq t\leq M-1$ and $\hat{\phi}_i(0)=\phi^*_{\theta}$. It is worth noting that equation \eqref{lsa} turns out to be a linear stochastic approximation (LSA) of the underlying model parameter $\phi_i$. \emph{Hence, the problem  reduces to finding the finite time error bound of LSA with a trajectory of correlated samples following the linear dynamics}.


\textbf{Preliminary on LQR}: Based on the well-known result that the LQR problem in the LTI system can be solved with a linear feedback policy $u_t=Kx_t$, we assume that each block $i$  evolves with a stabilizing controller $K$ (cf. \cite{dean2017sample} for the controller synthesis) during the online adaptation, i.e., $A_i+B_iK$ is a stable matrix with spectral radius $\rho_i<1$, thereby generating a trajectory of $M$ samples. 
The stationary distribution of the linear dynamics $\{x_t\}$ is $\mathcal{N}(0, P_{\infty})\triangleq \nu_{\infty}$ where $P_{\infty}$ uniquely solves the Lyapunov equation 
$(A_i+B_iK)P_{\infty}(A_i+B_iK)^T-P_{\infty}+I=0$.
Furthermore, it has been shown in \cite{tu2017least} that the $\beta$-mixing coefficient of LTI for $t\geq 1$ with stable $A_i+B_iK$ is
\begin{align*}
    \beta(t)=\sup\nolimits_{k\geq 1}\mathbb{E}_{x\sim \nu_{k}}[\|\mathbb{P}_{x_t}(\cdot|x_0=x)-\nu_{\infty}\|_{tv}]
    \leq  C_m\rho^t
\end{align*}\par
for constant $C_m$, any $\rho\in(\rho_i,1)$ and distribution $\nu_{k}$, where $\|\cdot\|_{tv}$ refers to the total-variation norm.


\textbf{Finite time error:} 
 Along the same line as in \cite{bhandari2018finite}, let $C_{\phi}=\max\nolimits_{\phi_i\in\Theta}\|\phi_i\|$
 and define
$\proj_{\phi}\triangleq \arg \min_{\phi':\|\phi'\|\leq C_{\phi}}\|\phi-\phi'\|_2$
as the projection operator onto a norm ball of radius $C_{\phi}$ to circumvent unreliable estimators  outside the ball. 
Further, to mitigate the impact of the  outlier samples with extremely large state values,  the outlier samples $z_{t,i}$ are projected onto a norm ball of radius $C_z<\infty$:
\begin{align}\label{projectz}
\Tilde{z}_{t,i}\triangleq \arg \min\nolimits_{z'_{t,i}:\|z'_{t,i}\|\leq C_{z}}\|z_{t,i}-z'_{t,i}\|_2.
\end{align}\par
We caution that $C_z$ should be  set to be large enough such that only extreme outlier samples  would be projected as in \eqref{projectz}. Since there are only a few samples for online adaptation in each block, the likelihood of this happening is very small. We use this projection for online learning because technically it is impossible to bound the spectrum of $z_tz_t^T$ for only a few  samples.
It follows that the `modified' gradient  is given by $g_t(\hat{\phi}_i(t))=\Tilde{z}_{t,i}\Tilde{z}_{t,i}^T\hat{\phi}_i(t)-\Tilde{z}_{t,i}(\phi_i^T\Tilde{z}_{t,i}+w_{t,i})^T$.  
Applying $g_t(\hat{\phi}_i(t))$ and operator $\proj_{\phi}$ to \eqref{lsa}, 
the update step can be transformed to
    $\hat{\phi}_i(t+1)=\proj_{\phi}(\hat{\phi}_i(t)-\alpha g_t(\hat{\phi}_i(t)))$.



For convenience,
define 
$\hat{P}_{\infty}=\begin{bmatrix}
    I & K^T
    \end{bmatrix}^T P_{\infty}\begin{bmatrix}
    I & K^T
    \end{bmatrix}$.
Using  a similar approach in the finite time analysis of TD-learning \cite{bhandari2018finite},
we have the following result characterizing the finite time error.

\begin{restatable}{proposition}{propositiona}
\label{thm:thm.2}
Suppose that the system evolves with a stabilizing controller $K$ during online adaptation and that
the learning rate $\alpha$   satisfies  $\alpha<\frac{1-\rho}{2\lambda_{min}(\hat{P}_{\infty})}$. With the model initialization $\phi^*_{\theta}$, we have
            \begin{small}
            \begin{align*}
            \mathbb{E}[\|\hat{\phi}_i&(M)-\phi_i\|^2]\leq \frac{\alpha C_g}{2\lambda_{min}(\hat{P}_{\infty})}
            +[1-2\alpha \lambda_{min}(\hat{P}_{\infty})]^M\left(\|\phi^*_{\theta}-\phi_i\|^2+\frac{2\alpha M\Tilde{C}_{\phi}}{1-2\alpha \lambda_{min}(\hat{P}_{\infty})}\right),
            \end{align*}\par
            \end{small}
 for some constant $C_g$ and $\Tilde{C_{\phi}}$.           
    
\end{restatable}
Clearly, for a sufficiently small learning rate, the gap between the parameter $\hat{\phi}_i(M)$ and the underlying parameter $\phi_i$ decreases exponentially. Moreover, the larger $\lambda_{min}(\hat{P}_{\infty})$ is, the faster the convergence is. Different controllers for the LQR in block $i$ can be then designed based on the model estimation using online adaptation, with performance characterized by the model estimation error (see Appendix~H for details). 
It is worth noting that for a large sample size when the system nearly approaches its steady state, the finite time error bound about LSA in \cite{srikant2019finite} can be directly applied here, which indicates an exponential decay about the mean squared estimation error. Note that the result in \cite{simchowitz2018learning} shows that the maximum squared estimation error for the least square estimator (LSE) decays linearly with the sample size when the sample size is large enough. Our simulation studies (relegated to Appendix~J) indicate that when the underlying model is near the Meta-L initialization, the recursive LSA starting from this initialization clearly outperforms  the LSE for small sample sizes.

\section{Conclusions}

System identification plays a critical role in characterizing the fundamental limits of reinforcement learning algorithms in a LTV system. In this study, we  propose an innovative episodic block model for the LTV system and leverage  meta-learning for  system identification therein.
We carry out a comprehensive non-asymptotic analysis of the performance of Meta-L based system identification with correlated samples. With the proposed two-scale martingale small-ball approach for offline Meta-L, the derived upper bound on the distance between Meta-L based model initialization and the underlying model parameters is sharp and   encapsulates the impact of the model similarity and the sample size on system identification in LTV systems. Furthermore, by leveraging recent advances in linear stochastic approximation,  our results on online adaptation with correlated samples  indicate that the error between the model estimation and the underlying model decays exponentially.

\section*{Acknowledgement}

The authors would like to thank Prof. R. Srikant's stimulating discussions and valuable comments on an earlier version.

\bibliography{citation}
\bibliographystyle{plain}

\newpage
\appendix
\section{Proof of Lemma \ref{lem:lem.solu}}
\lemmasolution*

We first calculate the one-step gradient descent update. Observe that 

\begin{equation}
\begin{aligned}
\mathcal{L}(\tau_d(M, L-1), \hat{\phi}_d^T ) & \triangleq  \frac{1}{2}\left(\sum_{t=M}^{L-1} \left\| x_{t+1} - \hat{\phi}_d^T
\begin{bmatrix} x_{t} \\ u_{t} \end{bmatrix} 
\right\|_2^2 \right) \\
& = \frac{1}{2}\left( \sum_{t=M}^{L-1} (x_{t+1} - \hat{\phi}_d^T z_t)^T(x_{t+1} - \hat{\phi}_d^T z_t) \right)\\
& =  \frac{1}{2} \left(\sum_{t=M}^{L-1} \left(\|x_{t+1}\|_2^2 + (-x_{t+1})^T\hat{\phi}_d^Tz_t + z_t^T\hat{\phi}_d(-x_{t+1})+z_t^T\hat{\phi}_d\hat{\phi}_d^Tz_t \right) \right). 
\end{aligned}
\label{loss:K}
\end{equation}
It follows that
\begin{equation}
\begin{aligned}
\frac{\partial \mathcal{L}(\tau_d(M, L-1), \hat{\phi}_d^T )}{\partial \hat{\phi}_d^T } = z_{te,d}z_{te,d}^T \hat{\phi}_d - z_{te,d}x_{te,d}^T.
\end{aligned}
\end{equation}
Similarly, the derivative of $\mathcal{L}(\tau_d(0, M-1), \phi_{\theta}^T ) $ can be obtained as follows:
\begin{equation*}
\begin{aligned}
\mathcal{L}(\tau_d(0, M-1), \phi_{\theta}^T ) & \triangleq \frac{1}{2} \left
( \sum_{t=0}^{M-1} \left\| x_{t+1} - \phi_{\theta}^T 
\begin{bmatrix} x_{t} \\ u_{t} \end{bmatrix}
\right\|_2^2  \right) 
\end{aligned}
\end{equation*}
and
\begin{equation}
\begin{aligned}
\frac{\partial \mathcal{L}(\tau_d(0, M-1), \phi_{\theta}^T)}{\partial \phi_{\theta}^T} =  z_{tr,d}z_{tr,d}^T \phi_{\theta}- z_{tr,d}x_{tr,d}^T.
\end{aligned}
\end{equation}
Then, the relation between $\hat{\phi}_d$ and $\phi_{\theta}$ follows immediately:
\begin{equation}
\begin{aligned}
\hat{\phi}_d &= \phi_{\theta} - \alpha z_{tr,d}z_{tr,d}^T\phi_{\theta} + \alpha z_{tr,d}x_{tr,d},\\
\frac{\partial \hat{\phi}_d}{\partial \phi_{\theta}} &= I - \alpha z_{tr,d}z_{tr,d}^T.
\end{aligned}
\label{theta:d}
\end{equation}

Now we are ready to compute the optimal meta-parameter $\phi^*_{\theta}$. For convenience, we define the meta-learning objective function as:
\begin{equation*}
F(\phi_{\theta})\triangleq   \sum_{d=1}^{D}\mathcal{L}(\tau_{d}(M, L-1), \hat{\phi}_d).
\end{equation*}


It can be shown that by setting
\begin{equation*}
\begin{aligned}
\frac{\partial F(\phi_{\theta})}{\partial \phi_{\theta}}=&  \sum_{d=1}^{D} { ( I- \alpha z_{tr,d}z_{tr,d}^T)( z_{te,d}z_{te,d}^T ( I- \alpha z_{tr,d}z_{tr,d}^T)\phi_{\theta})} \\ 
&+  \sum_{d=1}^{D} { ( I- \alpha z_{tr,d}z_{tr,d}^T)(\alpha z_{te,d}z_{te,d}^T z_{tr,d} x_{tr,d}^T- z_{te,d}x_{te,d}^T)} \\
\triangleq& 0,
\end{aligned}
\end{equation*}
we can have that
\begin{align}
&\sum_{d=1}^{D} { ( I- \alpha z_{tr,d}z_{tr,d}^T)( z_{te,d}z_{te,d}^T ( I- \alpha z_{tr,d}z_{tr,d}^T)\phi_{\theta})}\nonumber\\
=& \sum_{d=1}^{D} { ( I- \alpha z_{tr,d}z_{tr,d}^T)(z_{te,d}x_{te,d}^T - \alpha z_{te,d}z_{te,d}^T z_{tr,d} x_{tr,d}^T)}. 
\end{align}

Let $Z_d= ( I- \alpha z_{tr,d}z_{tr,d}^T)z_{te,d}$ and $Z:=[Z_1,...,Z_D]$. It follows that
\begin{equation}
\begin{aligned}
ZZ^T\phi_{\theta} &= Z
\begin{bmatrix}
x_{te,1}^T - \alpha z_{te,1}^T z_{tr,1} x_{tr,1}^T\\
x_{te,2}^T - \alpha z_{te,2}^T z_{tr,2} x_{tr,2}^T\\
\vdots\\
x_{te,D}^T - \alpha z_{te,D}^T z_{tr,D} x_{tr,D}^T\\
\end{bmatrix}	 
\end{aligned}
\end{equation}
which indicates that 
\begin{equation}
\begin{aligned}
\phi^*_{\theta} &= (Z^T)^{\dagger}
\begin{bmatrix}
x_{te,1}^T - \alpha z_{te,1}^T z_{tr,1} x_{tr,1}^T\\
x_{te,2}^T - \alpha z_{te,2}^T z_{tr,2} x_{tr,2}^T\\
\vdots\\
x_{te,D}^T - \alpha z_{te,D}^T z_{tr,D} x_{tr,D}^T\\
\end{bmatrix}\\
&=(Z^T)^{\dagger}\Tilde{W}\\
&= (Z^T)^{\dagger}
\begin{bmatrix}
z_{te,1}^T\phi_{1} + w_{te,1}^T - \alpha z_{te,1}^T z_{tr,1} (z_{tr,1}^T\phi_{1} + w_{tr,1})\\
z_{te,2}^T\phi_{2} + w_{te,2}^T - \alpha z_{te,2}^T z_{tr,2} (z_{tr,2}^T\phi_{2} + w_{tr,2})\\
\vdots\\
z_{te,D}^T\phi_{D} + w_{te,D}^T - \alpha z_{te,D}^T z_{tr,D} (z_{tr,D}^T\phi_{D} + w_{tr,D})
\end{bmatrix}\\
&= (Z^T)^{\dagger}\left( 
\begin{bmatrix}
(z_{te,1}^T - \alpha z_{te,1}^T z_{tr,1} z_{tr,1}^T)\phi_{1} \\
(z_{te,2}^T - \alpha z_{te,2}^T z_{tr,2} z_{tr,2}^T)\phi_{2} \\
\vdots\\
(z_{te,D}^T - \alpha z_{te,D}^T z_{tr,D} z_{tr,D}^T)\phi_{D} 
\end{bmatrix}
+
\begin{bmatrix}
w_{te,1}^T - \alpha z_{te,1}^Tz_{tr,1}w_{tr,1}^T\\
w_{te,2}^T - \alpha z_{te,2}^Tz_{tr,2}w_{tr,2}^T\\
\vdots\\
w_{te,D}^T - \alpha z_{te,D}^Tz_{tr,D}w_{tr,D}^T
\end{bmatrix}
\right)\\
&\triangleq (Z^T)^{\dagger}(\Pi+W).
\end{aligned}
\end{equation}

\section{Proof of Lemma \ref{lem:lem.2}}\label{sec:prooflem2}
\lemmaa*
Note that 
\begin{align}\label{lbzz}
    ZZ^T&=\sum_{d=1}^D ( I- \alpha z_{tr,d}z_{tr,d}^T)z_{te,d}z^T_{te,d}(I- \alpha z_{tr,d}z_{tr,d}^T).
\end{align}
To obtain the lower bound of $ZZ^T$, we first find a high-probability lower bound on the term $I- 2\alpha z_{tr,d}z_{tr,d}^T$ such that the right hand side of \eqref{lbzz} can be bounded from below by the sum of testing data correlation over all $D$ blocks.
\begin{mylemma}\label{lem:lem.a.1}
For any block $d$ with model parameter $\phi_d\in\Theta$ and $\delta\in (0,1)$, the following inequality holds:
\begin{equation*}
    \mathbb{P}\left[z_{tr,d}z_{tr,d}^T \npreceq \frac{2(m+n)}{\delta}M\Gamma_{M-1}\right]\leq \frac{\delta}{2}.
\end{equation*}
\end{mylemma}
\begin{proof}
Recall that for $d$-th block, 
\begin{equation*}
    z_{t,d}z_{t,d}^T=\begin{bmatrix} x_{t,d} \\ u_{t,d} \end{bmatrix}
    \begin{bmatrix}
    x_{t,d}^T & u_{t,d}^T
    \end{bmatrix} 
    =\begin{bmatrix}
    x_{t,d}x_{t,d}^T & x_{t,d}u_{t,d}^T\\
    u_{t,d}x_{t,d}^T & u_{t,d}u_{t,d}^T
    \end{bmatrix}.
\end{equation*}
Given $x_{t,d}=\phi_d^T z_{t-1,d} + w_{t-1,d}$, we have $\mathbb{E}[x_{t,d}u_{t,d}^T]=\mathbb{E}[u_{t,d}x_{t,d}^T]=0$ since $x_{t,d}$ and $u_{t,d}$ are independent, where the expectation is taken with respect to the control action $u_{t,d}\sim\mathcal{N}(0,\sigma_a^2I_m)$ and the noise $w_t\sim \mathcal{N}(0, \sigma_w^2I_n)$. Moreover, 
\begin{equation*}
    x_{t+1,d}=A_d x_{t,d}+B_d u_{t,d}+w_{t,d}=A_d x_{t,d}+w'_{t,d},
\end{equation*}
where $w'_{t,d}\triangleq B_d u_{t,d}+w_{t,d}$ and $w'_{t,d}\sim \mathcal{N}(0, \sigma_a^2 B_d B_d^T+\sigma_w^2 I_n)$. It can be shown that
\begin{align*}
    \mathbb{E}[x_{t,d}x_{t,d}^T]&=\mathbb{E}\left[\left(\sum_{i=0}^{t-1} A_d^i w'_{t-1-i,d}\right)\left(\sum_{i=0}^{t-1} (w'_{t-1-i,d})^T(A_d^i)^T\right)\right]\\
    &=\sum_{i=0}^{t-1} A_d^i(\sigma_a^2 B_d B_d^T+\sigma_w^2I_n)(A_d^i)^T\\
    &=\sigma_a^2 G_{t,d}+\sigma_w^2 F_{t,d}
\end{align*}
and
\begin{equation*}
    \mathbb{E}[u_{t,d}u_{t,d}^T]=\sigma_a^2 I_m.
\end{equation*}
From the definition of $\Gamma_{t,d}$ and $\Gamma_{M-1}$, it follows that
\begin{align}\label{ubzzmean}
    \mathbb{E}[z_{tr,d}z_{tr,d}^T]=\mathbb{E}\left[\sum_{t=0}^{M-1} z_{t,d}z_{t,d}^T\right]
    &=\sum_{t=0}^{M-1} \mathbb{E}[z_{t,d}z_{t,d}^T]\nonumber\\
    &=\begin{bmatrix}
    \sum_{t=0}^{M-1} (\sigma_a^2 G_{t,d}+\sigma_w^2 F_{t,d}) & 0\\
    0 & M\sigma_a^2 I_m
    \end{bmatrix}\nonumber\\
    &\preceq \begin{bmatrix}
    M(\sigma_a^2 G_{M-1,d}+\sigma_w^2 F_{M-1,d}) & 0\\
    0 & M\sigma_a^2 I_m
    \end{bmatrix}\nonumber\\
    &\preceq M\Gamma_{M-1}.
\end{align}
Appealing to Markov's Inequality, we conclude that
\begin{align*}
    &\mathbb{P}\left[z_{tr,d}z_{tr,d}^T \npreceq \frac{2(m+n)}{\delta}M\Gamma_{M-1}\right]\\
    =&\mathbb{P}\left[\lambda_{max}((M\Gamma_{M-1})^{-1/2}z_{tr,d}z_{tr,d}^T(M\Gamma_{M-1})^{-1/2})\geq \frac{2(m+n)}{\delta}\right]\\
    \leq& \frac{\delta}{2(m+n)}\mathbb{E}\left[\lambda_{max}((M\Gamma_{M-1})^{-1/2}z_{tr,d}z_{tr,d}^T(M\Gamma_{M-1})^{-1/2})\right]\\
    \leq& \frac{\delta}{2(m+n)}\mathbb{E}\left[Tr((M\Gamma_{M-1})^{-1/2}z_{tr,d}z_{tr,d}^T(M\Gamma_{M-1})^{-1/2})\right]\\
    \leq& \frac{\delta}{2}
\end{align*}
where the last inequality  holds because of \eqref{ubzzmean}.
\end{proof}

\subsection{Upper bound on $\lambda_{max}(z_{tr,d}z_{tr,d}^T)$}
In what follows, we aim to obtain a tighter upper bound on $\lambda_{max}(z_{tr,d}z_{tr,d}^T)$. 

\begin{mylemma}\label{lem:boundlambda}
For any $\delta\in (0,1)$, the following inequality holds:
\begin{align*}
    \mathbb{P}\Big[\lambda_{max}(z_{tr,d}z_{tr,d}^T)\leq M^3\|\Bar{B}\|^2\Big(1+3\sqrt{\log\frac{10M}{\delta}}\Big)^2\max\{m\sigma_a^2,n\sigma_w^2\}\Big]
    \geq 1-\frac{\delta}{2}.
\end{align*}
\end{mylemma}

\begin{proof}
First, it clear that
\begin{align*}
    \lambda_{max}(z_{tr,d}z_{tr,d}^T)=&\lambda_{max}\left(\sum_{t=0}^{M-1} z_{t,d}z_{t,d}^T\right)\\
    \leq& Tr\left(\sum_{t=0}^{M-1} z_{t,d}z_{t,d}^T\right)\\
    = & Tr\left(\sum_{t=0}^{M-1} x_{t,d}x_{t,d}^T\right)+Tr\left(\sum_{t=0}^{M-1} u_{t,d}u_{t,d}^T\right).
\end{align*}
We next seek upper bounds on $Tr\left(\sum_{t=0}^{M-1} x_{t,d}x_{t,d}^T\right)$ and $Tr\left(\sum_{t=0}^{M-1} u_{t,d}u_{t,d}^T\right)$, respectively.

(1) For the term $Tr\left(\sum_{t=0}^{M-1} x_{t,d}x_{t,d}^T\right)$, through sophisticated manipulation, we can have the following with regard to $x_{t,d}x_{t,d}^T$ given $x_{0,d}=0$:
\begin{align*}
    x_{t,d}x_{t,d}^T=&\sum_{i=0}^{t-1} A_d^{t-1-i}B_d u_{i,d}u_{i,d}^T B_d^T (A_d^T)^{t-1-i}+\sum_{i=0}^{t-1} A_d^{t-1-i} w_{i,d}w_{i,d}^T(A_d^T)^{t-1-i}\\
    &+\sum_{i=0}^{t-1}\sum_{j=0}^{t-1} A_d^{t-1-i} B_d u_{i,d} w_{j,d}^T(A_d^T)^{t-1-j}+\left[\sum_{i=0}^{t-1}\sum_{j=0}^{t-1} A_d^{t-1-i} B_d u_{i,d} w_{j,d}^T(A_d^T)^{t-1-j}\right]^T\\
    &+\sum_{i=0}^{t-1}\sum_{j=0,j\neq i}^{t-1} A_d^{t-1-i}B_d u_{i,d}u_{j,d}^T B_d^T (A_d^T)^{t-1-j}+\sum_{i=0}^{t-1}\sum_{j=0,j\neq i}^{t-1} A_d^{t-1-i}w_{i,d}w_{j,d}^T (A_d^T)^{t-1-j}.
\end{align*}
Note that
\begin{align*}
    &A_d^{t-1-i}B_d u_{i,d}u_{i,d}^TB_d^T(A_d^T)^{t-1-i}+A_d^{t-1-j}w_{j,d}w_{j,d}^T(A_d^T)^{t-1-j}-A_d^{t-1-i} B_d u_{i,d} w_{j,d}^T(A_d^T)^{t-1-j}\\
    &-[A_d^{t-1-i} B_d u_{i,d} w_{j,d}^T(A_d^T)^{t-1-j}]^T\\
    =& A_d^{t-1-i}B_d u_{i,d}[u_{i,d}^T B_d^T (A_d^T)^{t-1-j}-w_{j,d}^T(A_d^T)^{t-1-j}]+A_d^{t-1-j}w_{j,d}[w_{j,d}^T(A_d^T)^{t-1-j}\\
    &-u_{i,d}^TB_d^T(A_d^T)^{t-1-i}]\\
    =&[A_d^{t-1-i}B_d u_{i,d}-A_d^{t-1-j}w_{j,d}][A_d^{t-1-i}B_d u_{i,d}-A_d^{t-1-j}w_{j,d}]^T\\
    \succeq& 0.
\end{align*}
It then follows that
\begin{align*}
    x_{t,d}x_{t,d}^T\preceq&\sum_{i=0}^{t-1} A_d^{t-1-i}B_d u_{i,d}u_{i,d}^T B_d^T (A_d^T)^{t-1-i}+\sum_{i=0}^{t-1} A_d^{t-1-i} w_{i,d}w_{i,d}^T(A_d^T)^{t-1-i}\\
    &+\sum_{i=0}^{t-1}\sum_{j=0}^{t-1}[A_d^{t-1-i}B_d u_{i,d}u_{i,d}^TB_d^T(A_d^T)^{t-1-i}+A_d^{t-1-j}w_{j,d}w_{j,d}^T(A_d^T)^{t-1-j}]\\
    &+\sum_{i=0}^{t-1}\sum_{j=0,j\neq i}^{t-1} A_d^{t-1-i}B_d u_{i,d}u_{j,d}^T B_d^T (A_d^T)^{t-1-j}+\sum_{i=0}^{t-1}\sum_{j=0,j\neq i}^{t-1} A_d^{t-1-i}w_{i,d}w_{j,d}^T (A_d^T)^{t-1-j}\\
    =&(t+1)\sum_{i=0}^{t-1} A_d^{t-1-i}B_d u_{i,d}u_{i,d}^T B_d^T (A_d^T)^{t-1-i}+(t+1)\sum_{i=0}^{t-1} A_d^{t-1-i} w_{i,d}w_{i,d}^T(A_d^T)^{t-1-i}\\
     &+\sum_{i=0}^{t-1}\sum_{j=0,j\neq i}^{t-1} A_d^{t-1-i}B_d u_{i,d}u_{j,d}^T B_d^T (A_d^T)^{t-1-j}+\sum_{i=0}^{t-1}\sum_{j=0,j\neq i}^{t-1} A_d^{t-1-i}w_{i,d}w_{j,d}^T (A_d^T)^{t-1-j}
\end{align*}
which indicates that
\begin{align*}
    &\sum_{t=1}^{M-1}x_{t,d}x_{t,d}^T\\
    \preceq& M\sum_{j=1}^M\sum_{i=0}^{M-j} A_d^{j-1}B_d u_{i,d}u_{i,d}^T B_d^T (A_d^T)^{j-1}+M\sum_{j=1}^M\sum_{i=0}^{M-j} A_d^{j-1} w_{i,d}w_{i,d}^T (A_d^T)^{j-1}\\
    &+\sum_{t=1}^M\sum_{i=0}^{t-1}\sum_{j=0,j\neq i}^{t-1} A_d^{t-1-i}B_d u_{i,d}u_{j,d}^T B_d^T (A_d^T)^{t-1-j}+\sum_{t=1}^M\sum_{i=0}^{t-1}\sum_{j=0,j\neq i}^{t-1} A_d^{t-1-i}w_{i,d}w_{j,d}^T (A_d^T)^{t-1-j}.
\end{align*}
Therefore, the trace $Tr\left(\sum_{t=0}^{M-1} x_{t,d}x_{t,d}^T\right)$ can be bounded from above as follows:
\begin{align*}
    &Tr\left(\sum_{t=0}^{M-1} x_{t,d}x_{t,d}^T\right)\\
    \leq& M\sum_{j=1}^M Tr\left[\sum_{i=0}^{M-j} A_d^{j-1}B_d u_{i,d}u_{i,d}^T B_d^T (A_d^T)^{j-1}\right]+M\sum_{j=1}^M Tr\left[\sum_{i=0}^{M-j} A_d^{j-1} w_{i,d}w_{i,d}^T (A_d^T)^{j-1}\right]\\
    &+\sum_{t=1}^M\sum_{i=0}^{t-1}\sum_{j=0,j\neq i}^{t-1} Tr[A_d^{t-1-i}B_d u_{i,d}u_{j,d}^T B_d^T (A_d^T)^{t-1-j}]\\
    &+\sum_{t=1}^M\sum_{i=0}^{t-1}\sum_{j=0,j\neq i}^{t-1} Tr[A_d^{t-1-i}w_{i,d}w_{j,d}^T (A_d^T)^{t-1-j}].
\end{align*}

Next, we first find an upper bound on $Tr\left[\sum_{i=0}^{M-j} A_d^{j-1}B_d u_{i,d}u_{i,d}^T B_d^T (A_d^T)^{j-1}\right]$. It can be shown that
\begin{align*}
    Tr\left[\sum_{i=0}^{M-j} A_d^{j-1}B_d u_{i,d}u_{i,d}^T B_d^T (A_d^T)^{j-1}\right]&=Tr\left[A_d^{j-1}B_d (\sum_{i=0}^{M-j} u_{i,d}u_{i,d}^T) B_d^T (A_d^T)^{j-1}\right]\\
    &=Tr\left[B_d^T (A_d^T)^{j-1} A_d^{j-1}B_d (\sum_{i=0}^{M-j} u_{i,d}u_{i,d}^T)\right]\\
    &\leq \lambda_{max}(B_d^T (A_d^T)^{j-1} A_d^{j-1}B_d) Tr\left(\sum_{i=0}^{M-j} u_{i,d}u_{i,d}^T\right)\\
    &\leq \|A_d^{j-1}B_d\|^2 Tr\left(\sum_{i=0}^{M-j} u_{i,d}u_{i,d}^T\right)
\end{align*}
where the first inequality holds because the following is true based on Von Neumann's trace inequality:
\begin{align}\label{traceproduct}
    \lambda_{min}(X) Tr(Y)\leq Tr(XT)\leq \lambda_{max}(X) Tr(Y)
\end{align}
for positive semi-definite matrices $X$ and $Y$. Therefore, it suffices to bound $Tr\left(\sum_{i=0}^{M-j} u_{i,d}u_{i,d}^T\right)$ from above.

Let $\mathcal{U}_j=\sum_{i=0}^{M-j} u_{i,d}u_{i,d}^T$. It is clear that $\mathcal{U}_j$ follows a pseudo Wishart distribution $\mathcal{SW}_m(M-j,\sigma_a^2 I_m)$ considering that the dimension $m$ is generally larger than the training size $M$. Based on \cite{horn2012matrix}, there exists a matrix $Q\in\mathbb{R}^{m\times r_j}$ such that $\mathcal{U}_j=Q\Tilde{\mathcal{U}}_j Q^T$, where $r_j$ is the rank of $\mathcal{U}_j$ ($r_j\leq m$) and $\Tilde{\mathcal{U}}_j$ follows a Wishart distribution, i.e., $\Tilde{\mathcal{U}}_j\sim \mathcal{W}_{r_j}(M-J, I_{r_j})$. Hence,
\begin{align*}
    Tr\left(\sum_{i=0}^{M-j} u_{i,d}u_{i,d}^T\right)=&Tr(Q\Tilde{\mathcal{U}}_j Q^T)\\
    =&Tr(Q^TQ\Tilde{\mathcal{U}}_j)\\
    \leq& \lambda_{max}(Q^TQ) Tr(\Tilde{\mathcal{U}}_j)\\
    =&\sigma_a^2 Tr(\Tilde{\mathcal{U}}_j),
\end{align*}
such that
\begin{align}\label{bounduu}
    Tr\left[\sum_{i=0}^{M-j} A_d^{j-1}B_d u_{i,d}u_{i,d}^T B_d^T (A_d^T)^{j-1}\right]\leq \sigma_a^2 \|A_d^{j-1}B_d\|^2 Tr(\Tilde{\mathcal{U}}_j).
\end{align}
Following the same line, it can be shown that
\begin{align}\label{boundww}
    Tr\left[\sum_{i=0}^{M-j} A_d^{j-1} w_{i,d}w_{i,d}^T (A_d^T)^{j-1}\right]\leq \sigma_w^2 \|A_d^{j-1}\|^2 Tr(\Tilde{\mathcal{W}}_j)
\end{align}
where $\Tilde{\mathcal{W}}_j\sim\mathcal{W}_{q_j}(M-J, I_{q_j})$ for some $q_j\leq n$.

For the term $Tr[A_d^{t-1-i}B_d u_{i,d}u_{j,d}^T B_d^T (A_d^T)^{t-1-j}]$, we have
\begin{align}\label{bounduij}
    &Tr[A_d^{t-1-i}B_d u_{i,d}u_{j,d}^T B_d^T (A_d^T)^{t-1-j}]\nonumber\\
    \leq&m\|A_d^{t-1-i}B_d u_{i,d}u_{j,d}^T B_d^T (A_d^T)^{t-1-j}\|\nonumber\\
    \leq& m\|A_d^{t-1-i}B_d\|\|A_d^{t-1-j}B_d\|\|u_{i,d}u_{j,d}^T\| \nonumber\\
    \leq& m\|A_d^{t-1-i}B_d\|\|A_d^{t-1-j}B_d\| Tr(u_{i,d}u_{j,d}^T).
\end{align}
And similarly,
\begin{align}\label{boundwij}
    Tr[A_d^{t-1-i}w_{i,d}w_{j,d}^T (A_d^T)^{t-1-j}]\leq n\|A_d^{t-1-i}A_d^{t-1-j}\| Tr(w_{i,d}w_{j,d}^T).
\end{align}

Combing \eqref{bounduu} - \eqref{boundwij}, we can obtain that
\begin{align*}
    Tr\left(\sum_{t=0}^{M-1} x_{t,d}x_{t,d}^T\right)
    \leq& M\sigma_a^2\sum_{j=1}^M [\|A_d^{j-1}B_d\|^2 Tr(\Tilde{\mathcal{U}}_j)]+M\sigma_w^2\sum_{j=1}^M [\|A_d^{j-1}\|^2 Tr(\Tilde{\mathcal{W}}_j)]\\
    &+m\sum_{t=1}^M\sum_{i=0}^{t-1}\sum_{j=0,j\neq i}^{t-1}[\|A_d^{t-1-i}B_d\|\|A_d^{t-1-j}B_d\| Tr(u_{i,d}u_{j,d}^T)]\\
    &+n\sum_{t=1}^M\sum_{i=0}^{t-1}\sum_{j=0,j\neq i}^{t-1}[\|A_d^{t-1-i}A_d^{t-1-j}\| Tr(w_{i,d}w_{j,d}^T)]\\
    \leq& M\sigma_a^2\|\Bar{B}\|^2\sum_{j=1}^M Tr(\Tilde{\mathcal{U}}_j)+M\sigma_w^2\sum_{j=1}^M Tr(\Tilde{\mathcal{W}}_j)\\
    &+m\|\Bar{B}\|^2\sum_{t=1}^M\sum_{i=0}^{t-1}\sum_{j=0,j\neq i}^{t-1}Tr(u_{i,d}u_{j,d}^T)\\
    &+n\sum_{t=1}^M\sum_{i=0}^{t-1}\sum_{j=0,j\neq i}^{t-1}Tr(w_{i,d}w_{j,d}^T).
\end{align*}

Moreover, since $\Tilde{\mathcal{U}}_j\sim \mathcal{W}_{r_j}(M-J, I_{r_j})$, $Tr(\Tilde{\mathcal{U}}_j)\sim \chi^2_{r_j(M-J)}$, for which we have the following high probability bound:
\begin{align}\label{boundchisquareu}
   & \mathbb{P}(Tr(\Tilde{\mathcal{U}}_j)\leq Mm+2\sqrt{Mma}+2a)\nonumber\\
    \geq & \mathbb{P}(Tr(\Tilde{\mathcal{U}}_j)\leq r_j(M-J)+2\sqrt{ar_j(M-J)}+2a)\nonumber\\
    \geq& 1-e^{-a}
\end{align}
for any $a>0$. Similarly, we have
\begin{align}\label{boundchisquarew}
    \mathbb{P}(Tr(\Tilde{\mathcal{W}}_j)\leq Mn+2\sqrt{Mna}+2a)\geq 1-e^{-a}.
\end{align}
For the term $Tr(u_{i,d}u_{j,d}^T)$, based on the Chernoff bound, we can have
\begin{align*}
    &\mathbb{P}[Tr(u_{i,d}u_{j,d}^T)\leq b]\\
    \geq& 1-\inf_{c>0} e^{-cb}\mathbb{E}[e^{b Tr(u_{i,d}u_{j,d}^T)}]\\
    =& 1-\inf_{c>0} e^{-cb}\mathbb{E}[e^{b\sum_{k=1}^m u_{i,d}(k)u_{j,d}(k)}]\\
    =& 1-\inf_{c>0} e^{-cb}\prod_{k=1}^m \mathbb{E}[e^{t u_{i,d}(k)u_{j,d}(k)}]\\
    =& 1-\inf_{0<c<1/\sigma_a^2} e^{-cb}\left(\frac{1}{\sqrt{\pi(1-\sigma_a^4 c^2)}}\right)^m\\
    \geq& 1-e^{-\frac{b}{\sqrt{2}\sigma_a^2}}\left(\frac{\pi}{2}\right)^{-m/2},
\end{align*}
which implies that
\begin{align}\label{boundchernoffu}
    \mathbb{P}\left[Tr(u_{i,d}u_{j,d}^T)\leq \sqrt{2}\sigma_a^2 \left(a'-\frac{m}{2}\log\frac{\pi}{2}\right)\right]\geq 1-e^{-a'}
\end{align}
for $a'>0$.
Following the same line, we can obtain the following for the term $Tr(w_{i,d}w_{j,d}^T)$:
\begin{align}\label{boundchernoffw}
    \mathbb{P}\left[Tr(w_{i,d}w_{j,d}^T)\leq \sqrt{2}\sigma_w^2 \left(a'-\frac{n}{2}\log\frac{\pi}{2}\right)\right]\geq 1-e^{-a'}.
\end{align}

Based on \eqref{boundchisquareu} - \eqref{boundchernoffw}, it follows that
\begin{align*}
    &\mathbb{P}\Big[Tr\left(\sum_{t=0}^{M-1} x_{t,d}x_{t,d}^T\right)\leq M^2\sigma_a^2\|\Bar{B}\|^2(Mm+2\sqrt{Mm\log\frac{10M}{\delta}}+2\log\frac{10M}{\delta})\\
    &+M^2\sigma_w^2(Mn+2\sqrt{Mn\log\frac{10M}{\delta}}+2\log\frac{10M}{\delta})\\
    &+\sqrt{2}\sigma_a^2mM^3\|\Bar{B}\|^2\log\frac{10M^3}{\delta(\pi/2)^{m/2}}+\sqrt{2}\sigma_w^2nM^3\log\frac{10M^3}{\delta(\pi/2)^{n/2}}\Big]\\
    \geq& 1-\frac{2\delta}{5}.
\end{align*}

(2) Next, for the term $Tr\left(\sum_{t=0}^{M-1} u_{t,d}u_{t,d}^T\right)$, it can be seen that
\begin{align*}
    Tr\left(\sum_{t=0}^{M-1} u_{t,d}u_{t,d}^T\right)=&\sum_{t=0}^{M-1} Tr(u_{t,d}u_{t,d}^T)\\
    =&\sum_{t=0}^{M-1} \sum_{i=1}^m u_{t,d}^2(i)\\
    =&\sigma_a^2\sum_{t=0}^{M-1} \sum_{i=1}^m \Tilde{u}_{t,d}^2(i)\\
    =&\sigma_a^2\mathcal{U}
\end{align*}
where $\Tilde{u}_{t,d}^2\sim\mathcal{N}(0, I_m)$ and $\mathcal{U}\sim\chi^2_{Mm}$. Therefore,
\begin{align*}
    \mathbb{P}\left[Tr\left(\sum_{t=0}^{M-1} u_{t,d}u_{t,d}^T\right)\leq \sigma_a^2(Mm+2\sqrt{Mm\log\frac{10}{\delta}}+2\log\frac{10}{\delta})\right]\geq 1-\frac{\delta}{10}.
\end{align*}

In a nutshell, we can obtain that with probability $1-\delta/2$, the following holds:
\begin{align}\label{lambdamax_upperbound}
    &\lambda_{max}(z_{tr,d}z_{tr,d}^T)\nonumber\\
    \leq& M^2\sigma_a^2\|\Bar{B}\|^2(Mm+2\sqrt{Mm\log\frac{10M}{\delta}}+2\log\frac{10M}{\delta})\nonumber\\
    &+M^2\sigma_w^2(Mn+2\sqrt{Mn\log\frac{10M}{\delta}}+2\log\frac{10M}{\delta})\nonumber\\
    &+\sqrt{2}\sigma_a^2mM^3\|\Bar{B}\|^2\log\frac{10M^3}{\delta(\pi/2)^{m/2}}+\sqrt{2}\sigma_w^2nM^3\log\frac{10M^3}{\delta(\pi/2)^{n/2}}\nonumber\\
    &+\sigma_a^2\left(Mm+2\sqrt{Mm\log\frac{10}{\delta}}+2\log\frac{10}{\delta}\right)\nonumber\\
    \leq& M^2\sigma_a^2\|\Bar{B}\|^2(Mm+2\sqrt{Mm\log\frac{10M}{\delta}}+2\log\frac{10M}{\delta}+\sqrt{2}Mm\log\frac{10M^3}{\delta(\pi/2)^{m/2}})\nonumber\\
    &+M^2\sigma_w^2(Mn+2\sqrt{Mn\log\frac{10M}{\delta}}+2\log\frac{10M}{\delta}+\sqrt{2}Mn\log\frac{10M^3}{\delta(\pi/2)^{n/2}})\nonumber\\
    &+\sigma_a^2\left(Mm+2\sqrt{Mm\log\frac{10}{\delta}}+2\log\frac{10}{\delta}\right)\nonumber\\
    \leq& M^3m\sigma_a^2\|\Bar{B}\|^2(1+3\sqrt{\log\frac{10M}{\delta}})^2+M^3n(1+3\sqrt{\log\frac{10M}{\delta}})^2\nonumber\\
    \leq& M^3\|\Bar{B}\|^2\Bigg(1+3\sqrt{\log\frac{10M}{\delta}}\Bigg)^2\max\{m\sigma_a^2,n\sigma_w^2\}.
\end{align}
\end{proof}

\subsection{Lower bound on $\lambda_{min}(ZZ^T)$}

Note that
\[
{\Bar{\lambda}} \triangleq 
M^3\|\Bar{B}\|^2\Big(1+3\sqrt{\log\frac{10DM}{\delta}}\Big)^2\max\{m\sigma_a^2,n\sigma_w^2\}.
\]

Based on Lemma \ref{lem:boundlambda}, and by setting the learning rate
    $0<\alpha< 1/\Bar{\lambda}$,
we can have $\min_d\{\lambda_{min}(I-\alpha z_{tr,d}z_{tr,d}^t)\}\geq 1-\alpha \Bar{\lambda}>0$, i.e., $I-\alpha z_{tr,d}z_{tr,d}^T\succeq (1-\alpha\Bar{\lambda})I\succ 0$, with probability $1-\delta/2$. Consequently, we have the following result about the lower bound on $\lambda_{min}(ZZ^T)$:

\begin{mylemma}\label{lem:lowerboundzz}
With probability $1-\delta/2$ for any $\delta\in (0,1)$, the following inequality holds:
\begin{align*}
   \lambda_{min}(ZZ^T)\geq \frac{(1-\alpha\Bar{\lambda})^2}{2}\lambda_{min}\left(\sum_{d=1}^D z_{te,d}z^T_{te,d}\right).
\end{align*}
\end{mylemma}

\begin{proof}

To prove Lemma \ref{lem:lowerboundzz}, let $\bE=I-\alpha z_{tr,d}z_{tr,d}^t\succ 0$, and we first show that 
\begin{align}\label{eq:pd}
    \frac{1}{(1-\alpha\Bar{\lambda})^2}\bE\bG\bE\succeq \bG
\end{align}
for $\bG\in\mathbb{R}^{(m+n)\times (m+n)}$ and $\bG\succ 0$.

For ease of exposition, let $\bar{\bE}=\frac{1}{1-\alpha\Bar{\lambda}}\bE$ and $\bC=\Bar{\bE}\bG\bar{\bE}$. Then $\bC\succ 0$.  Based on Corollary 7.6.5 in \cite{horn2012matrix}, there exists a nonsingular matrix $\bS\in\mathbb{R}^{(m+n)\times (m+n)}$, such that
\begin{align*}
    \bC=\bS I\bS^T ~~\text{and}~~ \bG=\bS\Sigma\bS^T
\end{align*}
where $\Sigma=\diag(\bd_1,...,\bd_{m+n})$ is diagonal. To show $\bC\succeq\bG$, it suffices to show $\bS(I-\Sigma)\bS^T\succeq 0$, which is the case if and only if all $\bd_i\leq 1$ for $i\in[1, m+n]$.

To show $\bd_i\leq 1$ for $i\in[1, m+n]$, it is worth to note that
\begin{align*}
    \bG\bC^{-1}=\bS\sigma\bS^T(\bS^T)^{-1}\bS^{-1}=\bS\Sigma\bS^{-1}.
\end{align*}
Hence, $\bG\bC^{-1}$ is similar to $\Sigma$, and they have the same eigenvalues $\bd_1,...,\bd_{m+n}$. Moreover, according to Lemma 5.6.10 in \cite{horn2012matrix}, there exists a matrix norm $|||\cdot|||$ such that the following holds for the spectral $\rho(\bG\bC^{-1})$:
\begin{align*}
    \rho(\bG\bC^{-1})&=\rho(\bG(\Bar{\bE})^{-1}\bG^{-1}(\Bar{\bE})^{-1})\\
    &\leq |||\bG(\Bar{\bE})^{-1}\bG^{-1}(\Bar{\bE})^{-1}|||\\
    &\leq |||\bG(\Bar{\bE})^{-1}\bG^{-1}|||\cdot |||(\Bar{\bE})^{-1}|||\\
    &\leq [\rho(\bG(\Bar{\bE})^{-1}\bG^{-1})+\epsilon_1][\rho(\Bar{\bE}^{-1})+\epsilon_1]\\
    &=[\rho(\Bar{\bE}^{-1})+\epsilon_1]^2\\
    &\leq \left[\frac{1-\alpha\Bar{\lambda}}{\lambda_{min}(\bE)}+\epsilon_1\right]^2\\
    &\leq 1
\end{align*}
for some $0<\epsilon_1\leq 1-\frac{1-\alpha\Bar{\lambda}}{\lambda_{min}(\bE)}$. Therefore, \eqref{eq:pd} holds.

Next, let $\bG=z_{te,d}z_{te,d}^T+\epsilon_2 I$ for $\epsilon_2>0$. Hence, $\bG\succ 0$. It follows that
\begin{align*}
    \bE(z_{te,d}z_{te,d}^T+\epsilon_2 I)\bE\succeq (1-\alpha\Bar{\lambda})^2 (z_{te,d}z_{te,d}^T+\epsilon_2 I)
\end{align*}
such that
\begin{align*}
    \sum_{d=1}^D \bE(z_{te,d}z_{te,d}^T+\epsilon_2 I)\bE\succeq (1-\alpha\Bar{\lambda})^2 \sum_{d=1}^D (z_{te,d}z_{te,d}^T+\epsilon_2 I).
\end{align*}
It can then be seen that
\begin{align*}
    \lambda_{min}\left( \sum_{d=1}^D \bE(z_{te,d}z_{te,d}^T+\epsilon_2 I)\bE\right)\geq(1-\alpha\Bar{\lambda})^2\lambda_{min}\left(\sum_{d=1}^D z_{te,d}z_{te,d}^T\right)+\epsilon_2D(1-\alpha\Bar{\lambda})^2
\end{align*}
and
\begin{align*}
    \lambda_{min}\left( \sum_{d=1}^D \bE(z_{te,d}z_{te,d}^T+\epsilon_2 I)\bE\right)\leq \lambda_{min}\left(\sum_{d=1}^D \bE z_{te,d}z_{te,d}^T\bE\right)+\epsilon_2\lambda_{max}\left(\sum_{d=1}^D \bE\bE\right).
\end{align*}
Therefore, we can have
\begin{align*}
    \lambda_{min}(ZZ^T)=&\lambda_{min}\left(\sum_{d=1}^D \bE z_{te,d}z_{te,d}^T\bE\right)\\
    \geq& (1-\alpha\Bar{\lambda})^2\lambda_{min}\left(\sum_{d=1}^D z_{te,d}z_{te,d}^T\right)+\epsilon_2D(1-\alpha\Bar{\lambda})^2-\epsilon_2\lambda_{max}\left(\sum_{d=1}^D \bE\bE\right).
\end{align*}
Let $\epsilon_2=\frac{(1-\alpha\Bar{\lambda})^2\lambda_{min}\left(\sum_{d=1}^D z_{te,d}z_{te,d}^T\right)}{2\left[\lambda_{max}\left(\sum_{d=1}^D \bE\bE\right)-D(1-\alpha\Bar{\lambda})^2\right]}$. We can obtain that
\begin{align*}
     \lambda_{min}(ZZ^T)\geq \frac{(1-\alpha\Bar{\lambda})^2}{2}\lambda_{min}\left(\sum_{d=1}^D z_{te,d}z^T_{te,d}\right).
\end{align*}

\end{proof}


\subsection{Lower bound on $\lambda_{min}\left(\sum_{d=1}^D z_{te,d}z^T_{te,d}\right)$}

To obtain a lower bound on $\sqrt{\lambda_{min}(ZZ^T)}$, it suffices to find a lower bound on $\lambda_{min}\left(\sum_{d=1}^D z_{te,d}z^T_{te,d}\right)$. Note that
\begin{equation*}
    \sum_{d=1}^D z_{te,d}z^T_{te,d}=\sum_{d=1}^D \sum_{t=M}^{L-1} z_{t,d}z^T_{t,d},
\end{equation*}
which indicates that
\begin{equation}\label{minproduct}
    \lambda_{min}\left(\sum_{d=1}^D z_{te,d}z^T_{te,d}\right)=\min_{u\in\mathcal{S}^{m+n-1}}\sum_{d=1}^D\sum_{t=M}^{L-1} \langle z_{t,d}, u\rangle^2,
\end{equation}
where $\mathcal{S}^{m+n-1}$ denotes the unit sphere in $\mathbb{R}^{m+n}$. Hence, it reduces to find a lower bound of the right hand side of the equation \eqref{minproduct}. 

Given any realization sequence of $D$ blocks $\{\phi_d\}$, if we put the testing sequence  for all $D$ blocks together as a `super-sequence', this `super-sequence' can be then regarded as a combination of $D$ independent martingale processes  where within each block the sequence of the system states is a martingale process with respect to the filtration  $\mathcal{F}_t:=\sigma(z_{M,d},...,z_{t,d}, w_{M,d},..., w_{t,d})$. Therefore, finding a lower bound of \eqref{minproduct} reduces to quantifying the covariates $z_{t,d}$ for the super-sequence.

To this end, we propose a two-scale small-ball method. More specifically, by dividing the testing sequence $z_{te,d}$ of each block $d$ as a set of mini-blocks with block size $k$, a martingale small-ball method is applied to evaluate the covariates within the entire testing sequence. Next, observing that in fact each block as a whole also satisfies the small-ball condition, we apply the Mendelson's small-ball method once again to evaluate the sample covariances of the entire super-sequence.

To this end, we first need to check if the sequence $\{z_{t,d}\}_{t\geq 0}$ satisfies the Martingale small-ball condition \cite{simchowitz2018learning}. 
\begin{mylemma}\label{lem:lem.a.2}
For any block $d$, let $x_{0,d}$ be any initial state in $\mathbb{R}^n$. Then, the process $z_{t,d}$ satisfies the $(k,\Gamma_{\lfloor k/2\rfloor,d},\frac{3}{20})$-BMSB condition for $k\in [1, \lfloor L-M\rfloor/2]$.
\end{mylemma}
\begin{proof}
Consider the $d$-th block $x_{t+1,d}=A_d x_{t,d}+B_d u_{t,d}+w_{t,d}$, where $x_{0,d}\in \mathbb{R}^n$, $a_{t,d}\sim\mathcal{N}(0, \sigma_a^2I_m)$ and $w_{t,d}\sim\mathcal{N}(0, \sigma_w^2I_m)$, we can have
\begin{equation*}
    x_{t,d}=A_d^{t}x_{0,d}+\sum_{i=0}^{t-1} A_d^i w_{t-1-i,d}'.
\end{equation*}
Since $w'_{t,d}\sim \mathcal{N}(0, \sigma_a^2 B_d B_d^T+\sigma_w^2 I_n)$, we have $x_{t_0+t,d}|\mathcal{F}_{t_0}\sim\mathcal{N}(A_d^{t}x_{t_0,d}, \sigma_a^2 G_{t,d}+\sigma_w^2 F_{t,d})$, such that
\begin{equation}\label{zdistribution}
    z_{t_0+t,d}=\begin{bmatrix}
    x_{t_0+t,d}\\
    u_{t_0+t,d}
    \end{bmatrix}
    \sim\mathcal{N}\left(\begin{bmatrix}
    A_d^{t}x_{t_0,d}\\
    0
    \end{bmatrix}, \begin{bmatrix}
    \sigma_a^2 G_{t,d}+\sigma_w^2 F_{t,d} & 0\\
    0 & \sigma_a^2I_m
    \end{bmatrix}\right)\triangleq\mathcal{N}(\mu_{t,d}, \Gamma_{t,d}).
\end{equation}
Hence, $\langle u, z_{t_0+t,d}\rangle \sim \mathcal{N}(\langle u, \mu_{t,d}\rangle, u^T \Gamma_{t,d}u)$ for $u\in\mathcal{S}^{m+n-1}$. It can be shown that for $k\in [1, \lfloor L-M\rfloor/2]$
\begin{equation*}
    \frac{1}{k}\sum_{t=1}^k \mathbb{P}\left(|\langle u, z_{t,d}\rangle|\geq\sqrt{u^T\Gamma_{\lfloor k/2\rfloor,d}u}\right)\geq \frac{3}{20},
\end{equation*}
indicating that for each block $d$ the process $z_{t,d}$ satisfies the $(k,\Gamma_{\lfloor k/2\rfloor,d}, \frac{3}{20})$-BMSB condition.
\end{proof}
Therefore, if we divide $z_{te,d}$ for each block $d$ as $\lfloor (L-M)/k\rfloor$ mini-blocks, each with size $k$, the process $z_{t,d}$ for each mini-block satisfies the $(k,\Gamma_{\lfloor k/2\rfloor,d}, \frac{3}{20})$-BMSB condition given any fixed $u\in\mathcal{S}^{m+n-1}$. 

For ease of exposition, we first take a look at the scalar case. More specifically, suppose we have $D$ independent blocks, each with size $\widehat{M}:=L-M$ and any initial state, and each block $\{z_{t,d}\}_{1\leq t\leq \widehat{M}}$ is a scalar process that satisfies the $(k,v_d,p)$-BMSB condition. Denote $S=\lfloor \widehat{M}/k\rfloor$. The following result in \cite{simchowitz2018learning} characterizes the property for such sequences:
\begin{mypro}\label{pro:pro.a.1}\cite{simchowitz2018learning}
Let $\{z_{t,d}\}_{t\geq 1}$ be a scalar process that satisfies the $(k,v_d,p)$-BMSB condition. Then,
\begin{equation*}
    \mathbb{P}\left[\sum_{t=1}^{\widehat{M}}z_{t,d}^2\geq \frac{v_d^2p^2kS}{8}\right]\geq 1-e^{-\frac{Sp^2}{8}}.
\end{equation*}
\end{mypro}

Let $X_d=\sum_{t=1}^{\widehat{M}}z_{t,d}^2$, $W_d=\frac{v_d^2p^2}{8}kS$ and $\kappa=1-e^{-\frac{S p^2}{8}}$. 
It is clear that the larger $v_d$ and $k$ are, the tighter the lower bound on $X_d$ is. This is also one of the main reasons that we devise a two-scale small-ball method to differentiate the correlation structure within each block and the conditional independence of system states across different blocks, given the sequence of $D$ blocks.
Then from Proposition \ref{pro:pro.a.1} we conclude that the super-sequence $\{X_d\}_{d=1}^D$ with independent elements satisfies the  small-ball condition \cite{mendelson2014learning}, i.e.,  $\mathbb{P}[X_d\geq W_d]\geq \kappa$. Hence, based on a binomial estimate, we can obtain a high-probability lower bound of the sum over the super-sequence.
\begin{mypro}\label{pro:pro.a.2}
Suppose all the block parameters lie in a compact set $\Theta$, and for each block $X_d$ satisfies the small-ball condition that $\mathbb{P}[X_d\geq W_d]\geq \kappa$ where $W_d=\frac{v_d^2p^2}{8}kS$. Denote $\Bar{v}=\max\{\|v_d\|\}$ and $\underline{v}=\min\{\|v_d\|\}$ as the maximum and the minimum of $\|v_d\|$ over $\Theta$, respectively. Then,
\begin{equation*}
    \mathbb{P}\left[\sum_{d=1}^D X_d\leq \frac{D\kappa Skp^2\underline{v}^2}{12}\right]\leq \exp\left\{-\frac{2D\kappa^2\underline{v}^4}{9\Bar{v}^4}\right\}.
\end{equation*}
\end{mypro}
\begin{proof}
Consider the following Bernoulli random variables
\begin{equation*}
      R_d=\mathbbm{1}\left(X_d\geq W_d\right), \forall d\in[1,D],
\end{equation*}
where $\mathbb{P}(R_d=1)\geq \kappa$. It is easy to check that $X_d\geq W_dR_d$, such that for some constant $C_1$ the following holds:
\begin{equation}\label{relaxwr}
    \mathbb{P}\left[\sum_{d=1}^D\sum_{t=1}^{\widehat{M}} z_{t,d}^2 \leq C_1\right]=\mathbb{P}\left[\sum_{d=1}^D X_d \leq C_1\right]\leq \mathbb{P}\left[\sum_{d=1}^D W_dR_d \leq C_1\right].
\end{equation}
Since $R_d$'s are independent random variables, we can find an upper bound on $\mathbb{P}\left[\sum_{d=1}^D W_dR_d \leq C_1\right]$ in \eqref{relaxwr} by studying the concentration behavior.

Moreover, observe that
\begin{equation}\label{wr}
    \mathbb{E}_{R_d}[W_dR_d]=W_d\mathbb{P}[R_d=1]\geq \kappa W_d.
\end{equation}
and $0\leq W_dR_d\leq \frac{\Bar{v}^2p^2}{8}kS$. 
Based on Hoeffding's Inequality, we can have that
\begin{align*}
    &\mathbb{P}\left[\sum_{d=1}^D W_dR_d-\kappa\sum_{d=1}^D W_d\leq -\epsilon\kappa\sum_{d=1}^D W_d\right]\\
    \leq& \mathbb{P}\left[\sum_{d=1}^D W_dR_d-\sum_{d=1}^D\mathbb{E}_{R_d}[W_dR_d]\leq -\epsilon\kappa\sum_{d=1}^D W_d\right]\\
    \leq& \exp\left\{-\frac{2\epsilon^2\kappa^2(\sum_{d=1}^D W_d)^2}{D\Bar{v}^4p^4k^2S^2/64}\right\}\\
    \leq& \exp\left\{-\frac{2\epsilon^2\kappa^2D\underline{v}^4}{\bar{v}^4}\right\}.
\end{align*}

Taking $\epsilon=\frac{1}{3}$, it follows that
\begin{align*}
    \mathbb{P}\left[\sum_{d=1}^D W_dR_d \leq \frac{p^2kS}{12} \kappa D\underline{v}^2\right]\leq
    \mathbb{P}\left[\sum_{d=1}^D W_dR_d \leq \frac{2}{3} \kappa \sum_{d=1}^D W_d\right]
    \leq& \exp\left\{-\frac{2D\kappa^2\underline{v}^4}{9\Bar{v}^4}\right\}
\end{align*}
Combining with \eqref{relaxwr}, Proposition \ref{pro:pro.a.2} can be proved.
\end{proof}

Proposition \ref{pro:pro.a.2} implies that for a fixed $u\in\mathcal{S}^{m+n-1}$, we can have a point-wise lower bound for $\sum_{d=1}^D\sum_{t=M}^{L-1} \langle z_{t,d}, u\rangle^2$:
\begin{equation}\label{point-wise}
    \mathbb{P}\left[\sum_{d=1}^D\sum_{t=M}^{L-1} \langle z_{t,d}, u\rangle^2\leq \frac{D(L-M)p^2\kappa}{12}u^T\underline{\Gamma}_{\lfloor k/2\rfloor}u\right]\leq \exp\left\{-\frac{2D\kappa^2}{9\Bar{v}^4}\lambda^2_{min}(\underline{\Gamma}_{\lfloor k/2\rfloor})\right\}.
\end{equation}
Based on \eqref{minproduct}, we need to obtain a lower bound of $\min_{u\in\mathcal{S}^{m+n-1}}\sum_{d=1}^D\sum_{t=M}^{L-1} \langle z_{t,d}, u\rangle^2$ from this point-wise lower bound. This step can be achieved by approximating this minimum with an $\epsilon$-net. More specifically, we use the following two lemmas in \cite{simchowitz2018learning} to show this. Denote by $\mathcal{S}_M$ the set of all points $x\in\mathbb{R}^{m+n}$ such that $\|M^{-1/2}x\|_2=1$ for a given $M\in\mathbb{R}^{(m+n)\times(m+n)}$ and $M\succ 0$.
\begin{mylemma}\label{lem:lem.a.3}
Let $Q\in \mathbb{R}^{(m+n)\times D(L-M)}$ and consider matrices $0\prec \Gamma_{min}\preceq \Gamma_{max}\in\mathbb{R}^{(m+n)\times (m+n)}$. Let $\mathcal{T}$ be a $1/4$-net of $\mathcal{S}_{\Gamma_{min}}$ in the metric $\|\Gamma_{max}^{1/2}(\cdot)\|_2$. Then if $\inf_{u\in \mathcal{T}} u^TQQ^Tu\geq 1$ and $QQ^T\preceq \Gamma_{max}$, we have
$QQ^T\succeq \Gamma_{min}/2$.
\end{mylemma}
\begin{mylemma}\label{lem:lem.a.4}
Let $0\prec \Gamma_{min}\preceq \Gamma_{max}\in\mathbb{R}^{(m+n)\times (m+n)}$, and let $\mathcal{T}$ be a minimal $\epsilon\leq 1/2$-net of $\mathcal{S}_{\Gamma_{min}}$ in the norm $\|\Gamma_{max}^{1/2}(\cdot)\|_2$. Then, $\log(|\mathcal{T}|)\leq (m+n)\log\left(1+\frac{2}{\epsilon}\right)+\log(det(\Gamma_{max}\Gamma_{min}^{-1}))$.
\end{mylemma}

Lemma \ref{lem:lem.a.3} characterizes the relationship between the pointwise lower bound and uniform lower bound when covering $\mathcal{S}^{m+n-1}$ with a $1/4$-net, while Lemma \ref{lem:lem.a.4} estimates the cardinality of the $1/4$-net $\mathcal{T}$ based on a standard volumetric argument. Armed with these two Lemmas, we are ready to prove Lemma \ref{lem:lem.2}.

Denote $Z_{te}=[z_{te,1} \cdots z_{te,D}]\in\mathbb{R}^{(m+n)\times D(L-M)}$ such that $Z_{te}Z_{te}^T=\sum_{d=1}^D z_{te,d}z_{te,d}^T=\sum_{d=1}^D\sum_{t=M}^{L-1}z_{t,d}z_{t,d}^T$.
First, along the same line as in Lemma \ref{lem:lem.a.1}, we can show that 
\begin{equation}\label{ZZupper}
    \mathbb{P}\left[Z_{te}Z_{te}^T\npreceq \frac{4(m+n)}{\delta}D(L-M)\Gamma_{L-1}=\Gamma_{max}\right]\leq \frac{\delta}{4}.
\end{equation}
Let $Q=Z_{te}$ and $\Gamma_{min}=\frac{D(L-M)\kappa p^2}{12}\underline{\Gamma}_{\lfloor k/2\rfloor}$. Denote the event by $\varepsilon_2:=\{ Z_{te}Z_{te}^T\npreceq \Gamma_{max} \}$.
We have that
\begin{align*}
    &\mathbb{P}\left[\sum_{d=1}^D z_{te,d}z^T_{te,d}\nsucceq \frac{D(L-M)p^2\kappa}{24}\underline{\Gamma}_{\lfloor k/2\rfloor}\right]\\
    \leq& \mathbb{P}\left[\exists u\in\mathcal{T}: \sum_{d=1}^D\sum_{t=M}^{L-1} \langle z_{t,d}, u\rangle^2 \leq \frac{D(L-M)p^2\kappa}{12}u^T\underline{\Gamma}_{\lfloor k/2\rfloor}u|\varepsilon_2^c\right]+\mathbb{P}[\varepsilon_2]\\
    \leq& \exp\Bigg\{-\frac{2D\kappa^2}{9\Bar{v}^4}\lambda^2_{min}(\underline{\Gamma}_{\lfloor k/2\rfloor}) + (m+n)\log(9)+\log det(\Gamma_{max}\Gamma_{min}^{-1})\Bigg\}+\frac{\delta}{4}\\
    \leq& \exp\Bigg\{-\frac{2D\kappa^2}{9\Bar{v}^4}\lambda^2_{min}(\underline{\Gamma}_{\lfloor k/2\rfloor}) + (m+n)\log(9)+(m+n)\log\left(\frac{Tr(\Gamma_{max})}{m+n}\right)\\
    &+(m+n)\log\left(\frac{1}{(m+n)\lambda_{min}(\Gamma_{min})}\right)\Bigg\}+\frac{\delta}{4}\\
    \overset{(a)}{\leq}& \exp\Bigg\{-\frac{2D\kappa^2}{9\Bar{v}^4}\lambda^2_{min}(\underline{\Gamma}_{\lfloor k/2\rfloor}) + (m+n)\log(9)+(m+n)\log\left(\frac{96(m+n)}{\delta p^2(1-e^{-\frac{p^2\lfloor (L-M)/k\rfloor}{8}})}\right)\\
    &+(m+n)\log\left(\frac{Tr(\Gamma_{L-1})}{m+n}\right)+(m+n)\log\left(\frac{1}{(m+n)\lambda_{min}(\underline{\Gamma}_{\lfloor k/2\rfloor})}\right)\Bigg\}+\frac{\delta}{4}\\
    \overset{(b)}{\leq}& \exp\Bigg\{-\frac{2D\kappa^2}{9\Bar{v}^4}\lambda^2_{min}(\underline{\Gamma}_{\lfloor k/2\rfloor}) + (m+n)\log(9)+(m+n)\log\left(\frac{8\times 96(m+n)}{\delta p^4}\right)\\
    &+(m+n)\log\left(\frac{Tr(\Gamma_{L-1})}{m+n}\right)+(m+n)\log\left(\frac{1}{(m+n)\lambda_{min}(\underline{\Gamma}_{\lfloor k/2\rfloor})}\right)\Bigg\}+\frac{\delta}{4}\\
    \leq& \exp\Bigg\{-\frac{2D\kappa^2}{9\Bar{v}^4}\lambda^2_{min}(\underline{\Gamma}_{\lfloor k/2\rfloor}) + 4(m+n)\log\left(\frac{6(m+n)}{\delta p}\right)+(m+n)\log\left(\frac{Tr(\Gamma_{L-1})}{m+n}\right)\\
    &+(m+n)\log\left(\frac{1}{(m+n)\lambda_{min}(\underline{\Gamma}_{\lfloor k/2\rfloor})}\right)\Bigg\}+\frac{\delta}{4}\\
    \leq& \exp\Bigg\{-\frac{2D\kappa^2}{9\Bar{v}^4}\lambda^2_{min}(\underline{\Gamma}_{\lfloor k/2\rfloor}) + 4(m+n)\log\left(\frac{6(m+n)}{\delta p}\right)+(m+n)\log\left(\frac{Tr(\Gamma_{L-1})}{\lambda_{min}(\underline{\Gamma}_{\lfloor k/2\rfloor})}\right)\Bigg\}+\frac{\delta}{4}\\
    \overset{(c)}{\leq}& \frac{\delta}{2},
\end{align*}
where (a) is true because $Tr(\Gamma_{max})=4(m+n)D(L-M) Tr(\Gamma_{L-1})/\delta$,
(b) holds because $k\leq\frac{L-M}{2}$, $p\in (0,1]$ and $1-e^{-x}\geq\frac{x}{2}$ for $x\in(0,1]$, and (c) is true if 
\begin{align*}
    &D\left(1-e^{-\frac{p^2\lfloor (L-M)/k\rfloor}{8}}\right)^2\\
    \geq&\frac{9\lambda_{max}^2(\Gamma_{\lfloor k/2\rfloor})}{2\lambda^2_{min}(\underline{\Gamma}_{\lfloor k/2\rfloor})}\left\{\log\left(\frac{4}{\delta}\right)+4(m+n)\log\left(\frac{6(m+n)}{\delta p}\right)+(m+n)\log\left(\frac{Tr(\Gamma_{L-1})}{\lambda_{min}(\underline{\Gamma}_{\lfloor k/2\rfloor})}\right)\right\}.
\end{align*}
Therefore, we can conclude that with probability $1-\frac{\delta}{2}$,
\begin{equation*}
    \lambda_{min}\left(\sum_{d=1}^D z_{te,d}z^T_{te,d}\right)\geq \frac{D(L-M)p^2(1-e^{-\frac{p^2\lfloor (L-M)/k\rfloor}{8}})}{24}\lambda_{min}(\underline{\Gamma}_{\lfloor k/2\rfloor}).
\end{equation*}
Combining with Lemma \ref{lem:lem.a.1} and \ref{lem:lowerboundzz}, based on the union bound, we can have that with probability  $1-\delta$
\begin{equation*}
    \lambda_{min}(ZZ^T)\geq \frac{D(L-M)p^2(1-e^{-\frac{p^2\lfloor (L-M)/k\rfloor}{8}})(1-\alpha\Bar{\lambda})^2\lambda_{min}(\underline{\Gamma}_{\lfloor k/2\rfloor})}{48},
\end{equation*}
thus finishing the proof of Lemma \ref{lem:lem.2}.

\section{Proof of Lemma \ref{lem:lem.3}}
\lemmab*

Note that $\|U^TP\|\leq\underset{\substack{v\in\mathcal{S}^{n-1}, u\in\mathcal{S}^{m+n-1}\setminus \{0\}}}{\sup}\frac{u^TZPv}{\|Z^Tu\|}$. For convenience, denote $Z=\begin{bmatrix}
Z_1 & \cdots & Z_D
\end{bmatrix}$ with $Z_d=(I-\alpha z_{tr,d}z_{tr,d}^T)z_{te,d}$, and 
$P=\begin{bmatrix}
P_1 &  \cdots  & P_D
\end{bmatrix}^T$ with $P_d=Z_d^T(\phi_d-\phi_j)$. Let $Y_d\triangleq u^TZ_dZ_d^T(\phi_d-\phi_j)v$.
It is clear that
\begin{equation*}
    u^TZPv=\sum_{d=1}^D u^TZ_dZ_d^T(\phi_d-\phi_j)v=\sum_{d=1}^D Y_d
\end{equation*}
which is a sum of $D$ independent random variables $\{Y_d\}_{d=1}^D$ for given $u$, $v$ and the blocks $\{\phi_d\}$. In what follows, we first take  a closer look at the random variable $Y_d$ by studying its expectation and variance.

1)   Since $\mathbb{E}[Y_d]=u^T\mathbb{E}[Z_dZ_d^T](\phi_d-\phi_j)v$, we focus on $\mathbb{E}[Z_dZ_d^T]$.  Following the same line in the proof of \eqref{eq:pd}, it can be shown that 
\begin{align}\label{eq:dominatezdzd}
    (I-\alpha z_{tr,d}z_{tr,d}^T)(z_{te,d}z_{te,d}^T+\epsilon_3 I)(I-\alpha z_{tr,d}z_{tr,d}^T)\preceq z_{te,d}z_{te,d}^T+\epsilon_3 I
\end{align}
for some $\epsilon_3>0$,
such that
\begin{align*}
    Z_dZ_d^T&=(I-\alpha z_{tr,d}z_{tr,d}^T)z_{te,d}z_{te,d}^T(I-\alpha z_{tr,d}z_{tr,d}^T)\\
    &\preceq z_{te,d}z_{te,d}^T+\epsilon_3 I-\epsilon_3 (I-\alpha z_{tr,d}z_{tr,d}^T)(I-\alpha z_{tr,d}z_{tr,d}^T)\\
    &=z_{te,d}z_{te,d}^T+\epsilon_3 [2\alpha z_{tr,d}z_{tr,d}^T-\alpha^2(z_{tr,d}z_{tr,d}^T)^2].
\end{align*}
Let $\epsilon_3=\frac{1}{2\alpha}$. It follows that
\begin{align*}
    Z_dZ_d^T\preceq z_{te,d}z_{te,d}^T+z_{tr,d}z_{tr,d}^T-\frac{\alpha}{2}(z_{tr,d}z_{tr,d}^T)^2.
\end{align*}
Let $z_d=[z_{0,d},...,z_{L-1,d}]$. Then, we can have
\begin{align*}
    \mathbb{E}[Z_dZ_d^T]&\preceq\mathbb{E}[z_{te,d}z_{te,d}^T+z_{tr,d}z_{tr,d}^T-\frac{\alpha}{2}(z_{tr,d}z_{tr,d}^T)^2]\\
    &=\mathbb{E}[z_{te,d}z_{te,d}^T+z_{tr,d}z_{tr,d}^T]-\frac{\alpha}{2}\mathbb{E}[(z_{tr,d}z_{tr,d}^T)^2]\\
    &=\mathbb{E}[z_dz_d^T]-\frac{\alpha}{2}\mathbb{E}[(z_{tr,d}z_{tr,d}^T)^2]\\
    &\preceq \mathbb{E}[z_dz_d^T]-\frac{\alpha}{2}(\mathbb{E}[z_{tr,d}z_{tr,d}^T])^2\\
    &=\sum_{t=0}^{L-1} \Gamma_{t,d}-\frac{\alpha}{2}\left(\sum_{t=0}^{M-1} \Gamma_{t,d}\right)^2\\
    &\preceq L\Gamma_{L-1}-\frac{\alpha}{2}\left(\sum_{t=0}^{M-1} \underline{\Gamma}_{t}\right)^2.
\end{align*}

2) Similarly, $\var[Y_d]=v^T(\phi_d-\phi_j)^T\cov[Z_dZ_d^T](\phi_d-\phi_j)v$, where
    $\|\cov[Z_dZ_d^T]\|\leq \|\cov[z_{te,d}z_{te,d}^T]\|\leq\|\mathbb{E}[z_{te,d}z_{te,d}^Tz_{te,d}z_{te,d}^T]\|$.
Expanding $\mathbb{E}[z_{te,d}z_{te,d}^Tz_{te,d}z_{te,d}^T]$, we can have that

\begin{align*}
    \|\mathbb{E}[z_{te,d}z_{te,d}^Tz_{te,d}z_{te,d}^T]\|
    =\left\|\begin{bmatrix}
    \mathbb{E}[\Tilde{Z}_{11}] & \mathbb{E}[\Tilde{Z}_{12}]\\
    \mathbb{E}[\Tilde{Z}_{21}] & \mathbb{E}[\Tilde{Z}_{22}]
    \end{bmatrix}\right\|
    \leq \|\mathbb{E}[\Tilde{Z}_{11}]\|+\|\mathbb{E}[\Tilde{Z}_{12}]\|+\|\mathbb{E}[\Tilde{Z}_{21}]\|+\|\mathbb{E}[\Tilde{Z}_{22}]\|
\end{align*}
where
\begin{align*}
   \mathbb{E}[\Tilde{Z}_{11}]=& \mathbb{E}\left[\left(\sum_{t=M}^{L-1}x_{t,d}x_{t,d}^T\right)\left(\sum_{t=M}^{L-1}x_{t,d}x_{t,d}^T\right)+\left(\sum_{t=M}^{L-1}x_{t,d}u_{t,d}^T\right)\left(\sum_{t=M}^{L-1}u_{t,d}x_{t,d}^T\right)\right],\\
   \mathbb{E}[\Tilde{Z}_{12}]=&
    \mathbb{E}\left[\left(\sum_{t=M}^{L-1}x_{t,d}x_{t,d}^T\right)\left(\sum_{t=M}^{L-1}x_{t,d}u_{t,d}^T\right)+\left(\sum_{t=M}^{L-1}x_{t,d}u_{t,d}^T\right)\left(\sum_{t=M}^{L-1}u_{t,d}u_{t,d}^T\right)\right],\\
    \mathbb{E}[\Tilde{Z}_{21}]=& \mathbb{E}\left[\left(\sum_{t=M}^{L-1}u_{t,d}x_{t,d}^T\right)\left(\sum_{t=M}^{L-1}x_{t,d}x_{t,d}^T\right)+\left(\sum_{t=M}^{L-1}u_{t,d}u_{t,d}^T\right)\left(\sum_{t=M}^{L-1}u_{t,d}x_{t,d}^T\right)\right],\\
    \mathbb{E}[\Tilde{Z}_{22}]=& \mathbb{E}\left[\left(\sum_{t=M}^{L-1}u_{t,d}x_{t,d}^T\right)\left(\sum_{t=M}^{L-1}x_{t,d}u_{t,d}^T\right)+\left(\sum_{t=M}^{L-1}u_{t,d}u_{t,d}^T\right)\left(\sum_{t=M}^{L-1}u_{t,d}u_{t,d}^T\right)\right].
\end{align*}
In what follows, we quantify each term.

(a) For the term $\|\mathbb{E}[\Tilde{Z}_{11}]\|$, we can first show that 
    \begin{align}\label{zfirst}
        &\left\|\mathbb{E}\left[\left(\sum_{t=M}^{L-1}x_{t,d}x_{t,d}^T\right)\left(\sum_{t=M}^{L-1}x_{t,d}x_{t,d}^T\right)\right]\right\|\nonumber\\
        =&\left\|\mathbb{E}\left[\sum_{t=M}^{L-1}x_{t,d}x_{t,d}^Tx_{t,d}x_{t,d}^T+\sum_{i=M}^{L-1}\sum_{j=M,j\neq i}^{L-1}x_{i,d}x_{i,d}^Tx_{j,d}x_{j,d}^T\right]\right\|\nonumber\\
        \leq&\sum_{t=M}^{L-1} \|\mathbb{E}[x_{t,d}x_{t,d}^Tx_{t,d}x_{t,d}^T]\|+\sum_{i=M}^{L-1}\sum_{j=M,j\neq i}^{L-1}\|\mathbb{E}[x_{i,d}x_{i,d}^Tx_{j,d}x_{j,d}^T]\|.
    \end{align}
    Further, for $\|\mathbb{E}[x_{t,d}x_{t,d}^Tx_{t,d}x_{t,d}^T]\|$ we have that 
    \begin{align*}
        \|\mathbb{E}[x_{t,d}x_{t,d}^Tx_{t,d}x_{t,d}^T]\|\leq \mathbb{E}[\|x_{t,d}x_{t,d}^Tx_{t,d}x_{t,d}^T\|]
        =&\mathbb{E}[x_{t,d}^Tx_{t,d}x_{t,d}^Tx_{t,d}]\\
        =&\mathbb{E}[\sum_{i=1}^n x_{t,d}^4(i)+\sum_{i=1}^n\sum_{j=1,j\neq i}^n x_{t,d}^2(i)x_{t,d}^2(j)]\\
        =&\sum_{i=1}^n \mathbb{E}[x_{t,d}^4(i)]+\sum_{i=1}^n\sum_{j=1,j\neq i}^n \mathbb{E}[x_{t,d}^2(i)x_{t,d}^2(j)]
    \end{align*}
    where $x_{t,d}(i)$ is the $i$-th element of the Gaussian vector $x_{t,d}$. 
    Here $\mathbb{E}[x_{t,d}^4(\cdot)]$ is the fourth moment of the Gaussian random variable $x_{t,d}(\cdot)$, which can be bounded above by some value $\mathcal{M}_4$ uniformly for any $d$ and $t$. 
    For the cross term $\mathbb{E}[x_{t,d}^2(i)x_{t,d}^2(j)]$ where $i\neq j$, based on the Holder's inequality, we can show that
    \begin{align}\label{holder}
        \mathbb{E}[x_{t,d}^2(i)x_{t,d}^2(j)]\leq \sqrt{\mathbb{E}[x_{t,d}^4(i)]\mathbb{E}[x_{t,d}^4(j)]}\leq \mathcal{M}_4.
    \end{align}
    Hence, it follows that
    \begin{equation}\label{xxxx}
        \|\mathbb{E}[x_{t,d}x_{t,d}^Tx_{t,d}x_{t,d}^T]\|\leq n^2\mathcal{M}_4.
    \end{equation}
    Similarly, for the cross term $\|\mathbb{E}[x_{i,d}x_{i,d}^Tx_{j,d}x_{j,d}^T]\|$ we have that
    \begin{align*}
        \|\mathbb{E}[x_{i,d}x_{i,d}^Tx_{j,d}x_{j,d}^T]\|\leq& \mathbb{E}[\|x_{i,d}x_{i,d}^T\|\|x_{j,d}x_{j,d}^T\|]\\
        \leq& \sqrt{\mathbb{E}[x^T_{i,d}x_{i,d}x^T_{i,d}x_{i,d}]\mathbb{E}[x^T_{j,d}x_{j,d}x^T_{j,d}x_{j,d}]}\\
        \leq& n^2\mathcal{M}_4.
    \end{align*}
    
 We can then  conclude that 
    \begin{equation}\label{z11_1}
        \left\|\mathbb{E}\left[\left(\sum_{t=M}^{L-1}x_{t,d}x_{t,d}^T\right)\left(\sum_{t=M}^{L-1}x_{t,d}x_{t,d}^T\right)\right]\right\|\leq (L-M)^2n^2\mathcal{M}_4.
    \end{equation}
    
    Next, for the term $\mathbb{E}\left[\left(\sum_{t=M}^{L-1}x_{t,d}u_{t,d}^T\right)\left(\sum_{t=M}^{L-1}u_{t,d}x_{t,d}^T\right)\right]$, similar to \eqref{zfirst}, it follows that
        \begin{align*}
        &\left\|\mathbb{E}\left[\left(\sum_{t=M}^{L-1}x_{t,d}u_{t,d}^T\right)\left(\sum_{t=M}^{L-1}u_{t,d}x_{t,d}^T\right)\right]\right\|\\
        \leq&\sum_{t=M}^{L-1} \|\mathbb{E}[x_{t,d}u_{t,d}^Tu_{t,d}x_{t,d}^T]\|+\sum_{i=M}^{L-1}\sum_{j=M,j\neq i}^{L-1}\|\mathbb{E}[x_{i,d}u_{i,d}^Tu_{j,d}x_{j,d}^T]\|.
    \end{align*}
For the term $\mathbb{E}[x_{t,d}u_{t,d}^Tu_{t,d}x_{t,d}^T]$, since $x_{t,d}$ and $u_{t,d}$ are independent, we can have that
\begin{align*}
    \mathbb{E}[x_{t,d}u_{t,d}^Tu_{t,d}x_{t,d}^T]=&\var[x_{t,d}u_{t,d}^T]+\mathbb{E}[x_{t,d}u_{t,d}^T]\mathbb{E}^T[x_{t,d}u_{t,d}^T]\\
    =& \var[x_{t,d}]+\var[u_{t,d}]\\
    \preceq& \sigma_a^2 G_{L-1}+\sigma_w^2 F_{L-1}+\sigma_a^2I_m
\end{align*}
where the last step is based on Lemma \ref{lem:lem.a.2} given that $x_{0,d}=0$ after reset in offline learning and $var[u_{t,d}]=\sigma^2_a I_m$. For the cross term $\mathbb{E}[x_{i,d}u_{i,d}^Tu_{j,d}x_{j,d}^T]$ where $i\neq j$, based on the Holder's inequality, we can obtain that
\begin{align}\label{xuux}
  \|\mathbb{E}[x_{i,d}u_{i,d}^Tu_{j,d}x_{j,d}^T]\|\leq& \mathbb{E}[\|x_{i,d}u_{i,d}^T\|\|u_{j,d}x_{j,d}^T\|]\nonumber\\
  \leq& \sqrt{\mathbb{E}[\|x_{i,d}u_{i,d}^T\|^2]\mathbb{E}[\|u_{j,d}x_{j,d}^T\|^2]}\nonumber\\
  \leq& \sqrt{\mathbb{E}\left[\sum_{p=1}^n\sum_{q=1}^m x_{i,d}^2(p)u_{i,d}^2(q)\right]\mathbb{E}\left[\sum_{p=1}^n\sum_{q=1}^m x_{j,d}^2(p)u_{j,d}^2(q)\right]}\nonumber\\
  \leq& mn\mathcal{M}_2\sigma_a^2
\end{align}
where the second moment $\mathbb{E}[x_{t,d}^2(\cdot)]$ is uniformly bounded above by some constant $\mathcal{M}_2$. Let $\Tilde{\mathcal{M}}_2=\max\left\{\|\sigma_a^2 G_{L-1}+\sigma_w^2 F_{L-1}+\sigma_a^2I_m\|, mn\mathcal{M}_2\sigma_a^2\right\}$, 
then it follows that
    \begin{equation}\label{z11_2}
        \left\|\mathbb{E}\left[\left(\sum_{t=M}^{L-1}x_{t,d}u_{t,d}^T\right)\left(\sum_{t=M}^{L-1}u_{t,d}x_{t,d}^T\right)\right]\right\|\leq (L-M)^2\Tilde{\mathcal{M}}_2.
    \end{equation}
Combing \eqref{z11_1} and \eqref{z11_2}, we conclude that
\begin{equation*}
    \|\mathbb{E}[\Tilde{Z}_{11}]\|\leq (L-M)^2(n^2\mathcal{M}_4+\Tilde{\mathcal{M}}_2).
\end{equation*}

(b) We analyze $\mathbb{E}[\Tilde{Z}_{12}]$ and $\mathbb{E}[\Tilde{Z}_{21}]$ together since $\mathbb{E}[\Tilde{Z}_{12}]=\mathbb{E}^T[\Tilde{Z}_{21}]$. Note that
\begin{equation*}
    \mathbb{E}\left[\left(\sum_{t=M}^{L-1}x_{t,d}x_{t,d}^T\right)\left(\sum_{t=M}^{L-1}x_{t,d}u_{t,d}^T\right)\right]=\sum_{t=M}^{L-1} \mathbb{E}[x_{t,d}x_{t,d}^Tx_{t,d}u_{t,d}^T]+\sum_{i=M}^{L-1}\sum_{j=M,j\neq i}^{L-1}\mathbb{E}[x_{i,d}x_{i,d}^Tx_{j,d}u_{j,d}^T].
\end{equation*}
Due to the independence between $x_{t,d}$ and $u_{t,d}$, $\mathbb{E}[x_{t,d}x_{t,d}^Tx_{t,d}u_{t,d}^T]=0$. For the cross term $\mathbb{E}[x_{i,d}x_{i,d}^Tx_{j,d}u_{j,d}^T]$: 1) If $i<j$, $\mathbb{E}[x_{i,d}x_{i,d}^Tx_{j,d}u_{j,d}^T]=0$ since $u_{j,d}$ is independent with other terms; 2) If $i>j$, based on \eqref{xxxx} and \eqref{xuux}, it can be shown that
\begin{align*}
    \|\mathbb{E}[x_{i,d}x_{i,d}^Tx_{j,d}u_{j,d}^T]\|\leq \sqrt{\mathbb{E}[\|x_{i,d}x_{i,d}^T\|^2]\mathbb{E}[\|x_{j,d}u_{j,d}^T\|^2]}\leq \sqrt{n^2\mathcal{M}_4\Tilde{\mathcal{M}}_2}.
\end{align*}
Therefore, it follows that
\begin{equation}\label{z12_1}
    \left\|\mathbb{E}\left[\left(\sum_{t=M}^{L-1}x_{t,d}x_{t,d}^T\right)\left(\sum_{t=M}^{L-1}x_{t,d}u_{t,d}^T\right)\right]\right\|\leq \frac{(L-M)^2}{2}\sqrt{n^2\mathcal{M}_4\Tilde{\mathcal{M}}_2}.
\end{equation}
Moreover,
\begin{equation*}
    \mathbb{E}\left[\left(\sum_{t=M}^{L-1}x_{t,d}u_{t,d}^T\right)\left(\sum_{t=M}^{L-1}u_{t,d}u_{t,d}^T\right)\right]=\sum_{t=M}^{L-1} \mathbb{E}[x_{t,d}u_{t,d}^Tu_{t,d}u_{t,d}^T]+\sum_{i=M}^{L-1}\sum_{j=M,j\neq i}^{L-1}\mathbb{E}[x_{i,d}u_{i,d}^Tu_{j,d}u_{j,d}^T].
\end{equation*}
It is clear that $\mathbb{E}[x_{t,d}u_{t,d}^Tu_{t,d}u_{t,d}^T]=0$. For the cross term $\mathbb{E}[x_{i,d}u_{i,d}^Tu_{j,d}u_{j,d}^T]$: 1) If $i<j$, $\mathbb{E}[x_{i,d}u_{i,d}^Tu_{j,d}u_{j,d}^T]=0$ due to the independence among $x_{i,d}$, $u_{i,d}$ and $u_{j,d}$; 2) If $i>j$, $x_{i,d}$ and $u_{j,d}$ are correlated but independent with $u_{i,d}$, so $\mathbb{E}[x_{i,d}u_{i,d}^Tu_{j,d}u_{j,d}^T]=0$. Hence,
\begin{equation}\label{z12_2}
    \mathbb{E}\left[\left(\sum_{t=M}^{L-1}x_{t,d}u_{t,d}^T\right)\left(\sum_{t=M}^{L-1}u_{t,d}u_{t,d}^T\right)\right]=0.
\end{equation}
Combing \eqref{z12_1} and \eqref{z12_2}, we conclude that
\begin{equation*}
    \|\mathbb{E}^T[\Tilde{Z}_{21}]\|=\|\mathbb{E}[\Tilde{Z}_{12}]\|\leq \frac{(L-M)^2}{2}\sqrt{n^2\mathcal{M}_4\Tilde{\mathcal{M}}_2}.
\end{equation*}

(c) For the term $\mathbb{E}[\Tilde{Z}_{22}]$, based on \eqref{z11_2}, we have that
\begin{equation}\label{z22_1}
    \left\|\mathbb{E}\left[\left(\sum_{t=M}^{L-1}u_{t,d}x_{t,d}^T\right)\left(\sum_{t=M}^{L-1}x_{t,d}u_{t,d}^T\right)\right]\right\|\leq (L-M)^2\Tilde{\mathcal{M}}_2.
\end{equation}
Further, note that
\begin{align*}
    &\left\|\mathbb{E}\left[\left(\sum_{t=M}^{L-1}u_{t,d}u_{t,d}^T\right)\left(\sum_{t=M}^{L-1}u_{t,d}u_{t,d}^T\right)\right]\right\|\\
    \leq&\sum_{t=M}^{L-1} \mathbb{E}[\|u_{t,d}u_{t,d}^Tu_{t,d}u_{t,d}^T\|]+\sum_{i=M}^{L-1}\sum_{j=M,j\neq i}^{L-1}\mathbb{E}[\|u_{i,d}u_{i,d}^Tu_{j,d}u_{j,d}^T\|]\\
    \leq& \sum_{t=M}^{L-1}\mathbb{E}[u_{t,d}^Tu_{t,d}u_{t,d}^Tu_{t,d}]+\sum_{i=M}^{L-1}\sum_{j=M,j\neq i}^{L-1}\sqrt{\mathbb{E}[\|u_{i,d}u_{i,d}^T\|^2]\mathbb{E}[\|u_{j,d}u_{j,d}^T\|^2]}\\
    =& \sum_{t=M}^{L-1}\mathbb{E}[u_{t,d}^Tu_{t,d}u_{t,d}^Tu_{t,d}]+\sum_{i=M}^{L-1}\sum_{j=M,j\neq i}^{L-1}\sqrt{\mathbb{E}[u_{i,d}^Tu_{i,d}u_{i,d}^Tu_{i,d}]\mathbb{E}[u_{j,d}^Tu_{j,d}u_{j,d}^Tu_{j,d}]}\\
    \leq& 3(L-M)^2m^2\sigma_a^4
\end{align*}
where the last step is true because 
\begin{equation*}
    \mathbb{E}[u_{t,d}^Tu_{t,d}u_{t,d}^Tu_{t,d}]=\sum_{i=1}^m \mathbb{E}[u_{t,d}^4(i)]+\sum_{i=1}^m\sum_{j=1,j\neq i}^m \mathbb{E}[u_{t,d}^2(i)u_{t,d}^2(j)]\leq 3m^2\sigma_a^4.
\end{equation*}
Together with \eqref{z22_1}, we can conclude that
\begin{equation*}
    \|\mathbb{E}[\Tilde{Z}_{22}]\|\leq (L-M)^2(\Tilde{\mathcal{M}}_2+3m^2\sigma_a^4).
\end{equation*}

In a nutshell, we can obtain that
\begin{align*}
    \var[Y_d]\leq& \|\phi_d-\phi_j\|^2\|\mathbb{E}[z_{te,d}z_{te,d}^Tz_{te,d}z_{te,d}^T]\|\\
    \leq& (L-M)^2\|\phi_d-\phi_j\|^2\left(n^2\mathcal{M}_4+2\Tilde{\mathcal{M}}_2+3m^2\sigma_a^4+\frac{1}{2}\sqrt{n^2\mathcal{M}_4\Tilde{\mathcal{M}}_2}\right)\\
    \triangleq& (L-M)^2\|\phi_d-\phi_j\|^2C_v.
\end{align*}

Next, based on Chebyshev's Inequality, we can have that
\begin{equation*}
    \mathbb{P}\left[\sum_{d=1}^D Y_d - \mathbb{E}\left[\sum_{d=1}^D Y_d\right]\geq \sqrt{\frac{1}{\delta}}\sqrt{\sum_{d=1}^D\var[Y_d]}\right]\leq \delta.
\end{equation*}
Let $V_{\phi}=\frac{1}{D}\sum_{d=1}^D \|\phi_d-\phi_j\|^2$. 
With probability $1-\delta$,
\begin{align}
    u^TZPv=\sum_{d=1}^D Y_d\leq& \mathbb{E}\left[\sum_{d=1}^D Y_d\right]+\sqrt{\frac{1}{\delta}}\sqrt{\sum_{d=1}^D\var[Y_d]}\nonumber\\
    \leq&\mathbb{E}\left[\sum_{d=1}^D Y_d\right]+(L-M)\sqrt{\frac{1}{\delta}}\sqrt{C_vDV_{\phi}}.
\end{align}
It thus follows that with probability $1-\delta$,
\begin{align*}
    \|U^TP\|\leq&\sup_{v\in\mathcal{S}^{n-1},u\in\mathcal{S}^{m+n-1}\setminus \{0\}}\frac{u^TZPv}{\|Z^Tu\|}\nonumber\\
    \leq & \frac{1}{\sqrt{u^TZZ^Tu}}\left\{\mathbb{E}\left[\sum_{d=1}^D Y_d\right]+(L-M)\sqrt{\frac{1}{\delta}}\sqrt{C_vDV_{\phi}}\right\}\nonumber\\
    \leq & \frac{1}{\sqrt{\lambda_{min}(ZZ^T)}}\left\{\sum_{d=1}^D u^T\mathbb{E}[Z_dZ_d^T](\phi_d-\phi_j)v+(L-M)\sqrt{\frac{1}{\delta}}\sqrt{C_vDV_{\phi}}\right\}\nonumber\\
    \leq & \frac{1}{\sqrt{\lambda_{min}(ZZ^T)}}\left\{\|\sum_{d=1}^D \mathbb{E}[Z_dZ_d^T](\phi_d-\phi_j)\|+(L-M)\sqrt{\frac{1}{\delta}}\sqrt{C_vDV_{\phi}}\right\}\nonumber\\
    \leq & \frac{1}{\sqrt{\lambda_{min}(ZZ^T)}}
    \left\{\sum_{d=1}^D \|\mathbb{E}[Z_dZ_d^T]\|\|(\phi_d-\phi_j)\|+(L-M)\sqrt{\frac{1}{\delta}}\sqrt{C_vDV_{\phi}}\right\}\\
    \leq & \frac{1}{\sqrt{\lambda_{min}(ZZ^T)}}\left\{
    \lambda_{max}\left(L \Gamma_{L-1}-\frac{\alpha}{2}\left(\sum_{t=0}^{M-1} \underline{\Gamma}_{t}\right)^2\right)\sum_{d=1}^D\|(\phi_d-\phi_j)\|+(L-M)\sqrt{\frac{1}{\delta}}\sqrt{C_vDV_{\phi}}\right\}\\
    \leq & \frac{D\eta\lambda_{max}\left(L\Gamma_{L-1}-\frac{\alpha}{2}\left(\sum_{t=0}^{M-1} \underline{\Gamma}_{t}\right)^2\right)+(L-M)\sqrt{\frac{1}{\delta}}\sqrt{C_vDV_{\phi}}}{\sqrt{\lambda_{min}(ZZ^T)}}
\end{align*}
thereby proving Lemma \ref{lem:lem.3}.


It is worth to note that the upper bound of $\|U^TP\|$ indeed is closely related to $\|\sum_{d=1}^D \mathbb{E}[Z_dZ_d^T](\phi_d-\phi_j)\|$, which involves sophisticated  coupling between $\mathbb{E}[Z_dZ_d^T]$ and $\phi_d$. With a more careful and complicated treatment of the coupling, this upper bound may be improved further.

\section{Upper bound on $\|U^TQ_w\|$}
  Along the lines in \cite{simchowitz2018learning}, we  study the quantities with $U$ in terms of $Z$, since $\|U^TQ_w\|\leq\sup_{v\in\mathcal{S}^{n-1},u\in\mathcal{S}^{m+n-1}\setminus \{0\}}\frac{u^TZQ_wv}{\|Z^Tu\|}$, with $\mathcal{S}^{n-1}$ being the unit sphere in $\mathbb{R}^n$.
 The key idea here is to control the deviation of sum of independent sub-Gaussian martingale sequences in terms of  variance proxies. 

  We have a few more words on $Z$. Let $Z=[Z_1\cdots Z_D]$ where $Z_d=(I-\alpha z_{tr,d}z_{tr,d}^T)z_{te,d}$ corresponds to the system response sequence of the $d$-th block. For each block, taking the first $M$ samples as training set $z_{tr,d}$ and the rest $L-M$ samples as the testing set $z_{te,d}$, it is not difficult to tell that $Z_d$ is a weighted sequence of the testing set $z_{te,d}$ in the sense that each element of $Z_d$ is a product of a constant matrix $C_{tr}=I-\alpha z_{tr,d}z_{tr,d}^T$ and $z_{t,d}$ for $M\leq t\leq L-1$. Therefore, the weighted system state sequence remains a martingale process for each block.


 Thanks to the independence among different blocks, we can obtain the following concentration result by using a similar martingale-Chernoff bound approach in \cite{simchowitz2018learning}:
\begin{restatable}{lem}{lemmac}
\label{lem:lem.4}
The following inequality holds with probability  $1-2\delta$: 
\begin{small}
\begin{equation*}
    \|U^TQ_w\|\leq 10\sigma_w\sqrt{C_Q+2\log \det(\lambda_{max}(\Gamma_{L-1})I)(\Tilde{\Gamma}_{min})^{-1}}\triangleq H_w,
\end{equation*}
\end{small}
where 
\begin{small}
\[C_Q=\log\frac{4}{3\delta}+n\log 5+8(m+n)\log 8+3(m+n)\log\left(\frac{m+n}{\delta}\right), \;
\Tilde{\Gamma}_{min}=\lambda_{min}(\underline{\Gamma}_{\lfloor k/2\rfloor})\left(1-\alpha\Bar{\lambda}\right)^2I. \]
\end{small}
\end{restatable}

The rest of this section is dedicated to prove Lemma \ref{lem:lem.4}. Note that 
\begin{equation}\label{utqw}
    \|U^TQ_w\|\leq\sup_{v\in\mathcal{S}^{n-1},u\in\mathcal{S}^{m+n-1}\setminus \{0\}}\frac{u^TZQ_wv}{\|Z^Tu\|}.
\end{equation}
Similar to the approach used in section \ref{sec:prooflem2} to obtain a lower bound of $\min_{u\in\mathcal{S}^{m+n-1}}\sum_{d=1}^D\sum_{t=M}^{L-1} \langle z_{t,d}, u\rangle^2$, we can approximate the supremum over an infinite set by the maximum over a $\epsilon$-net covering the infinite set.
\begin{mylemma}\label{lem:lem.c.1}\cite{simchowitz2018learning}
Let $Q\in \mathbb{R}^{(m+n)\times D(L-M)}$ have full row rank, $q\in \mathbb{R}^{D(L-M)}$, let $0\prec \Gamma_{min}\preceq QQ^T\preceq \Gamma_{max}\in\mathbb{R}^{(m+n)\times (m+n)}$, and let $\mathcal{T}$ be a $1/4$-net of $\mathcal{S}_{\Gamma_{min}}$ in the norm of $\|\Gamma_{max}^{1/2}\|_2$. Then,
\begin{equation*}
    \sup_{u\in\mathcal{S}^{m+n-1}}\frac{\langle Z^Tu, q\rangle}{\|Z^Tu\|}\leq 2\max_{u\in\mathcal{T}}\frac{\langle Z^Tu, q\rangle}{\|Z^Tu\|}.
\end{equation*}
\end{mylemma}
Therefore, by approximating $\mathcal{S}^{n-1}$ with a $1/2$-net $\mathcal{T}_1$ over $v$ and $\mathcal{S}^{m+n-1}$ with a $1/4$-net $\mathcal{T}_2$ over $u$, based on \eqref{utqw}, we can have
\begin{equation}\label{netappro}
    \|U^TQ_w\|\leq 4\max_{v\in\mathcal{T}_1}\max_{u\in\mathcal{T}_2}\frac{u^TZQ_wv}{\|Z^Tu\|}.
\end{equation}
To obtain an upper bound on \eqref{netappro}, we need to 1) evaluate the size of the $\epsilon$-nets, i.e., $|\mathcal{T}_1|$ and $|\mathcal{T}_2|$, and 2)  study the behaviour of $\frac{u^TZQ_wv}{\|Z^Tu\|}$ for a fixed pair of $(u,v)$ in the corresponding nets.

1) Recall the event $\varepsilon_2:=\{ Z_{te}Z_{te}^T\npreceq \Gamma_{max} \}$ defined in Appendix \ref{sec:prooflem2} where $Z_{te}=[z_{te,1} \cdots z_{te,D}]\in\mathbb{R}^{(m+n)\times D(L-M)}$ and $\Gamma_{max}=\frac{4(m+n)}{\delta}D(L-M)\Gamma_{L-1}$. Based on \eqref{ZZupper}, we have that
\begin{align}\label{upperzzt}
   \mathbb{P}\left[Z_{te}Z_{te}^T\preceq \Gamma_{max}\right]=\mathbb{P}[\varepsilon_2^c]\geq 1-\frac{\delta}{4}.
\end{align}

Based on \eqref{eq:dominatezdzd}, it follows that
\begin{align*}
    \sum_{d=1}^D [Z_dZ_d^T+\epsilon_4 (I-\alpha z_{tr,d}z_{tr,d}^T)^2]\preceq\sum_{d=1}^D z_{te,d}z_{te,d}^T+\epsilon_4DI,
\end{align*}
such that
\begin{align*}
    \lambda_{max}(ZZ^T)\leq \lambda_{max}(Z_{te}Z_{te}^T)+\epsilon_4D-\epsilon_4\lambda_{min}[\sum_{d=1}^D(I-\alpha z_{tr,d}z_{tr,d}^T)^2].
\end{align*}
Let $\epsilon_4=\frac{\lambda_{max}(Z_{te}Z_{te}^T)}{2[D-\lambda_{min}[\sum_{d=1}^D(I-\alpha z_{tr,d}z_{tr,d}^T)^2]]}$. Then, we have
\begin{align*}
    \lambda_{max}(ZZ^T)\leq 1.5\lambda_{max}(Z_{te}Z_{te}^T)\leq \frac{6(m+n)}{\delta}D(L-M)\lambda_{max}(\Gamma_{L-1}),
\end{align*}
which indicates that with probability $1-\delta/4$
\begin{align*}
    ZZ^T\preceq \frac{6(m+n)}{\delta}D(L-M)\lambda_{max}(\Gamma_{L-1})I=\Gamma'_{max}.
\end{align*}

And from Lemma \ref{lem:lem.2} we conclude that with probability $1-\delta$
\begin{equation}\label{lowerzz}
    ZZ^T\succeq \frac{D(L-M)p^2(1-e^{-\frac{p^2\lfloor (L-M)/k\rfloor}{8}})\lambda_{min}(\underline{\Gamma}_{\lfloor k/2\rfloor})}{48}\left(1-\alpha\Bar{\lambda}\right)^2I=\Gamma'_{min}.
\end{equation}
Therefore, based on Lemma 5.2 in \cite{vershynin2010introduction} about covering numbers of the unit sphere and Lemma \ref{lem:lem.a.4}, we can obtain that
\begin{equation}\label{boundsize}
    \log(|\mathcal{T}_1|)+\log(|\mathcal{T}_2|)\leq n\log 5+(m+n)\log 9 +\log(\det(\Gamma'_{max}(\Gamma'_{min})^{-1})).
\end{equation}

2) Next, to bound the point-wise $\frac{u^TZQ_wv}{\|Z^Tu\|}$ for fixed $(u,v)$, it suffices to study the concentration behaviour of a sum of sub-Gaussian random variables. In particular, 
\begin{equation}\label{sumzq}
    u^TZQ_wv=\sum_{d=1}^D u^TZ_dw_{te,d}^Tv
\end{equation}
and
\begin{equation}\label{sumzz}
    u^TZZ^Tv=\sum_{d=1}^D u^TZ_dZ_{d}^Tv,
\end{equation}
where both are the sum of in-block  sample covariances over all $D$ conditional independent blocks given the realizations of $\{\phi_d\}$. Therefore, if the term $u^TZ_dw_{te,d}^Tv$ for each block concentrates around the square root of $u^TZ_dZ_{d}^Tv$ for the same block, the summation \eqref{sumzq} is also likely to concentrate around the square root of \eqref{sumzz}.
To show this,
based on the Chernoff bound, it follows that for some constant $\gamma$ and $\beta$
\begin{align}\label{chernoff}
    &\mathbb{P}\left[\left\{u^TZQ_wv\geq \gamma \right\}\cap\left\{u^TZZ^Tu\leq\beta\right\}\right]\nonumber\\
    =&\mathbb{P}\left[\left\{\sum_{d=1}^D u^TZ_dw_{te,d}^Tv\geq \gamma \right\}\cap\left\{\sum_{d=1}^Du^TZ_dZ_d^Tu\leq\beta\right\}\right]\nonumber\\
    =&\inf_{\lambda>0}\mathbb{P}\left[\left\{e^{\lambda\sum_{d=1}^D u^TZ_dw_{te,d}^Tv}\geq e^{\lambda\gamma}\right\}\cap\left\{\sum_{d=1}^Du^TZ_dZ_d^Tu\leq\beta\right\}\right]\nonumber\\
    \leq&\inf_{\lambda>0} e^{-\lambda\gamma}e^{\lambda^2\sigma_w^2\beta/2}\mathbb{E}\left[e^{\lambda\sum_{d=1}^D u^TZ_dw_{te,d}^Tv-\lambda^2\frac{\sigma_w^2}{2}\sum_{d=1}^Du^TZ_dZ_d^Tu}\right]\nonumber\\
    =&\inf_{\lambda>0} e^{-\lambda\gamma}e^{\lambda^2\sigma_w^2\beta/2}\mathbb{E}\left[\prod_{d=1}^D e^{\lambda u^TZ_dw_{te,d}^Tv-\lambda^2\frac{\sigma_w^2}{2}u^TZ_dZ_d^Tu}\right]\nonumber\\
    =&\inf_{\lambda>0} e^{-\lambda\gamma}e^{\lambda^2\sigma_w^2\beta/2}\prod_{d=1}^D\mathbb{E}\left[e^{\lambda u^TZ_dw_{te,d}^Tv-\lambda^2\frac{\sigma_w^2}{2}u^TZ_dZ_d^Tu}\right]
\end{align}
where the last equality is true because of the conditional independence among different blocks. Note that 
\begin{equation*}
    Z_dw_{te,d}^T=\sum_{t=M}^{L-1}(I-\alpha z_{tr,d}z_{tr,d}^T)z_{t,d}w_{t,d}^T
\end{equation*}
where each term $(I-\alpha z_{tr,d}z_{tr,d}^T)z_{t,d}w_{t,d}^T$ behaves like a product of a constant $(I-\alpha z_{tr,d}z_{tr,d}^T)z_{t,d}$ and a sub-Gaussian random variable $w_{t,d}^T$ because $w_t|\mathcal{F}_{t-1}$ is mean-zero and $\sigma_w^2$-sub-Gaussian, 
and $z_{t,d}$ is fixed given $\mathcal{F}_{t-1}$. Similarly,
\begin{equation*}
    Z_dZ_d^T=\sum_{t=M}^{L-1}(I-\alpha z_{tr,d}z_{tr,d}^T)z_{t,d}z_{t,d}^T(I-\alpha z_{tr,d}z_{tr,d}^T)
\end{equation*}
where each term of the summation is also fixed given $\mathcal{F}_{t-1}$. Based on the tower rule \cite{simchowitz2018learning}, it is easy to show that
\begin{equation*}
    \mathbb{E}\left[e^{\lambda u^TZ_dw_{te,d}^Tv-\lambda^2\frac{\sigma_w^2}{2}u^TZ_dZ_d^Tu}\right]\leq 1.
\end{equation*}
Continuing with \eqref{chernoff}, we can have
\begin{equation*}
    \mathbb{P}\left[\left\{u^TZQ_wv\geq \gamma \right\}\cap\left\{u^TZZ^Tu\leq\beta\right\}\right]\leq e^{-\frac{\gamma^2}{2\sigma_w^2\beta}}.
\end{equation*}
Following a similar argument in Lemma 4.2 (b) in \cite{simchowitz2018learning}, we conclude that
\begin{equation*}
    \mathbb{P}\left[\left\{\frac{u^TZQ_wv}{\|Z^Tu\|}\geq \frac{H_w}{4}\right\}\cap\left\{\Gamma'_{min}\preceq ZZ^T\preceq \Gamma'_{max}\right\}\right]\leq \log\left(\left\lceil\frac{u^T\Gamma'_{max}u}{u^T\Gamma'_{min}u}\right\rceil\right)\exp\left(-\frac{H_w^2}{96\sigma_w^2} \right)
\end{equation*}
such that
\begin{small}
\begin{align*}
    &\mathbb{P}\left[\left\{\|U^TQ_w\|\geq H_w\right\}\cap\left\{\Gamma'_{min}\preceq ZZ^T\preceq \Gamma'_{max}\right\}\right]\\
    \leq&|\mathcal{T}_1||\mathcal{T}_2|\max_{v\in\mathcal{T}_1}\max_{u\in\mathcal{T}_2}\mathbb{P}\left[\left\{\frac{u^TZQ_wv}{\|Z^Tu\|}\geq \frac{H_w}{4}\right\}\cap\left\{\Gamma'_{min}\preceq ZZ^T\preceq \Gamma'_{max}\right\}\right]\\
    \leq&\exp\left(n\log 5+(m+n)\log 9 +\log(\det(\Gamma'_{max}(\Gamma'_{min})^{-1}))\right)\log\left(\left\lceil\frac{u^T\Gamma'_{max}u}{u^T\Gamma'_{min}u}\right\rceil\right)\exp\left(-\frac{H_w^2}{96\sigma_w^2} \right)\\
    \overset{(a)}{\leq}&\exp\left(n\log 5+(m+n)\log 9 +\log(\det(\Gamma'_{max}(\Gamma'_{min})^{-1}))\right)\\
    &\cdot\exp\left(\log(\lambda_{max}(\Gamma'_{max}(\Gamma'_{min})^{-1}))\right)\exp\left(-\frac{H_w^2}{96\sigma_w^2} \right)\\
    \leq&\exp\left(n\log 5+(m+n)\log 9+2\log(\det(\Gamma'_{max}(\Gamma'_{min})^{-1}))-\frac{H_w^2}{96\sigma_w^2}\right)\\
    \leq&\exp\left(n\log 5+8(m+n)\log 8+2(m+n)\log\left(\frac{m+n}{\delta}\right)+3\log (\det(\lambda_{max}(\Gamma_{L-1})I)(\Tilde{\Gamma}_{min})^{-1})-\frac{H_w^2}{96\sigma_w^2}\right)\\
    \leq&\frac{3\delta}{4},
\end{align*}\par
\end{small}
where (a) is true because 
\begin{small}
\begin{align*}
    \log\left(\left\lceil\frac{u^T\Gamma'_{max}u}{u^T\Gamma'_{min}u}\right\rceil\right)\leq \frac{u^T\Gamma'_{max}u}{u^T\Gamma'_{min}u}\leq \|(\Gamma'_{min})^{-\frac{1}{2}}\Gamma'_{max}(\Gamma'_{min})^{-\frac{1}{2}}\|=\lambda_{max}(\Gamma'_{max}(\Gamma'_{min})^{-1}),
\end{align*}\par
\end{small}
$\Tilde{\Gamma}_{min}=\lambda_{min}(\underline{\Gamma}_{\lfloor k/2\rfloor})\left(1-\alpha\Bar{\lambda}\right)^2I$
and $H_w$ is chosen as
\begin{small}
\begin{equation*}
    H_w=10\sigma_w\sqrt{\log\frac{4}{3\delta}+n\log 5+8(m+n)\log 8+3(m+n)\log\left(\frac{m+n}{\delta}\right)+2\log \det(\lambda_{max}(\Gamma_{L-1})I)(\Tilde{\Gamma}_{min})^{-1}}.
\end{equation*}
\end{small}

Therefore, Lemma \ref{lem:lem.4} can be proved by the union bound.

\section{Upper Bound on $\|U^TQ_0\|$}
Each row block $Q_{0,d}=z_{te,d}^Tz_{tr,d}w_{tr,d}^T$ in the matrix $Q_0$ is intimately related to how the training noise $w_{tr,d}^T$ is amplified during the system evolution, i.e.,  how the system is excited by the training noise $w_{tr,d}^T$. Similarly, instead of directly working with $U$, we study this error term with $Z$, i.e., $\|U^TQ_0\|\leq \sup_{v\in\mathcal{S}^{n-1},u\in\mathcal{S}^{m+n-1}} \frac{u^TZQ_0v}{\|Z^Tu\|}$. However, the martingale-Chernoff bound approach is not applicable here due to the complicated correlation structure between $z_{te,d}$ and $w_{tr,d}$ within each block. Therefore, we  will seek bounds for the numerator and the denominator separately, which leads to the following lemma.
\begin{restatable}{lem}{lemmad}
\label{lem:lem.5}
The following inequality holds with probability $1-\delta$:
\begin{small}
\begin{align*}
    \|U^TQ_0\|\leq& C_w\left(\frac{48}{Dp^2(1-e^{-\frac{p^2\lfloor (L-M)/k\rfloor}{8}})}\right)^{m+n}\sqrt{D^2M^3(L-M)^3}\sqrt{\det(\Gamma^{te}_{max}(\Tilde{\Gamma}_{min})^{-1})}\triangleq H_0,
\end{align*}
\end{small}
where 
\begin{small}
\[C_w=\sigma_w(\sqrt{n}+\sqrt{\frac{2}{D}\log\frac{4}{\delta}})(1+3\sqrt{\log\frac{20(L-M)D}{\delta}})(1+3\sqrt{\log\frac{20MD}{\delta}})\|\Bar{B}\|^2\max\{m\sigma_a^2,n\sigma_w^2\}, \]
\end{small}
and 
\begin{small}
\[\Gamma^{te}_{max}=1.5(L-M)^3\|\Bar{B}\|^2\Big(1+3\sqrt{\log\frac{20LD}{\delta}}\Big)^2\max\{m\sigma_a^2,n\sigma_w^2\}(1-\alpha\Bar{\lambda})^2I. \]
\end{small}
\end{restatable}


The rest of this section is dedicated to prove Lemma \ref{lem:lem.5}. Note that 
\begin{equation}\label{upperutq}
    \|U^TQ_0\|\leq\sup_{v\in\mathcal{S}^{n-1},u\in\mathcal{S}^{m+n-1}\setminus \{0\}}\frac{u^TZQ_0v}{\|Z^Tu\|},
\end{equation}
and $Q_{0,d}=z_{te,d}^Tz_{tr,d}w_{tr,d}^T$. 
Due to the complicated correlation between $Z_d=(I-\alpha z_{tr,d}z_{tr,d}^T)z_{te,d}$ and $w_{tr,d}$, we  bound $\|Z^Tu\|$ and $u^TZQ_0v$ in \eqref{upperutq} separately.

Observe that
\begin{equation}
    \|Z^Tu\|=\sqrt{u^TZZ^Tu}\geq \sqrt{u^T\Gamma'_{min}u}.
\end{equation}
Further, it can be shown that 
\begin{align*}
  u^TZQ_0v=\sum_{d=1}^D u^TZ_dQ_{0,d}v
    \leq& \sum_{d=1}^D \|u^TZ_dQ_{0,d}v\|\\
    \leq& \sum_{d=1}^D \|u^TZ_d\|\|Q_{0,d}v\|\\
    \leq& \max_d\{\sqrt{u^TZ_dZ_d^Tu}\}\sum_{d=1}^D\|Q_{0,d}v\|\\
    \leq&\max_d\{\sqrt{u^TZ_dZ_d^Tu}\}\max_{d}\{\|z_{te,d}^T\|\}\max_{d}\{\|z_{tr,d}^T\|\}\sum_{d=1}^D \|w_{tr,d}v\|.
\end{align*} 

Based on \eqref{eq:dominatezdzd}, we have
\begin{align*}
    \lambda_{max}[Z_dZ_d^T+\epsilon_5 (I-\alpha z_{tr,d}z_{tr,d}^T)^2]\leq \lambda_{max}(z_{te,d}z_{te,d}^T)+\epsilon_5
\end{align*}
and
\begin{align*}
    \lambda_{max}[Z_dZ_d^T+\epsilon_5 (I-\alpha z_{tr,d}z_{tr,d}^T)^2]\geq \lambda_{max}(Z_dZ_d^T)+\epsilon_5\lambda_{min}[(I-\alpha z_{tr,d}z_{tr,d}^T)^2].
\end{align*}
Hence, for $\epsilon_5=\frac{\lambda_{max}(z_{te,d}z_{te,d}^T)}{2[1-\lambda_{min}[(I-\alpha z_{tr,d}z_{tr,d}^T)^2]]}$, it follows that
\begin{align*}
    \lambda_{max}(Z_dZ_d^T)\leq& \lambda_{max}(z_{te,d}z_{te,d}^T)+\epsilon_5-\epsilon_5\lambda_{min}[(I-\alpha z_{tr,d}z_{tr,d}^T)^2]\\
    =&1.5\lambda_{max}(z_{te,d}z_{te,d}^T).
\end{align*}

Following the same line with Lemma \ref{lem:boundlambda}, we first have that with probability $1-\delta/4$:
\begin{equation*}
    \sqrt{u^TZ_dZ_d^Tu}\leq \sqrt{u^T\Gamma^{te}_{max}u},
\end{equation*}
and with probability $1-\delta/4$:
\begin{equation*}
    \max_d\{\|z_{te,d}^T\|\}=\max_d\{\sqrt{\lambda_{max}(z_{te,d}z_{te,d}^T)}\}\leq \sqrt{\lambda^{te}_{max}}
\end{equation*}
where $\Gamma^{te}_{max}=1.5\lambda^{te}_{max}I$ and
$\lambda^{te}_{max}=(L-M)^3\|\Bar{B}\|^2\Big(1+3\sqrt{\log\frac{20(L-M)D}{\delta}}\Big)^2\max\{m\sigma_a^2,n\sigma_w^2\}$.
Based on Lemma \ref{lem:boundlambda}, we have that with probability $1-\delta/4$:
\begin{equation*}
    \max_d\{\|z_{tr,d}^T\|\}=\max_d\{\sqrt{\lambda_{max}(z_{tr,d}z_{tr,d}^T)}\}\leq \sqrt{\lambda^{tr}_{max}}
\end{equation*}
where $\lambda^{tr}_{max}=M^3\|\Bar{B}\|^2\Big(1+3\sqrt{\log\frac{20MD}{\delta}}\Big)^2\max\{m\sigma_a^2,n\sigma_w^2\}$.

Since $w_{t,d}$'s are all independent across $t$ and $d$, we can have that $w_{tr,d}v\sim\mathcal{N}(0,\sigma_w^2I_n)$. Hence, 
\begin{equation*}
    \sum_{d=1}^D \|w_{tr,d}v\|=\sigma_w\sum_{d=1}^D\|\Tilde{w}_{tr,d}\|\leq \sigma_w\sqrt{D}\sqrt{\sum_{d=1}^D \|\Tilde{w}_{tr,d}\|^2}
\end{equation*}
where $\Tilde{w}_{tr,d}\sim\mathcal{N}(0,I_n)$. Let $W_D=\sum_{d=1}^D \|\Tilde{w}_{tr,d}\|^2$. Then, $W_D$ follows the chi-square distribution with $nD$ degrees of freedom, i.e., $W_D\sim\chi^2(nD)$, such that the following concentration result holds:
\begin{equation*}
    \mathbb{P}[W_D\geq nD+2\sqrt{nD\log\frac{4}{\delta}}+2\log\frac{4}{\delta}]\leq \frac{\delta}{4}.
\end{equation*}

Therefore, based on \eqref{upperutq}, we can conclude that with probability $1-\delta$:
\begin{align*}
    \|U^TQ_0\|\leq&\sup_{v\in\mathcal{S}^{n-1},u\in\mathcal{S}^{m+n-1}} \frac{u^TZQ_0v}{\|Z^Tu\|}\\
    \leq& \sup_{u\in\mathcal{S}^{m+n-1}}\sigma_w\sqrt{D}\sqrt{nD+2\sqrt{nD\log\frac{4}{\delta}}+2\log\frac{4}{\delta}}\sqrt{\lambda^{te}_{max}\lambda^{tr}_{max}}\sqrt{\frac{u^T\Gamma^{te}_{max}u}{u^T\Gamma'_{min}u}}\\
    \leq& \sigma_wD(\sqrt{n}+\sqrt{\frac{2}{D}\log\frac{4}{\delta}})[M(L-M)]^{1.5}(1+3\sqrt{\log\frac{20(L-M)D}{\delta}})(1+3\sqrt{\log\frac{20MD}{\delta}})\\
    &\cdot\|\Bar{B}\|^2\max\{m\sigma_a^2,n\sigma_w^2\}\sqrt{\det(\Gamma^{te}_{max}(\Gamma'_{min})^{-1})}\\
     \leq& \sigma_w(\sqrt{n}+\sqrt{\frac{2}{D}\log\frac{4}{\delta}})(1+3\sqrt{\log\frac{20(L-M)D}{\delta}})(1+3\sqrt{\log\frac{20MD}{\delta}})\|\Bar{B}\|^2\\
     &\cdot \max\{m\sigma_a^2,n\sigma_w^2\}\left(\frac{48}{Dp^2(1-e^{-\frac{p^2\lfloor (L-M)/k\rfloor}{8}})}\right)^{m+n}\sqrt{D^2M^3(L-M)^3}\sqrt{\det(\Gamma^{te}_{max}(\Tilde{\Gamma}_{min})^{-1})}\\
     =&H_0,
\end{align*}
where $H_0$ decays quickly for a large enough $D$.

\section{Proof of Theorem \ref{thm:thm.1}}
\theorema*
Based on Lemmas \ref{lem:lem.2}-\ref{lem:lem.5}, by the union bound, we can have
\begin{small}
\begin{align*}
    \|\phi^*_{\theta}-\phi_j\|\leq& \frac{1}{\sqrt{\lambda_{min}(ZZ^T)}}\left(\|U^TP\|+\|U^TQ_w\|+\alpha\|U^TQ_0\|\right)\\
    \leq& \frac{1}{\sqrt{\lambda_{min}(ZZ^T)}}\left(\frac{D\eta\lambda_{max}\left(L\Gamma_{L-1}-\frac{\alpha}{2}\left(\sum_{t=0}^{M-1} \underline{\Gamma}_{t}\right)^2\right)+(L-M)\sqrt{\frac{1}{\delta}}\sqrt{C_vDV_{\phi}}}{\sqrt{\lambda_{min}(ZZ^T)}}+H_w+\alpha H_0\right)\\
    =&\frac{D\eta\lambda_{max}\left(L\Gamma_{L-1}-\frac{\alpha}{2}\left(\sum_{t=0}^{M-1} \underline{\Gamma}_{t}\right)^2\right)+(L-M)\sqrt{\frac{1}{\delta}}\sqrt{C_vDV_{\phi}}}{\lambda_{min}(ZZ^T)}+\frac{H_w+\alpha H_0}{\sqrt{\lambda_{min}(ZZ^T)}}\\
    \leq& \frac{48\eta\lambda_{max}\left(L\Gamma_{L-1}-\frac{\alpha}{2}\left(\sum_{t=0}^{M-1} \underline{\Gamma}_{t}\right)^2\right)}{(L-M) p^2(1-e^{-\frac{p^2\lfloor (L-M)/k\rfloor}{8}})(1-\alpha\Bar{\lambda})^2\lambda_{min}(\underline{\Gamma}_{\lfloor k/2\rfloor})}
    +\frac{(L-M)\sqrt{\frac{1}{\delta}}\sqrt{C_vDV_{\phi}}}{\lambda_{min}(ZZ^T)}
    +\frac{H_w+\alpha H_0}{\sqrt{\lambda_{min}(ZZ^T)}}\\
    \triangleq& C_0\eta+\Tilde{O}\left(\frac{1}{\sqrt{D}}\right)\sqrt{V_{\phi}}+\Tilde{O}\left(\frac{1}{\sqrt{D(L-M)}}\right),
\end{align*}\par
\end{small}
where 
\begin{align*}
C_0=&\frac{48\lambda_{max}\left(L\Gamma_{L-1}-\frac{\alpha}{2}\left(\sum_{t=0}^{M-1} \underline{\Gamma}_{t}\right)^2\right)}{(L-M) p^2(1-e^{-\frac{p^2\lfloor (L-M)/k\rfloor}{8}})(1-\alpha\Bar{\lambda})^2\lambda_{min}(\underline{\Gamma}_{\lfloor k/2\rfloor})}.
\end{align*}

\section{Proof of Proposition  \ref{thm:thm.2}}
\propositiona*
We focus on the adaptation performance for the block $i$, i.e., $\|\hat{\phi}_i(M)-\phi_i\|$.
For ease of exposition, we use $\hat{\phi}_t$ to denote the update $\hat{\phi}_i(t)$ during the online adaptation for a new block $i$ at time $t$. We omit the subscript $i$ for the samples within block $i$, and with a bit abuse of notation, we use $z_t$  to `stand for' $\Tilde{z}_{t,i}$ in this proof.

For convenience, let $\Bar{g}(\phi)$ denote the expected  gradient when the linear system reaches its steady state and $\hat{x}_t$ denote the steady state. We can have that
\begin{align}\label{steadyg}
    \Bar{g}(\phi)=&\mathbb{E}[\nabla\mathcal{L}((\hat{x}_t,u_t,\hat{x}_{t+1}),\phi)]\nonumber\\
    =&\mathbb{E}[\hat{z}_t\hat{z}_t^T\phi-\hat{z}_t\hat{x}_{t+1}^T|]\nonumber\\
    =&\mathbb{E}[\hat{z}_t\hat{z}_t^T(\phi-\phi_i)-\hat{z}_tw_t^T]\nonumber\\
    =&\begin{bmatrix}
    I\\K
    \end{bmatrix}P_{\infty}\begin{bmatrix}
    I & K^T
    \end{bmatrix}(\phi-\phi_i)\triangleq \hat{P}_{\infty} \cdot (\phi-\phi_i).
\end{align}
Then it follows  that
\begin{align}\label{decomp}
    \|\hat{\phi}_{t+1}-\phi_i\|^2\leq&\|\hat{\phi}_{t}-\phi_i-\alpha g_t(\hat{\phi}_{t})\|^2\nonumber\\
    =&\|\hat{\phi}_{t}-\phi_i\|^2+\alpha^2\|g_t(\hat{\phi}_{t})\|^2-2\alpha \langle \hat{\phi}_{t}-\phi_i, g_t(\hat{\phi}_{t})\rangle\nonumber\\
    =&\|\hat{\phi}_{t}-\phi_i\|^2+\alpha^2\|g_t(\hat{\phi}_{t})\|^2-2\alpha\langle \hat{\phi}_{t}-\phi_i, \Bar{g}(\hat{\phi}_{t})\rangle+2\alpha\langle \phi_{i}-\hat{\phi}_t, g_t(\hat{\phi}_{t})-\Bar{g}(\hat{\phi}_{t})\rangle\nonumber\\
    =&\|\hat{\phi}_{t}-\phi_i\|^2-2\alpha Tr((\hat{\phi}_{t}-\phi_i)(\hat{\phi}_{t}-\phi_i)^T\hat{P}_{\infty})+\alpha^2\|g_t(\hat{\phi}_{t})\|^2+2\alpha\xi_t(\hat{\phi}_t)\nonumber\\
    \leq& \|\hat{\phi}_{t}-\phi_i\|^2-2\alpha\lambda_{min}(\hat{P}_{\infty})Tr((\hat{\phi}_{t}-\phi_i)(\hat{\phi}_{t}-\phi_i)^T)+\alpha^2\|g_t(\hat{\phi}_{t})\|^2+2\alpha\xi_t(\hat{\phi}_t)\nonumber\\
    \leq&[1-2\alpha \lambda_{min}(\hat{P}_{\infty})]\|\hat{\phi}_{t}-\phi_i\|^2+\alpha^2\|g_t(\hat{\phi}_{t})\|^2+2\alpha\xi_t(\hat{\phi}_t)
\end{align}
where   $\xi_t(\hat{\phi}_t):=\langle \phi_{i}-\hat{\phi}_t, g_t(\hat{\phi}_{t})-\Bar{g}(\hat{\phi}_{t})\rangle$. Next, we need to find upper bounds on  $\|g_t(\hat{\phi}_{t})\|^2$ (the gradient update at time $t$) and $\xi_t(\hat{\phi}_t)$ (which can be regarded as the estimation error when evaluating the gradient for the update).

(a) For the term $\|g_t(\hat{\phi}_{t})\|^2$, we can have
\begin{align}\label{uppergradient}
    \mathbb{E}[\|g_t(\hat{\phi}_{t})\|^2]
    =&\mathbb{E}[\|z_tz_t^T(\hat{\phi}_{t}-\phi_i)-z_tw_t^T\|^2]\nonumber\\
    \leq&\mathbb{E}[\|z_tz_t^T\|^2\|\hat{\phi}_{t}-\phi_i\|^2+2\|\hat{\phi}_{t}-\phi_i\|\|z_tz_t^T\|\|z_tw_t^T\|+\|z_tw_t^T\|^2]\nonumber\\
    \leq& 4C_{\phi}^2C_z^4+4C_{\phi}C_z^3\mathbb{E}[\|w_t\|]+C_z^2\mathbb{E}[\|w_t\|^2].
\end{align}

For convenience, let $w_t=[e_1,...,e_{n}]^T$. For the term $\mathbb{E}[\|w_t\|]$, we can have
\begin{align*}
    \mathbb{E}[\|w_t\|]=\mathbb{E}\left[\sqrt{\sum_{k=1}^n e_k^2}\right]=\sigma_w\mathbb{E}\left[\sqrt{\sum_{k=1}^n \left(\frac{e_k}{\sigma_w}\right)^2}\right]
    =\sigma_w\mathbb{E}[Y_1]
\end{align*}
where $Y_1$ follows the non-central chi distribution with degree of freedom $n$ and $\lambda_1=0$. It is known that the first moment of $Y_1$ is given by 
\begin{equation*}
    \mu_1'=\sqrt{\frac{\pi}{2}}\mathbb{L}_{1/2}^{(\frac{n}{2}-1)}\left(0\right)
\end{equation*}
where $\mathbb{L}_b^{(a)}(z)$ is a Laguerre function. Therefore,
\begin{equation*}
    \mathbb{E}[\|w_t\|]=\sigma_w\mu_1'.
\end{equation*}

For the term $\mathbb{E}[\|w_t\|^2]$, we can have
\begin{align*}
    \mathbb{E}[\|w_t\|^2]=\mathbb{E}\left[\sum_{k=1}^n e_k^2\right]=\sigma_w\mathbb{E}\left[\sum_{k=1}^n \left(\frac{e_k}{\sigma_w}\right)^2\right]
    =\sigma_w\mathbb{E}[Y_2]
\end{align*}
where $Y_2$ follows the non-central chi-squared  distribution with degree of freedom $n$ and $\lambda_2=0$. It is known that the first moment of $Y_2$ is equal to $n$, such that 
\begin{equation*}
    \mathbb{E}[\|w_t\|^2]=n\sigma_w.
\end{equation*}

Therefore, based on \eqref{uppergradient}, we have
\begin{equation*}
    \mathbb{E}[\|g_t(\hat{\phi}_{t})\|^2]\leq 4C_{\phi}^2C_z^4+4C_{\phi}C_z^3\sigma_w\mu_1'+C_z^2n\sigma_w\triangleq C_g.
\end{equation*}

(b) Next we need to analyze $\xi_t(\hat{\phi}_t):=\langle \phi_{i}-\hat{\phi}_t, g_t(\hat{\phi}_{t})-\Bar{g}(\hat{\phi}_{t})\rangle$. Specifically,
\begin{align}\label{xi_t}
    \mathbb{E}[\xi_t(\hat{\phi}_t)]=&\mathbb{E}[\langle \phi_{i}-\hat{\phi}_t, g_t(\hat{\phi}_{t})-\Bar{g}(\hat{\phi}_{t})\rangle]\nonumber\\
    =&\mathbb{E}[Tr((\phi_i-\hat{\phi}_t)(g_t(\hat{\phi}_{t})-\Bar{g}(\hat{\phi}_{t}))^T)]\nonumber\\
    =&Tr(\mathbb{E}[(\phi_i-\hat{\phi}_t)(g_t(\hat{\phi}_{t})-\Bar{g}(\hat{\phi}_{t}))^T]).
\end{align}
It follows that
\begin{align*}
    &\mathbb{E}[(\phi_i-\hat{\phi}_t)(g_t(\hat{\phi}_{t})-\Bar{g}(\hat{\phi}_{t}))^T]\\
    =&\mathbb{E}[(\phi_i-\hat{\phi}_t)(z_tz_t^T(\hat{\phi}_t-\phi_i)-z_tw_t^T-\hat{z}_t\hat{z}_t^T(\hat{\phi}_t-\phi_i)+\hat{z}_t\hat{w}_t^T)^T]\\
    =&\mathbb{E}[(\phi_i-\hat{\phi}_t)(\phi_i-\hat{\phi}_t)^T(\hat{z}_t\hat{z}_t^T-z_tz_t^T)]+\mathbb{E}[(\phi_i-\hat{\phi}_t)(\hat{w}_t\hat{z}_t^T-w_tz_t^T)]\\
    =&\mathbb{E}[(\phi_i-\hat{\phi}_t)(\phi_i-\hat{\phi}_t)^T\hat{z}_t\hat{z}_t^T]-\mathbb{E}[(\phi_i-\hat{\phi}_t)(\phi_i-\hat{\phi}_t)^Tz_tz_t^T].
\end{align*}

Denote  $f:=(\phi_i-\hat{\phi}_t)(\phi_i-\hat{\phi}_t)^Tzz^T$ as a function of state $x$, and $h:=\frac{f}{2\|f\|_{\infty}}$. Then, along the same line as in \cite{bhandari2018finite},
\begin{align*}
   \| \mathbb{E}[h(\hat{x}_t)]-\mathbb{E}[h(x_t)]\|=&\left\|\int hd\nu_{\infty}-\int hd\nu_t\right\|\\
    \leq d_{TV}(\nu_{\infty}, \nu_t),
\end{align*}
where $d_{TV}(P,Q)$ denotes the total-variation distance between probability measures $P$ and $Q$, and $\nu_t$ is the distribution of $x_t$ and $\nu_{\infty}$ is the steady distribution of $x_t$ (corresponding to the distribution of $\hat{x}_t$).
Based on the mixing property of LTI systems, we can know that
\begin{equation*}
    d_{TV}(\nu_{\infty}, \nu_t)\leq C_m\rho^t.
\end{equation*}
such that
\begin{equation*}
    \|\mathbb{E}[(\phi_i-\hat{\phi}_t)(g_t(\hat{\phi}_{t})-\Bar{g}(\hat{\phi}_{t}))^T]\|\leq 2\|f\|_{\infty}C_m\rho^t\leq 8C_z^2C_{\phi}^2C_m\rho^t.
\end{equation*}

Therefore, continuing with \eqref{xi_t}, we have
\begin{align*}
    \mathbb{E}[\xi_t(\hat{\phi}_t)]\leq& n\|\mathbb{E}[(\phi_i-\hat{\phi}_t)(g_t(\hat{\phi}_{t})-\Bar{g}(\hat{\phi}_{t}))^T]\|\\
    \leq& 8nC_z^2C_{\phi}^2C_m\rho^t\\
    \triangleq& \Tilde{C}_{\phi}\rho^t.
\end{align*}

Based on  \eqref{decomp}, we conclude that 
\begin{align*}
    \mathbb{E}[\|\hat{\phi}_{t+1}-\phi_i\|^2]\leq& \mathbb{E}[\|\hat{\phi}_{t}-\phi_i\|^2]+\alpha^2 C_g-2\alpha \lambda_{min}(\hat{P}_{\infty})\mathbb{E}[\|\hat{\phi}_t-\phi_i\|^2]+2\alpha\Tilde{C}_{\phi}\rho^t\\
    =&[1-2\alpha \lambda_{min}(\hat{P}_{\infty})]\mathbb{E}[\|\hat{\phi}_{t}-\phi_i\|^2]+\alpha^2C_g+2\alpha \Tilde{C}_{\phi}\rho^t.
\end{align*}
It follows that
\begin{small}
\begin{align*}
    \mathbb{E}[\|\hat{\phi}_M-\phi_i\|^2]\leq& [1-2\alpha \lambda_{min}(\hat{P}_{\infty})]^M \|\hat{\phi}_0-\phi_i\|^2+\alpha^2C_g\sum_{t=0}^{M-1}[1-2\alpha \lambda_{min}(\hat{P}_{\infty})]^t\\
    &+2\alpha \Tilde{C}_{\phi}\sum_{t=0}^{M-1}[1-2\alpha \lambda_{min}(\hat{P}_{\infty})]^t\rho^{M-1-t}\\
    \leq& [1-2\alpha \lambda_{min}(\hat{P}_{\infty})]^M\|\hat{\phi}_0-\phi_i\|^2+\frac{\alpha C_g}{2\lambda_{min}(\hat{P}_{\infty})}+2\alpha M\Tilde{C}_{\phi}[1-2\alpha \lambda_{min}(\hat{P}_{\infty})]^{M-1}\\
    =&[1-2\alpha \lambda_{min}(\hat{P}_{\infty})]^M\left(\|\phi^*_{\theta}-\phi_i\|^2+\frac{2\alpha M\Tilde{C}_{\phi}}{1-2\alpha \lambda_{min}(\hat{P}_{\infty})}\right)+\frac{\alpha C_g}{2\lambda_{min}(\hat{P}_{\infty})}.
\end{align*}
\end{small}

\section{Controller Design for LQR}
\label{app:controller}

For block $i$ where the model parameters $A_i$ and $B_i$ are constant, the system dynamics satisfy:
\begin{equation*}
    x_{t+1}=A_ix_t+B_iu_t+w_t.
\end{equation*}
As is standard, the incurred cost at time $t$ is given by the following quadratic function:
\begin{equation*}
    C_t=x_t^TSx_t+u_t^TRu_t,
\end{equation*}
where $S$ and $R$ are positive definite matrices. 
Denote $J^*(A_i, B_i)$ as the optimal cost per stage, i.e., 
\begin{equation*}
    J^*(A_i, B_i)=\min_{u}\lim_{T\rightarrow\infty}\frac{1}{T}\sum_{t=1}^T \mathbb{E}_{w_t}[C_t|A_i,B_i],
\end{equation*}
where the minimum is taken over measurable functions $u=\{u_t(\cdot)\}_{t\geq 1}$. 
To minimize the average cost per stage, it is well known that for a LTI  system with known parameters, the LQR problem admits an optimal static state-feedback control policy, i.e., $u_t=Kx_t$.

In the case when the model parameters $(A_i, B_i)$ are unknown, one can synthesize the control based on the model estimation $(\hat{A}_i,\hat{B}_i)$.  More specifically, 
define $A_i=\hat{A}_i-\Delta_{A_i}$ and $B_i=\hat{B}_i-\Delta_{B_i}$. With a transition to the high-probability bound, we can obtain that $\max\{\|\Delta_{A_i}\|_2, \|\Delta_{B_i}\|_2\}\leq \epsilon_i$ for block $i$ after online adaptation from meta-learning initialization $[A_{\theta}, B_{\theta}]$, as shown in Proposition \ref{thm:thm.2}. We introduce two different controller designs here based on the meta-learning estimation. 

\subsection{Certainty Equivalent Controller (CEC)}

The certainty equivalent controller for block $i$ can be obtained based on the certainty equivalence principle as following:


\begin{mini}
    {u}{\lim_{T\rightarrow\infty}\frac{1}{T}\sum_{t=1}^T \mathbb{E}_{w_t}[C_t],}
    {\label{cec}}{}
    \addConstraint{}{x_{t+1}=\hat{A}_ix_t+\hat{B}_iu_t+w_t}
    \addConstraint{}{u_t = K_ix_t.}
\end{mini}

Given the estimation error $\epsilon_i$ in system identification, the following result follows directly from \cite{mania2019certainty}, which provides a sub-optimality guarantee on the solution to problem \eqref{cec}.
\begin{corollary}\cite{mania2019certainty}
Let $J^*$ denote the minimal LQR cost achievable by any controller for the dynamical system with transition matrices $(A_i,B_i)$, and let $K_*$ denote its optimal static feedback controller.
Denote $P_*$ as the solution to the discrete Riccati equation with optimal $(A_i, B_i)$:
\begin{equation*}
    P_*=A_i^TP_*A_i-A_i^TP_*B_i(R+B_i^TP_*B_i)^{-1}B_i^TP_*A_i+S,
\end{equation*} 
and $\hat{P}$ as the solution to the discrete Riccati equation with estimation $(\hat{A}_i,\hat{B}_i)$. Suppose $m\leq n$. Let $\gamma>0$ such that $\rho(A_i+B_iK_*)\leq \gamma<1$. Also, let $\epsilon_i>0$ such that $\|\hat{A}_i-A_i\|\leq\epsilon_i$ and $\|\hat{B}_i-B_i\|\leq\epsilon_i$ and assume $\|\hat{P}-P_*\|\leq f(\epsilon_i)$ for some function $f$ such that $f(\epsilon_i)\geq \epsilon_i$. Define $\tau(M,\rho):=\sup\{\|M^k\|\rho^{-k}:k\geq 0\}$ and $\Gamma_*:=1+\max\{\|A_i\|,\|B_i\|,\|P_*\|,\|K_*\|\}$. Then, for positive definite $S$ and $R$, the certainty equivalent controller obtained from (16) achieves
\begin{equation*}
    C_t(K_i)-J^*\leq 200\sigma_w^2m\Gamma_*^9\frac{\tau(A_i+B_iK_*,\gamma)^2}{1-\gamma^2}f(\epsilon_i)^2,
\end{equation*}
as long as $f(\epsilon_i)$ is small enough so that the right hand side is smaller than $\sigma_w^2$.
\end{corollary}

For more details and experimental performance, interested readers may refer to \cite{mania2019certainty}.

\subsection{Robust Controller}

It has been shown in \cite{mania2019certainty} that the CEC performs well only if the model estimation error is small. Therefore, when the estimation error becomes larger, we need a more robust controller design. To secure the control performance with the model estimation error, it is natural to impose a robust optimization problem which aims to minimize the worst-case performance given the estimation uncertainty set. Hence, for each block $i$, we aim to solve the following robust optimization problem:
\begin{mini}
    {u}{\sup_{\max\{\|\Delta_{A_i}\|_2, \|\Delta_{B_i}\|_2\}\leq \epsilon_i} \lim_{T\rightarrow\infty}\frac{1}{T}\sum_{t=1}^T \mathbb{E}_{w_t}[C_t],}
    {\label{robustcontrol}}{}
    \addConstraint{}{x_{t+1}=(\hat{A}_i-\Delta_{A_i})x_t+(\hat{B}_i-\Delta_{B_i})u_t+w_t}
    \addConstraint{}{u_t = K_ix_t.}
\end{mini}
As shown in \cite{dean2018regret}, using the SLS approach \cite{wang2019system} which focuses on the system responses of a closed-loop system, we can convert the robust optimization problem to a semi-definite programming problem and further quantify the robust control performance as a function of the estimation error bound $\epsilon_i$.

More specifically, consider a static state-feedback control policy $K_i$ for block $i$, i.e., $u_t=K_ix_t$ for $t\in[(i-1)L+1, iL]$. The closed loop map from the disturbance process $\{w_{(i-1)L}, w_{(i-1)L+1}, ...\}$ to the state $x_t$ and control $u_t$ at time $t$ can be obtained by
\begin{align}\label{s1}
\begin{aligned}
    x_t &= \sum_{k=(i-1)L+1}^t (A_i+B_iK_i)^{t-k}w_{k-1},\\
    u_t &= \sum_{k=(i-1)L+1}^t K_i(A_i+B_iK_i)^{t-k}w_{k-1}.
\end{aligned}
\end{align}
Denote $\Phi_x^{i}(t):= (A_i+B_iK_i)^{t-1}$ and $\Phi_u^{i}(t):= K_i(A_i+B_iK_i)^{t-1}$. Then \eqref{s1} can be rewritten as
\begin{equation}\label{s2}
    \begin{bmatrix} x_t\\u_t \end{bmatrix} = \sum_{k=(i-1)L+1}^t \begin{bmatrix} \Phi_x^{i}(t-k+1)\\\Phi_u^{i}(t-k+1) \end{bmatrix} w_{k-1},
\end{equation}
where $\{\Phi_x^{i}(t), \Phi_u^{i}(t)\}$ are called the closed-loop system response elements induced by controller $K_i$ for the block $i$. For the system dynamics to satisfy \eqref{s2}, a sufficient and necessary condition is that  $\{\Phi_x^{i}(t), \Phi_u^{i}(t)\}$ for $t\in [(i-1)L+1, iL]$ must satisfy:
\begin{align}\label{s3}
    \Phi_x^{i}(t+1)=A_i\Phi_x^{i}(t)+B_i\Phi_u^{i}(t), \Phi_x^{i}(1)=I. 
\end{align}
With the system response, the SLS approach can convert the nonconvex constraint set \eqref{s1} with respect to controller $K_i$ to an affine set \eqref{s3} of $\{\Phi_x^{i}(t), \Phi_u^{i}(t)\}$. By transferring the time-domain system response to the frequency domain represented with boldface letters via $z$-transform, i.e., transfer function $\boldsymbol{\Phi}_x^i(z)=\sum_{t=1}^{\infty}\Phi_x^i(t)z^{-t}$, we can rewrite \eqref{s3} as
\begin{equation}\label{s4}
    \begin{bmatrix} zI-A_i & -B_i \end{bmatrix}\begin{bmatrix}\boldsymbol{\Phi}_x^i\\\boldsymbol{\Phi}_u^i \end{bmatrix}=I,
\end{equation}
and 
\begin{align*}
    \begin{aligned}
        \boldsymbol{\Phi}_x^i&=(zI-A_i-B_iK_i)^{-1},\\
        \boldsymbol{\Phi}_u^i&=K_i(zI-A_i-B_iK_i)^{-1},
    \end{aligned}
\end{align*}
implying the corresponding controller $K_i=\boldsymbol{\Phi}_u^i(\boldsymbol{\Phi}_x^i)^{-1}$.

Based on the affine constraint set \eqref{s4} on $\{\boldsymbol{\Phi}_x^i, \boldsymbol{\Phi}_u^i\}$, the SLS framework allows us to characterize the system responses on the true system $(A_i, B_i)$ with a controller computed using only the model parameter estimates $\{(\hat{A}_i, \hat{B}_i)\}$. If we denote $\Delta_i:=(\hat{A}_i-A_i)\boldsymbol{\Phi}_x^i+(\hat{B}_i-B_i)\boldsymbol{\Phi}_u^i$, simple algebra shows that
\begin{equation*}
    \begin{bmatrix} zI-\hat{A}_i & -\hat{B}_i \end{bmatrix}\begin{bmatrix}\boldsymbol{\Phi}_x^i\\\boldsymbol{\Phi}_u^i \end{bmatrix}=I \Leftrightarrow
    \begin{bmatrix} zI-A_i & -B_i \end{bmatrix}\begin{bmatrix}\boldsymbol{\Phi}_x^i\\\boldsymbol{\Phi}_u^i \end{bmatrix}=I+\Delta_i.
\end{equation*}
It can be shown that if $(I+\Delta_i)^{-1}$ exists, the controller $K_i=\boldsymbol{\Phi}_u^i(\boldsymbol{\Phi}_x^i)^{-1}$ computed with $(\hat{A}_i, \hat{B}_i)$ achieves the following system response on the true system $(A_i, B_i)$:
\begin{equation*}
    \begin{bmatrix} \boldsymbol{x}\\\boldsymbol{u} \end{bmatrix} = \begin{bmatrix} \boldsymbol{\Phi}_x^i\\ \boldsymbol{\Phi}_u^i \end{bmatrix}(I+\Delta_i)^{-1} \boldsymbol{w}.
\end{equation*}
Moreover, it has been shown in \cite{dean2017sample} that $K_i$ is a stabilizing controller for $(A_i, B_i)$ if $K_i$ stabilizes $(\hat{A}_i, \hat{B}_i)$ and $\|\Delta_i\|_{\mathcal{H}_{\infty}}<1$. Consider the space of proper and real rational stable transfer functions:
\begin{equation*}
    \mathcal{RH}_{\infty}(C, \rho):=\left\{\boldsymbol{\Phi}=\sum_{t=0}^{\infty} \Phi_t z^{-t}|\|\Phi_t\|\leq C\rho^t\right\}
\end{equation*}
with positive $C$ and $\rho\in[0, 1)$. To guarantee that $K_i$ stabilizes $(\hat{A}_i, \hat{B}_i)$, we need $z\boldsymbol{\Phi}_x^i\in \mathcal{RH}_{\infty}(C_x, \rho)$ and $z\boldsymbol{\Phi}_u^i\in \mathcal{RH}_{\infty}(C_u, \rho)$ for some positive constant $C_x$ and $C_u$.

Therefore, to obtain a robust stabilizing controller for block $i$, we can reformulate the non-convex optimization problem \eqref{robustcontrol} as the following quasi-convex problem \cite{dean2018regret}:
\begin{mini}
    {\gamma\in[0, 1)}{\frac{1}{1-\gamma}\min_{\boldsymbol{\Phi}_x^i, \boldsymbol{\Phi}_u^i} \left\|\begin{bmatrix}Q^{\frac{1}{2}} & 0 \\ 0 & R^{\frac{1}{2}}\end{bmatrix}\begin{bmatrix}\boldsymbol{\Phi}_x^i \\\boldsymbol{\Phi}_u^i \end{bmatrix}\right\|_{\mathcal{H}_2}} 
    {}{}
    \addConstraint{}{\begin{bmatrix} zI-\hat{A}_i & -\hat{B}_i \end{bmatrix}\begin{bmatrix}\boldsymbol{\Phi}_x^i\\\boldsymbol{\Phi}_u^i \end{bmatrix}=I, \left\|\begin{bmatrix}\boldsymbol{\Phi}_x^i \\\boldsymbol{\Phi}_u^i \end{bmatrix}\right\|_{\mathcal{H}_{\infty}}\leq\frac{\gamma}{\sqrt{2}\epsilon_i}}
    \addConstraint{}{z\boldsymbol{\Phi}_x^i\in \mathcal{RH}_{\infty}(C_x, \rho), z\boldsymbol{\Phi}_u^i\in \mathcal{RH}_{\infty}(C_u, \rho).}
\end{mini}
The objective function is quasi-convex with respect to $\gamma$ which can be solved efficiently. The inner problem is a convex but infinite-dimensional problem, which can be relaxed to a finite-dimensional problem via FIR truncation still with performance guarantee. After solving this problem, we can obtain a robustly stabilizing controller $K_i=\boldsymbol{\Phi}_u^i(\boldsymbol{\Phi}_x^i)^{-1}$ for block $i$. 
The following result characterizes the performance guarantee of the robust controller.

\begin{corollary}\cite{dean2018regret}
Let $J^*$ denote the minimal LQR cost achievable by any controller for the dynamical system with transition matrices $(A_i,B_i)$, and let $K_*$ denote its optimal static feedback controller. Suppose that the resolvent $\mathcal{R}_{A_i+B_iK_*}$, denoted by $(zI-(A_i+B_iK_*))^{-1}$, satisfies $\mathcal{R}_{A_i+B_iK_*}\in\mathcal{RH}_{\infty}(C_*,\rho_*)$ and that $(wlog)\rho_*\geq 1/e$. Suppose furthermore that $\epsilon_i$ is small enough to satisfy the following conditions:
\begin{align*}
    \epsilon_i(1+\|K_*\|)\|\mathcal{R}_{A_i+B_iK_*}\|_{\mathcal{H}_{\infty}}\leq& 1/5,\\
    \epsilon_i(1+\|K_*\|)C_*\leq& 1-\rho_*.
\end{align*}
Let $(\hat{A}_i,\hat{B}_i)$ be any estimates of the transition matrices such that $\max\{\|\Delta_{A_i}\|,\|\Delta_{B_i}\|\}\leq \epsilon_i$. Then if $(C_x,\rho)$ and $(C_u,\rho)$ are set as,
\begin{align*}
    C_x=&\frac{O(1)C_*}{1-\rho_*},\\
    C_u=&\frac{O(1)\|K_*\|C_*}{1-\rho_*},\\
    \rho=&(1/4)\rho_*+3/4,
\end{align*}
we have that (a) the problem \eqref{robustcontrol} is feasible, (b) letting $K_i$ denote an optimal solution to \eqref{robustcontrol}, the relative error in the LQR cost is
\begin{equation*}
    C_t(A_i,B_i,K_i)\leq [1+5\epsilon_i(1+\|K_*\|)\|\mathcal{R}_{A_i+B_iK_*}\|_{\mathcal{H}_{\infty}}]^2J^*.
\end{equation*}
\end{corollary}

For more details and experimental performance, interested readers may refer to \cite{dean2017sample,dean2018regret}.

\section{$\beta$-Mixing Coefficient of LTI systems}

Based on the well-known result that the LQR problem in the LTI system can be solved with a linear feedback policy $u_t=Kx_t$, we assume that each block $i$  evolves with a stabilizing controller $K$ during the online adaptation, i.e., $A_i+B_iK$ is a stable matrix, thereby generating a trajectory of $M$ samples. For example,  $K$ can be obtained by solving a robust optimization problem \eqref{robustcontrol} with the offline meta-learning initialization $\phi^*_{\theta}$ based on the estimation gap shown in Theorem 1. 

Let $\mathbb{T}$ denote the unit circle  and $\mathbb{D}$ as the open unit disk in the complex plane. Let $\mathcal{RH}_{\infty}$ denote the space of matrix-valued, real-rational functions which are analytic on $\mathbb{D}^c$. For $\Phi\in\mathcal{RH}_{\infty}$, the $\mathcal{H}_{\infty}$-norm $\|\Phi\|_{\mathcal{H}_{\infty}}$ is defined as $\|\Phi\|_{\mathcal{H}_{\infty}}\triangleq \sup_{s\in\mathbb{T}}\|\Phi(s)\|$. Furthermore, given a square matrix $A$, its resolvant $\Phi_{A}(s)$ is defined as $\Phi_{A}(s)\triangleq (sI-A)^{-1}$. When $A$ is stable, $\Phi_{A}\in\mathcal{RH}_{\infty}$ and $\|\Phi\|_{\mathcal{H}_{\infty}}<\infty$. It has been shown in \cite{tu2017least} that the $\beta$-mixing coefficient of LTI with stable $A_i+B_iK$ is
\begin{align*}
    \beta(t)&=\sup_{k\geq 1}\mathbb{E}_{x\sim \nu_{k}}[\|\mathbb{P}_{x_t}(\cdot|x_0=x)-\nu_{\infty}\|_{tv}]\\
    &\leq \frac{\|\Phi_{\rho^{-1}(A_i+B_iK)}\|_{\mathcal{H}_{\infty}}}{2}\sqrt{Tr(P_{\infty})+\frac{n}{1-\rho^2}}\rho^t\triangleq C_m\rho^t
\end{align*}
for any fixed $\rho\in(\rho(A_i+B_iK),1)$, $t\geq 1$ and any distribution $\nu_{k}$, where $\|\cdot\|_{tv}$ refers to the total-variation norm on probability measures and $\rho(A_i+B_iK)$ is the spectral radius of the matrix $A_i+B_iK$.

\section{Numerical Results}

\textbf{Experiments Setup:} To generate episodic blocks for offline Meta-L, the parameters are set as follows.  Moreover, the noise $w_t \sim \mathcal{N}(0,\,0.01 \times I_n)$, where $I_n$ is a $n\times n$ identity matrix, and action $u_t \sim \mathcal{N}(0,\,0.1 \times I_m)$.  For every block $d$, we assume that each element in both $A_d$ and $B_d$ follows a uniform distribution within $[0.5,1]$.
Along this setup, we generate $D$ episodic blocks for offline Meta-L, where each block has length $L$ and zero initial state. 
Further, for each episodic block $d$, the samples collected within the first $M$ time steps, i.e., $\tau_d(0, M-1)$, are taken as the training dataset for that block, whereas the rest of the block, i.e., $\tau_d(M,L-1)$, serves as the testing dataset. We set the learing rate $\alpha = 0.01$. To evaluate the estimation gap between meta-initialization $\phi^*_{\theta}$ and the underlying model parameter $\phi_i$ for a new block $i$ and the estimation error for the online adaptation, we measure the average estimation error in terms of 2-norm over $50$ testing blocks.
The estimated parameter after $M$ steps adaptation is denoted as $\phi(M)$ in the experiments.

\textbf{Offline Meta-L:} 
The performance of offline Meta-L is demonstrated in Figure \ref{fig:04} and Figure \ref{fig:05}.
For a fixed dimension, it can be seen from Figure \ref{fig:04} that the average estimation gap $\|\phi^*_{\theta}-\phi_i\|$ decreases with the number of blocks $D$ at first, and then converges to some constant at some point, corroborating the result in Theorem \ref{thm:thm.1}. Besides, this estimation gap clearly increases with the dimension $m$ and $n$, as illustrated in Figure \ref{fig:05}.
As shown in Figure \ref{fig:impactL}, one can always choose a smaller $L$ to improve the  accuracy of the episodic block model for the  LTV system, and the meta-L algorithm for offline learning can still work well as long as many blocks are available.

\begin{figure}[htbp]
\centering
\begin{minipage}[t]{0.48\textwidth}
\centering
\includegraphics[width=0.9\textwidth]{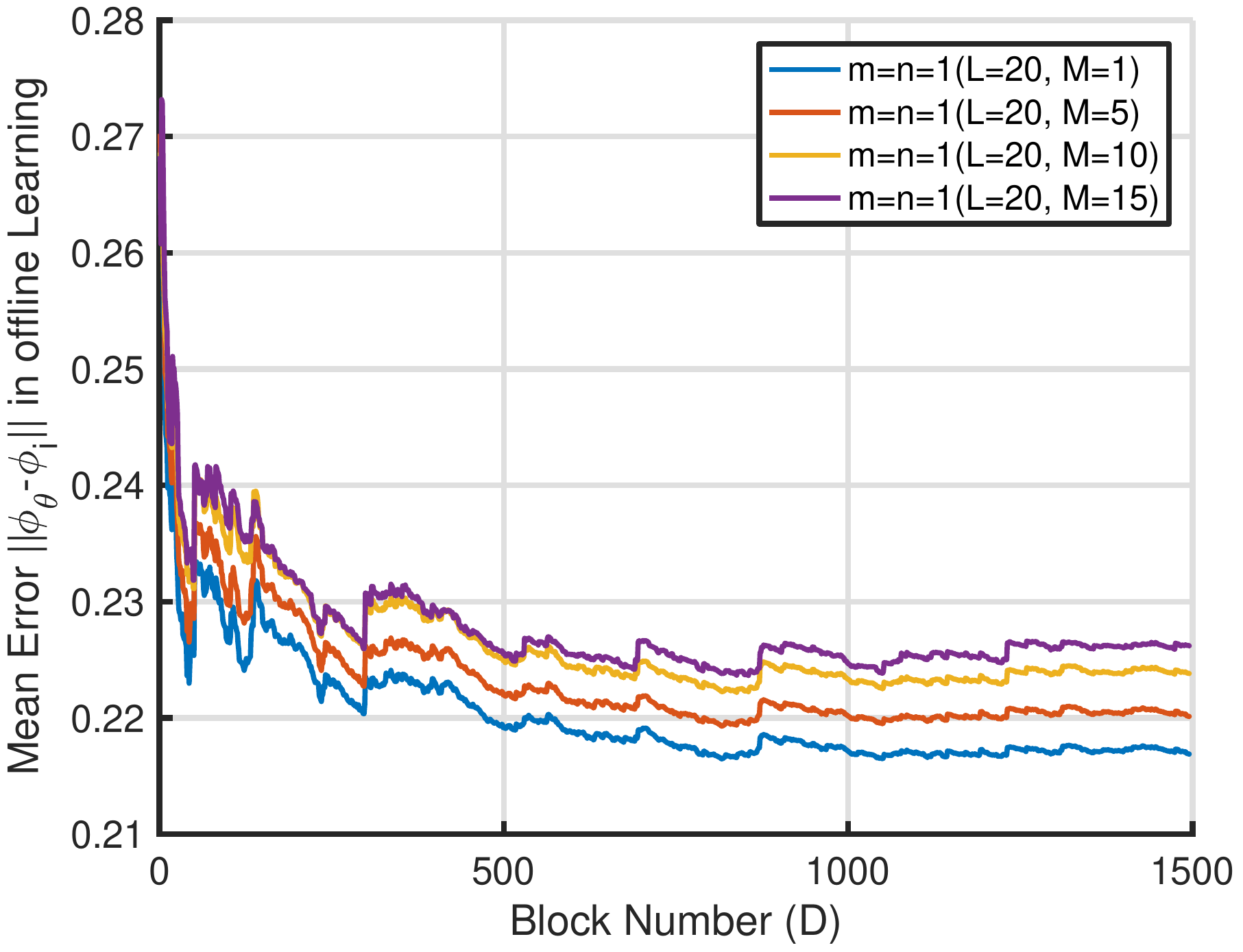}
\caption{Impact of the number of blocks $D$ on the average estimation gap for offline Meta-L  with same dimension but different training sizes.}\label{fig:04}
\end{minipage}\hfill
\begin{minipage}[t]{0.48\textwidth}
\centering
\includegraphics[width=0.9\textwidth]{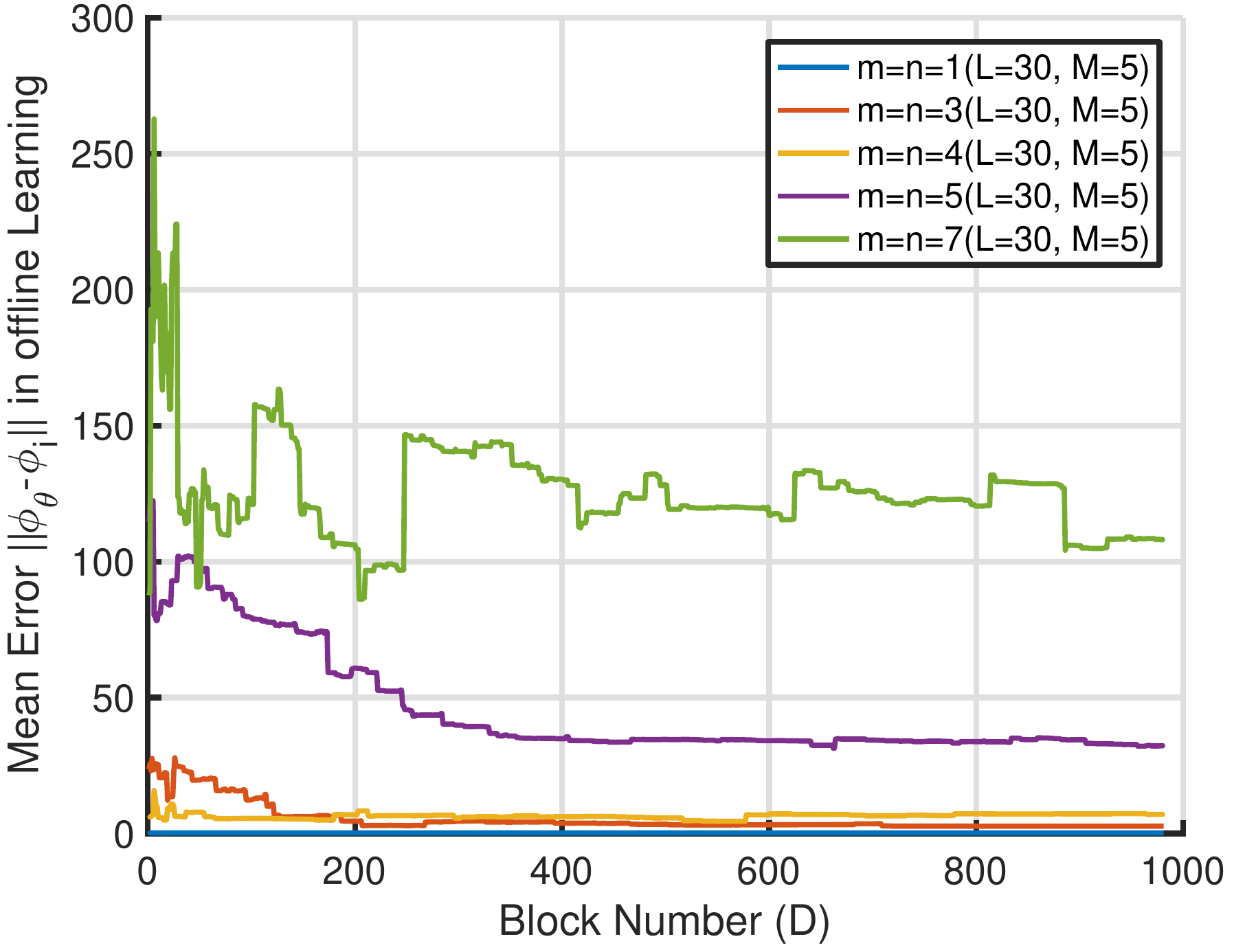}
\caption{Impact of the dimension $m$ and $n$ on the average estimation gap for offline Meta-L  with same block size.}\label{fig:05}
\end{minipage}
\end{figure}

\begin{figure}
    \centering
    \includegraphics[width=0.5\textwidth]{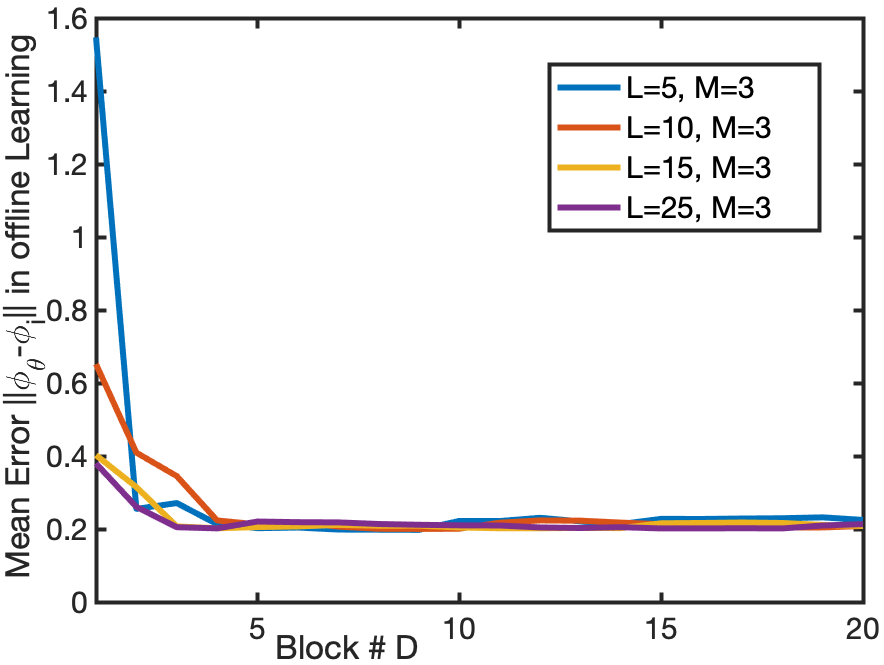}
    \caption{Impact of the block length $L$ on the average estimation gap for offline Meta-L.}
    \label{fig:impactL}
\end{figure}

\textbf{Online Adaptation:} Figure \ref{fig:06} demonstrates that the impact of different Meta-L model initialization $\phi^*_{\theta}$ on the average  estimation error $\|\phi(M)-\phi_i\|$. Clearly, the error decreases at first with the increase of the block number $D$ because a better $\phi^*_{\theta}$ would be learnt by the offline Meta-L. This trend stops when the model initialization $\phi^*_{\theta}$ stops improving with $D$.
Moreover, as illustrated in Figure \ref{fig:07}, for a suitably chosen learning rate $\alpha$ the average estimation error for the online adaptation decreases quickly with the increase of the number of training samples $M$.

\begin{figure}[htbp]
\centering
\begin{minipage}[t]{0.48\textwidth}
\centering
\includegraphics[width=0.9\textwidth]{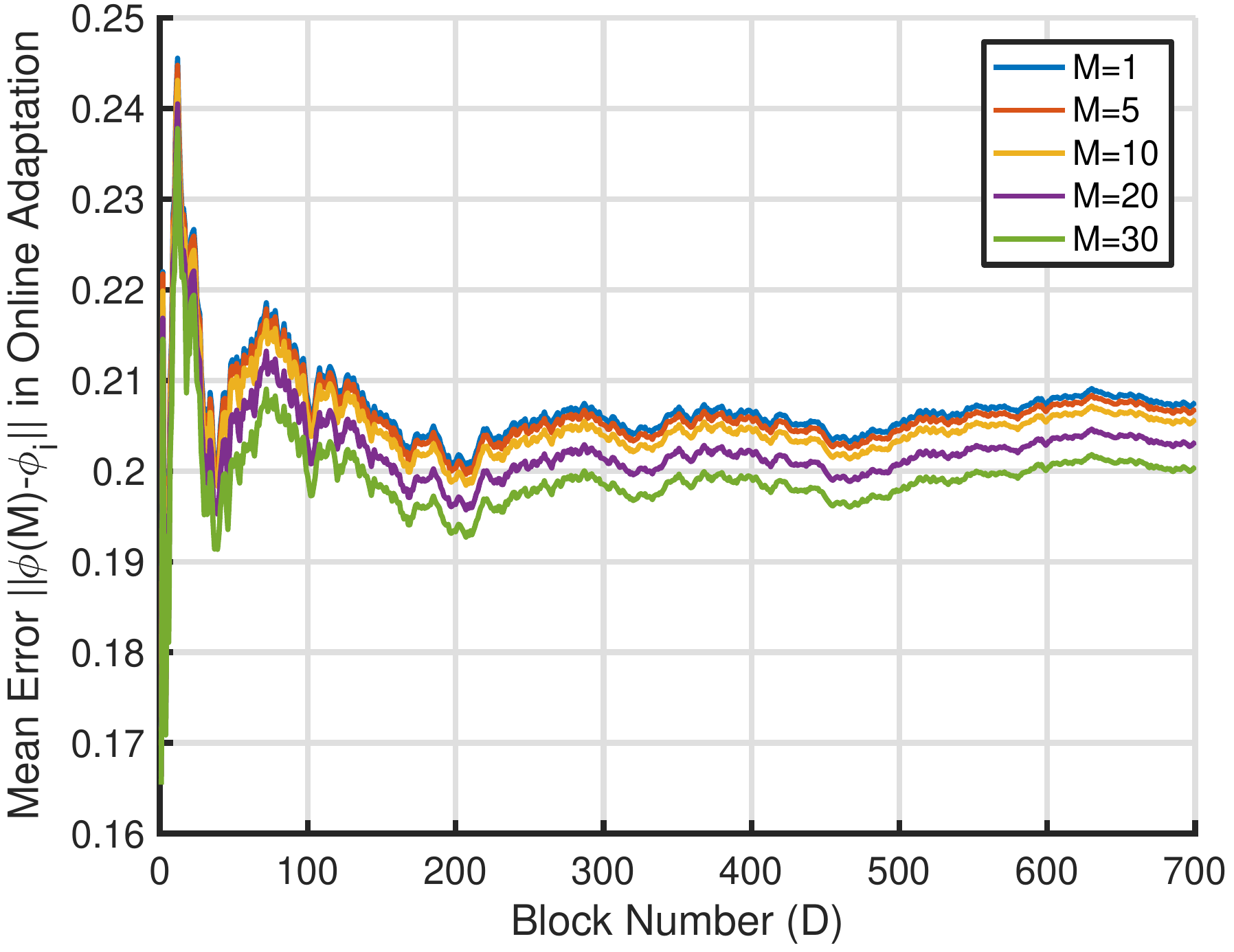}
\caption{Impact of offline learning block number $D$ on the average estimation error  for online adaptation under different training sizes.}\label{fig:06}
\end{minipage}\hfill
\begin{minipage}[t]{0.48\textwidth}
\centering
\includegraphics[width=0.9\textwidth]{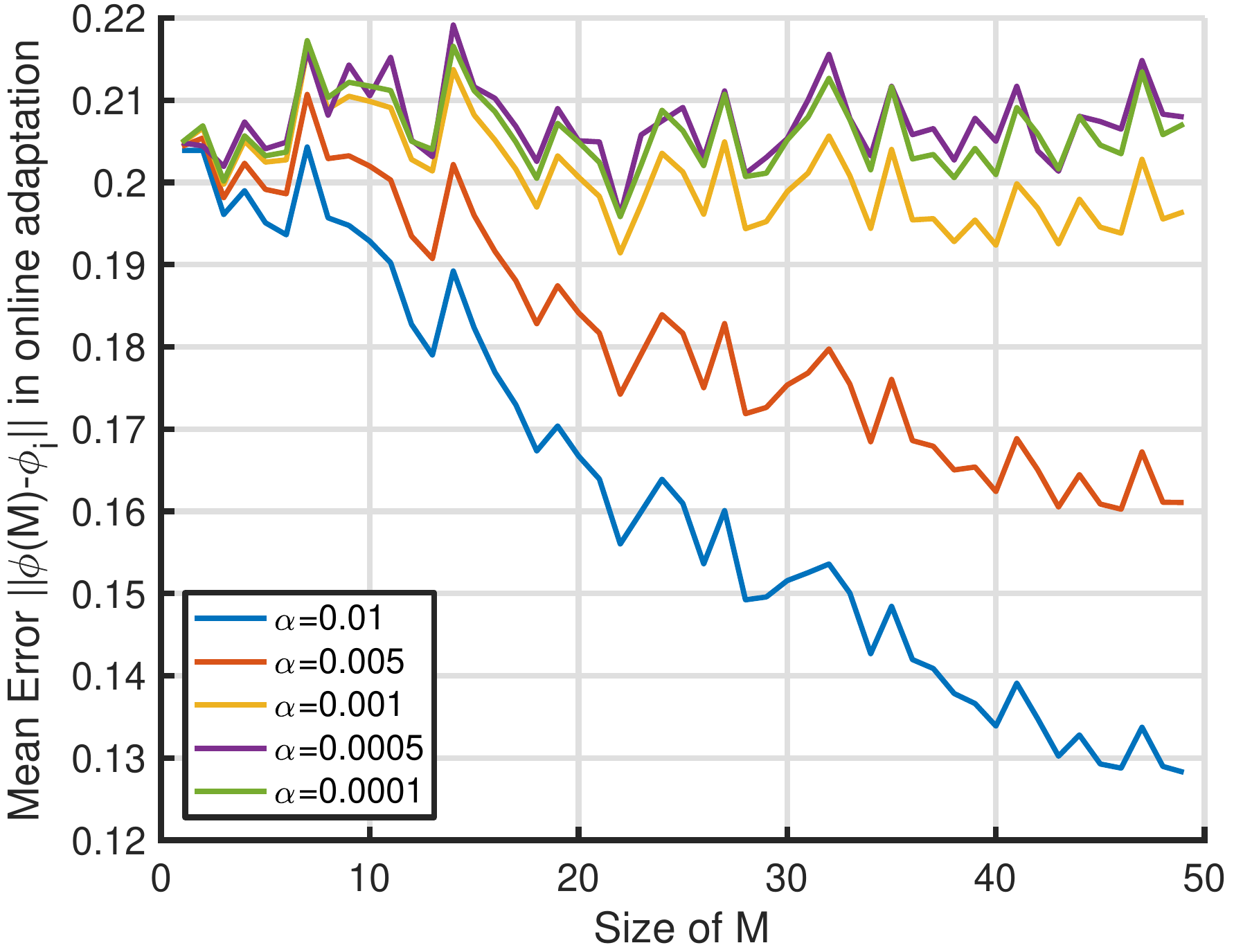}
\caption{Impact of sample size $M $ on the the average estimation error  for online adaptation under different learning rates.}\label{fig:07}
\end{minipage}
\end{figure}

\textbf{Least Square Estimator vs Meta-L:}
To demonstrate the advantage of Meta-L based system identification with small sample sizes, we compare the average estimation performance between the proposed Meta-L based system identification and the classical Least Square Estimator (LSE). More specifically, we set the noise $w_t\sim N(0,1)$ and $u_t\sim N(0,1)$. We also present the comparison under different values of meta-initialization $\phi^*_{\theta}$ to further substantiate the performance of Meta-L even with perturbed $\phi^*_{\theta}$, as demonstrated in Figure \ref{fig:08}-\ref{fig:11}. Clearly, the Meta-L based online adaptation outperforms the LSE in all cases with small sample sizes, which corroborates the benefits by using a good initial point in the recursive linear stochastic approximation algorithm, compared with the classical LSE where such a good initial point is not utilized.

\textbf{``Harmonic'' Block Model:} To show the performance of meta-learning under correlated block structures, we consider a simple ``harmonic'' block model where the block model parameters switch deterministically between blocks as $[(A_1,B_1),(A_2,B_2),(A_1,B_1),(A_2,B_2),...]$. Here, $(A_1,B_1)=(0.5,0.7)$ and $(A_2,B_2)=(0.8,0.8)$. In this setting $\phi^*_{\theta}$ learnt by offline meta-learning is  the middle point between $\phi^1=(A_1,B_1)$ and $\phi^2=(A_2,B_2)$ with  $\eta=\frac{1}{2}\|\phi^1-\phi^2\|$.
Since the two models $\phi^1$ and $\phi^2$ are bi-modal and hence `easier to find' by online adaptation, an additional gradient step using only a few samples, starting from  $\phi^*_{\theta}$,  can quickly converge to the true model (either $\phi^1$ or $\phi^2$)  in a new block, as illustrated in Figure \ref{fig:harmonic}. In a nutshell, the proposed meta-L based SI algorithm for general cases depends on the `joint effort' of  offline learning  and online adaptation, and this is the essence of meta-L algorithms.

\textbf{Downstream LQR Control based on Model Estimation:} Based on the model estimated via online adaptation using the first $M$ samples for each block, one can deploy the certainty equivalent controller or the robust controller introduced in Appendix \ref{app:controller} for the rest of the block. Since in our model each block is LTI, the controller design of LQR based on model estimation  in LTI systems can be directly applied here for the control of the rest of the block. Interested readers can refer to \cite{mania2019certainty,dean2017sample,dean2018regret} for more details and the empirical performance of the designed controllers with respect to the model estimation error, especially for the scenario when the model estimation error is small, corresponding our experiments on the model estimation performance of the proposed meta-learning based system identification.

\begin{figure}[htbp]
\centering
\begin{minipage}[t]{0.48\textwidth}
\centering
\includegraphics[width=0.9\textwidth]{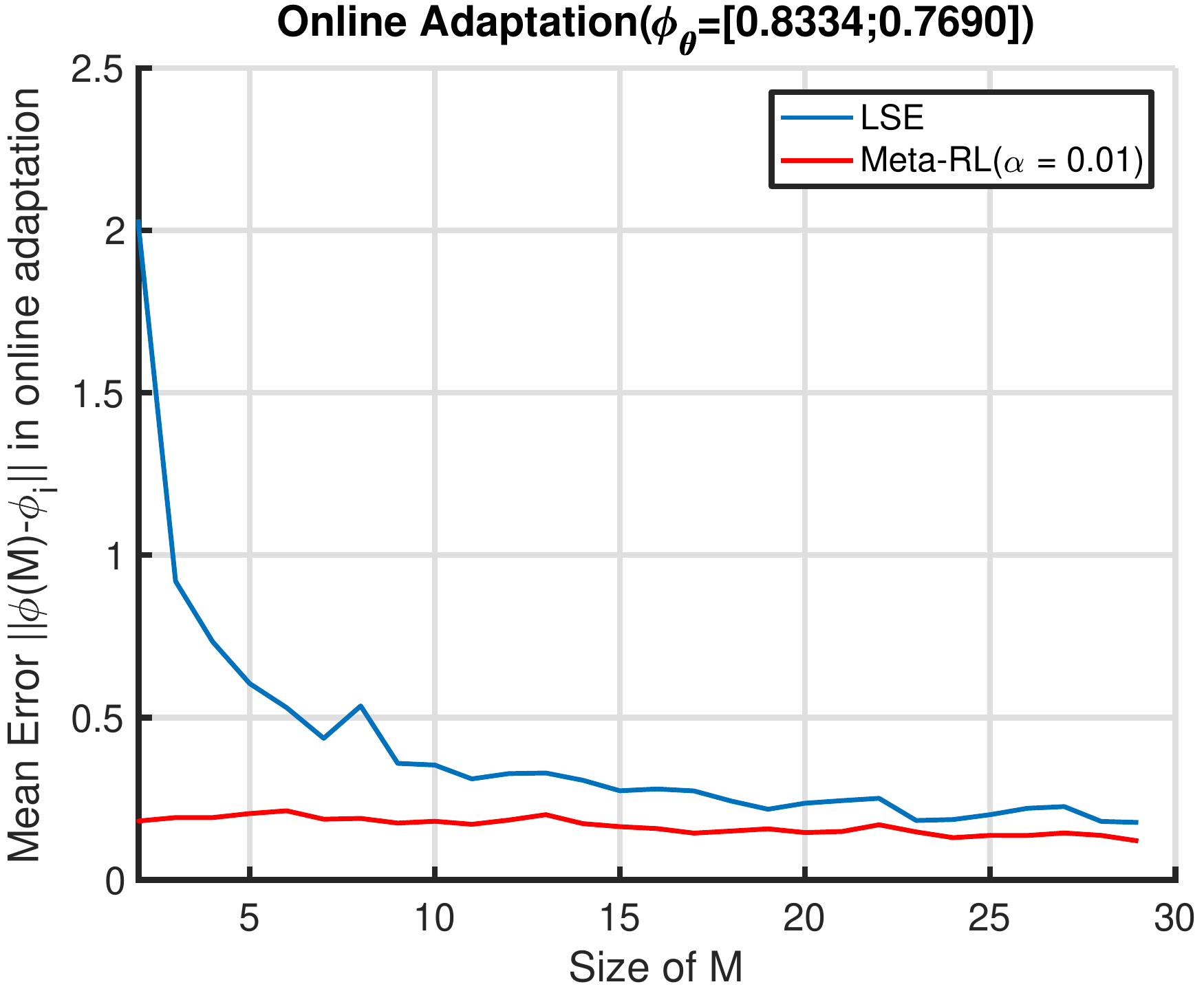}
\caption{Performance comparison for online adaptation between Least Square Estimator and Meta-L based fast adaptation with $\phi^*_{\theta}$ trained by using $300$ blocks in the offline Meta-L.}\label{fig:08}
\end{minipage}
\hfill
\begin{minipage}[t]{0.48\textwidth}
\centering
\includegraphics[width=0.9\textwidth]{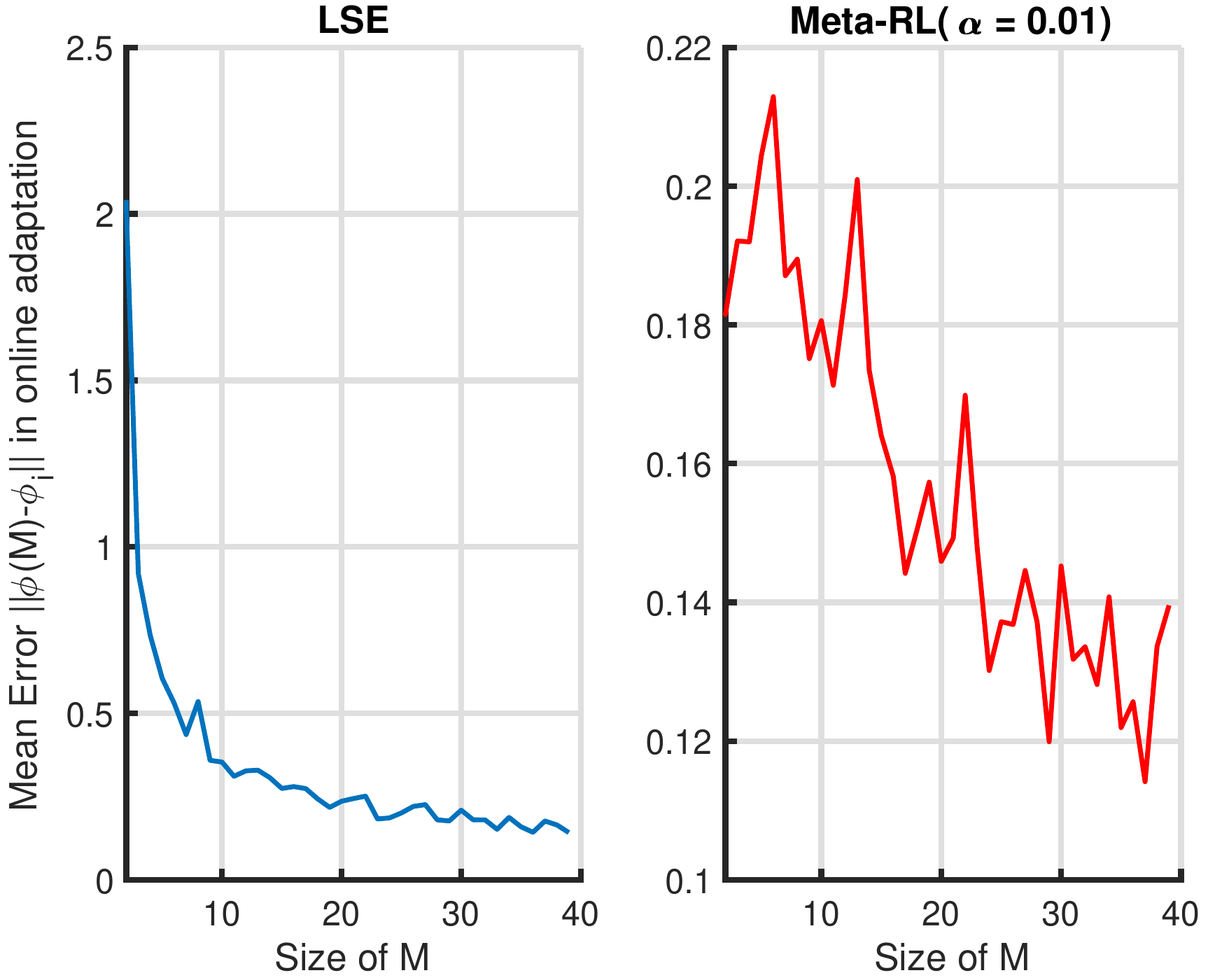}
\caption{A separate view on different scales for  the performance comparison illustrated in Figure \ref{fig:08}. }\label{fig:09}
\end{minipage}
\end{figure}

\begin{figure}[htbp]
\centering
\begin{minipage}[t]{0.48\textwidth}
\centering
\includegraphics[width=0.9\textwidth]{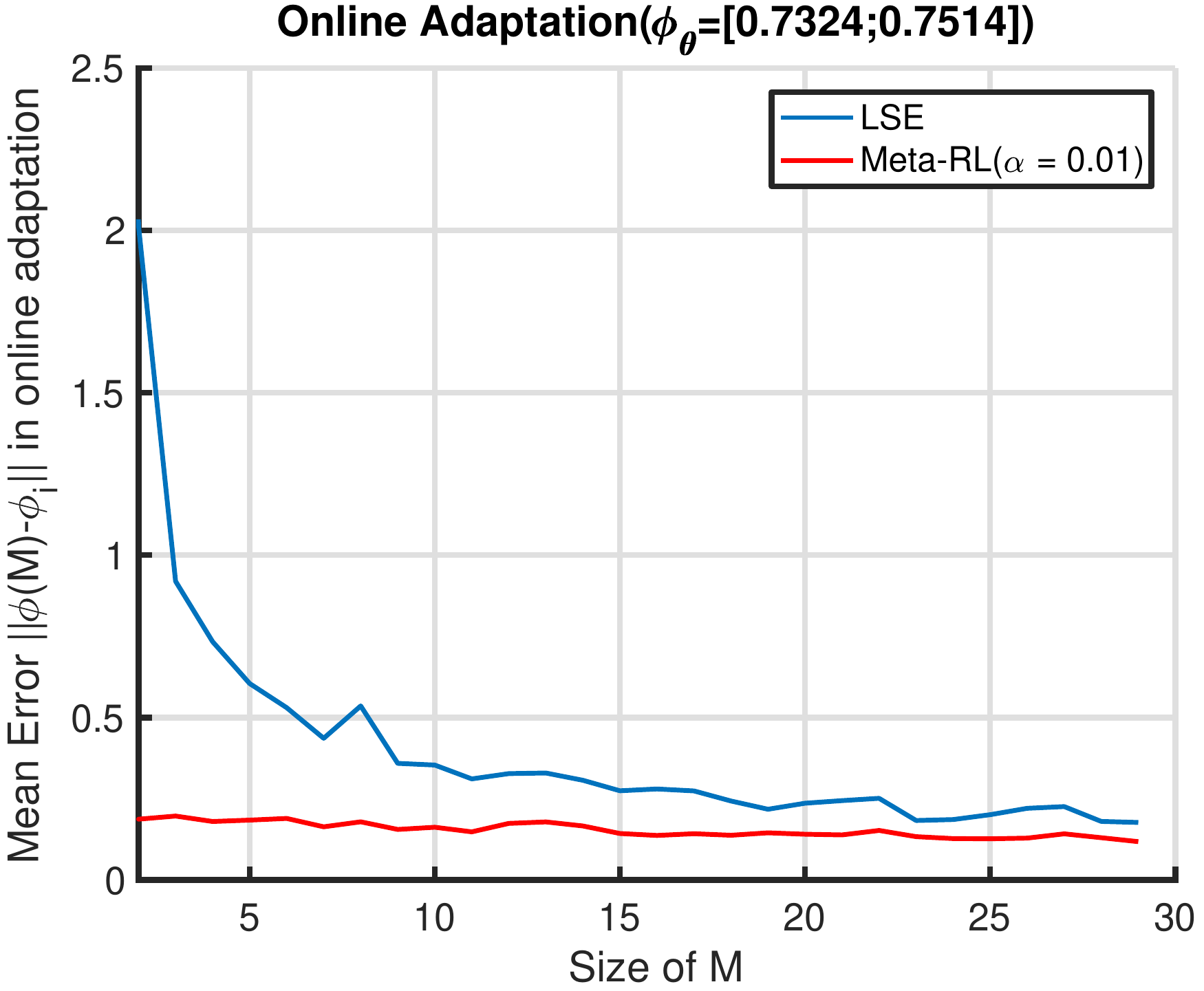}
\caption{Performance comparison for online adaptation between Least Square Estimator and Meta-L based fast adaptation with $\phi^*_{\theta}$ trained by using $1000$ blocks in the offline Meta-L.}\label{fig:10}
\end{minipage}
\hfill
\begin{minipage}[t]{0.48\textwidth}
\centering
\includegraphics[width=0.9\textwidth]{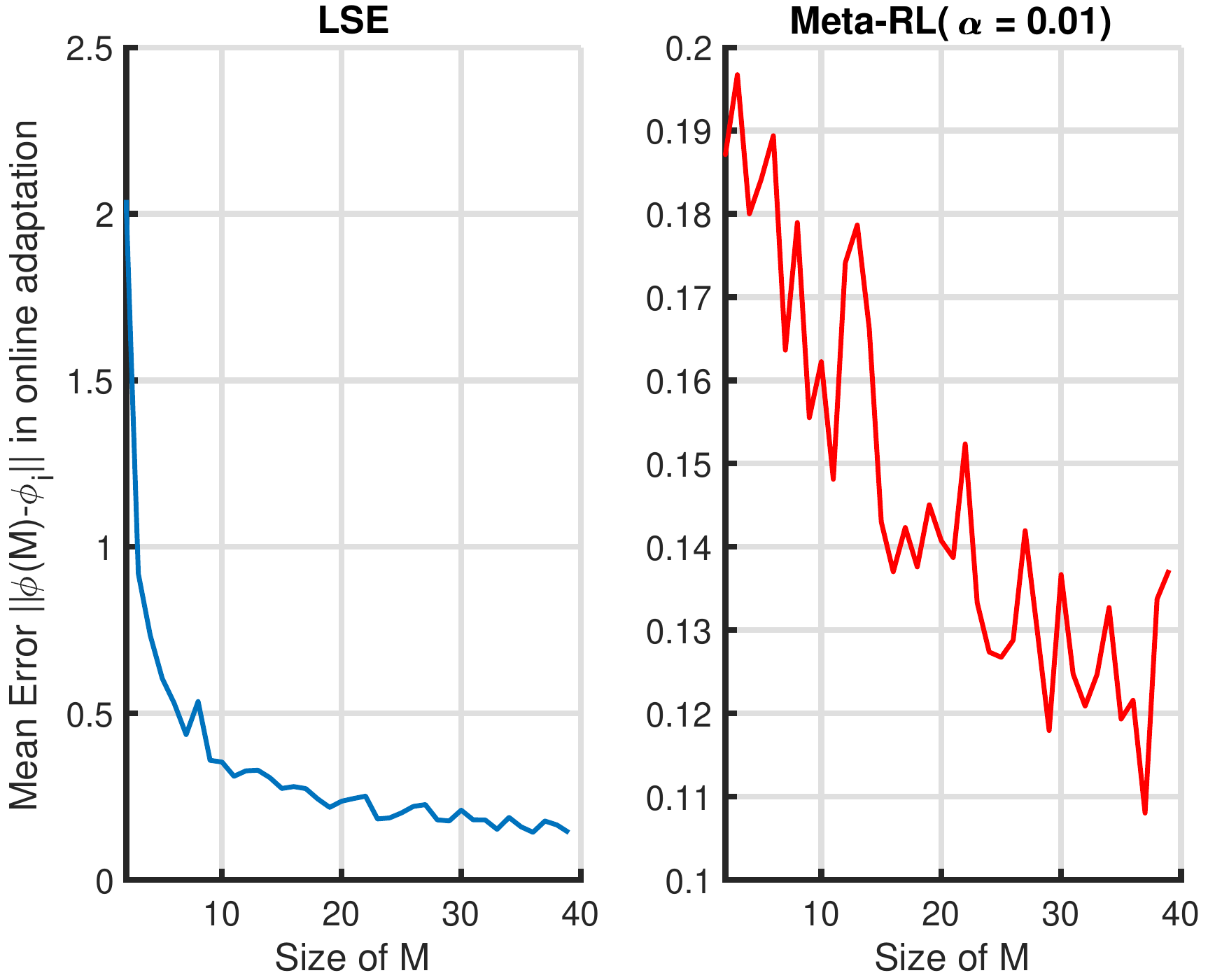}
\caption{A separate view on different scales for the performance comparison  illustrated in Figure \ref{fig:10}.}\label{fig:11}
\end{minipage}
\end{figure}

\begin{figure}
    \centering
    \includegraphics[width=0.45\textwidth]{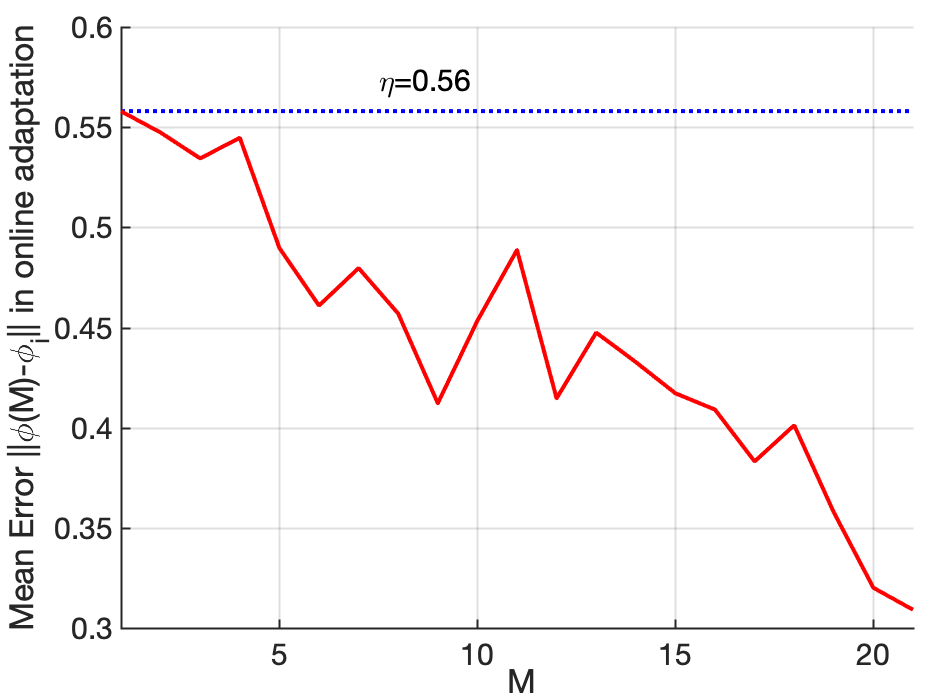}
    \caption{Online adaptation for the ``harmonic" block model.}
    \label{fig:harmonic}
\end{figure}

\end{document}

%% file: math_def.tex
\def \n2{{N_0 \over 2}}

\def\diag{\mbox{\textrm{diag}}}
\def\proj{\mbox{\textrm{proj}}}

\def \h5{\hspace{0.5in}}

\def \bS{\mbox{\boldmath $S$}}

\def \bS{\mbox{\boldmath $S$}}
\def \bC{\mbox{\boldmath $C$}}

\def \bE{\mbox{\boldmath $E$}}
\def \bG{\mbox{\boldmath $G$}}

\def \bd{\mbox{\boldmath $d$}}

\def \bS{\mbox{\boldmath $S$}}

\newcommand{\var}{\mbox{\rm var}}
\newcommand{\cov}{\mbox{\rm cov}}